\documentclass[12pt]{article}
\pdfoutput=1

\usepackage{graphicx}
\usepackage{arxiv}
\usepackage{natbib}
\usepackage[normalem]{ulem}
\usepackage{bm}
\usepackage[utf8]{inputenc}
\usepackage[T1]{fontenc}
\usepackage{hyperref}
\usepackage{url}
\usepackage{xcolor}
\usepackage{wrapfig}

\usepackage{amsmath}
\usepackage{amsfonts}
\usepackage{amssymb}
\usepackage{amsthm}
\usepackage[mathscr]{eucal}
\usepackage{bbm}

\usepackage{algorithm, algpseudocode}
\usepackage{mathtools}
\usepackage{cleveref}
\usepackage{autonum}
\usepackage{ esint }
\usepackage{stmaryrd}
\usepackage{trimclip}

\usepackage{tikz}
\usetikzlibrary{patterns}
\usetikzlibrary{decorations.pathreplacing,calligraphy}

\usepackage{enumitem}
\setlist[itemize]{leftmargin=1.5pc}

\usepackage{yfonts}
\renewcommand{\mod}{\delta}

\theoremstyle{definition}
\newtheorem{assumption}{Assumption}
\newtheorem{remark}{Remark}
\newtheorem{definition}{Definition}
\newtheorem{alg}{Algorithm}

\theoremstyle{plain}
\newtheorem{theorem}{Theorem}
\newtheorem{proposition}{Proposition}
\newtheorem{lemma}{Lemma}
\newtheorem{corollary}{Corollary}

\newtheorem{example}{Example}

\colorlet{sgreen}{black!65!green}

\newcommand{\Prate}{\tilde{\varepsilon}_P}

\renewcommand{\P}{\mathbb{P}}

\newcommand{\Abound}{\boldsymbol{\alpha}}
\newcommand{\Aconst}{C_{0}}

\newcommand{\thisHypName}{}
\newtheorem*{genericHyp}{\thisHypName}

\makeatletter
\def\namedlabel#1#2{\begingroup
    #2%
    \def\@currentlabel{#2}%
    \phantomsection\label{#1}\endgroup
}
\makeatother

\newcommand{\loss}{\ell}
\newcommand{\disc}{\mathrm{disc}}
\newcommand{\semiMetric}{{\rm dist}}
\newcommand{\semiDist}[2]{\semiMetric \!\left( #1 , #2 \right)}
\newcommand{\KLDiv}[2]{\mathcal{D}_{\text{kl}} \!\left(#1 | #2 \right)}

\newcommand{\imod}{\mod^{\scalebox{0.8}{-}\scalebox{0.6}{1}}}

\newcommand{\PQ}{\scalebox{0.6}{P,Q}}
\newcommand{\pivot}{\boldsymbol{\varepsilon}_{\PQ}}
\newcommand{\EPQ}{{\boldsymbol{\mathcal E}}_{\PQ}}

\newcommand{\dist}{\mathrm{dist}}

\newcommand{\X}{\mathcal{X}}
\newcommand{\Y}{\mathcal{Y}}

\newcommand{\Hyp}{\mathcal{H}}

\renewcommand{\H}{\Hyp}
\newcommand{\Xspl}{\mathbf{X}}
\newcommand{\Yspl}{\mathbf{Y}}
\newcommand{\V}{d_\Hyp}

\newcommand{\VV}{d}
\newcommand{\E}{\mathcal{E}}
\newcommand{\EE}{\Expectation}

\newcommand{\real}{\mathbb{R}}

\newcommand{\reals}{\real}
\newcommand{\overbar}[1]{\mkern 1.5mu\overline{\mkern-1.5mu#1\mkern-1.5mu}\mkern 1.5mu}

\newcommand{\exreals}{\overbar{\reals}_{+}}

\newcommand{\expec}{\mathbb{E} }

\newcommand{\norm}[1]{\left\|#1\right\|}

\newcommand{\abs}[1]{\left|#1\right|}
\newcommand{\paren}[1]{\left(#1\right)}
\newcommand{\braces}[1]{\left\{#1\right\}}

\newcommand{\pr}[1]{\mathbb{P}\left(#1\right)}
\newcommand{\prf}[2]{\mathbb{P}_{#1}\left(#2\right)}
\newcommand{\prob}{\mathbb{P}}

\newcommand{\AQ}{\mathcal{\ddot E}_Q}

\usepackage{dsfont}
\DeclareSymbolFont{bbold}{U}{bbold}{m}{n}
\DeclareSymbolFontAlphabet{\mathbbold}{bbold}
\newcommand{\ind}[1]{{\mathbbold 1}\!\left\{#1\right\}}

\renewcommand{\dim}{d}

\DeclareMathOperator*{\argmin}{argmin}

\newcommand{\diam}{\text{diam}}
\DeclareMathOperator*{\Expectation}{\expec}
\DeclareMathOperator*{\Prob}{\prob}

\newcommand{\nats}{\mathbb{N}}

    \newcommand{\hstar}{h^{\!*}}
\newcommand{\hstarP}{\hstar_P}
\newcommand{\hstarQ}{\hstar_Q}

\def\ERM{\mathrm{ERM}}

\def\hhat{\hat{h}}

\def\conf{\hat{\Hyp}}

\newsavebox{\savepar}

\newcommand{\ignore}[1]{}

\title{Help Me Help You: The Aggregate Value of Source and Target Data in Transfer Learning}

\author{
 Steve Hanneke \\
  Purdue University\\
  \texttt{steve.hanneke@gmail.com} \\
   \And
 Samory Kpotufe \\
  Columbia University \\
  \texttt{samory@columbia.edu}
}

\begin{document}

\maketitle

\begin{abstract}
Transfer learning, in regression or classification, concerns problems where one aims to use imperfect data from a source distribution $P$ to improve prediction w.r.t.\ a target distribution $Q$. While there is now much theoretical understanding of the \emph{unsupervised} setting of this problem (where no labeled target data is available), the supervised setting with labeled target data remains under-studied.

We identify and characterize an interesting dichotomy in the rates achievable in the supervised setting: there is a {\it weak transfer regime} where optimal rates are determined by information in either the source or target dataset but not both, and a {\it strong transfer regime} where much faster target rates are achievable, accounting for complementary information in source and target data. Furthermore, while previous theoretical tools for the unsupervised setting can be extended to capture 
the weak transfer regime, 
we argue that they cannot capture the {strong} regime due to limitations inherent in the various notions of relatedness they rely on. 
Instead, we provide an appropriate refinement of relatedness notions---termed \emph{weak and strong moduli of transfer}---that reveal the two regimes and lead to a nontrivial characterization of which pairs of distributions $P$ and $Q$ adhere to which regime. 

Finally, the above dichotomy reveals a new categorization of existing algorithmic approaches in terms of their optimality or lack-thereof across supervised regimes, and we provide a generic approach that remains nearly optimal in all regimes without prior distributional knowledge. 

\end{abstract}


\section{Introduction}

\emph{Transfer Learning}, in regression or classification, refers to the problem of improving prediction performance w.r.t. a target distribution $Q$ by relying on imperfect data from a related \emph{source} distribution $P$. The \emph{unsupervised} setting of this problem, where no labeled data from the target $Q$ is available, has received much theoretical attention for over a decade, and is now relatively well understood. However, the \emph{supervised} setting, where the learner also has access to some amount of labeled data from $Q$, remains under-studied. As it turns out, the supervised setting displays unexpected regimes of rates in terms of nuanced relationships between $P$ and $Q$, not well captured by previous works. 

For initial intuition on achievable rates in supervised transfer, we may consider as a baseline, the rates achievable by an oracle that relies on either unsupervised transfer from $P$ (and ignores the $Q$ data) or supervised learning from $Q$ (and ignores the $P$ data), whichever yields the best performance under $Q$. One may hope to be able to do better than such an oracle by properly aggregating $P$ and $Q$ data. In this work we show that there is in fact an interesting dichotomy with two regimes of rates. First, there is a \emph{weak transfer regime} where optimal rates are no better than that of such an oracle as above which ignores one of the datasets; interestingly, the weak regime appears to capture the limit of what is expressible by the bulk of traditional notions of relatedness that have appeared in the literature on transfer. Second, there is a \emph{strong transfer regime} where much faster rates are achievable, accounting for complementary information in both datasets. As an example of such complementary information, 
it may be that a small amount of $Q$ data helps prune away predictors that are good under $P$ but are quite suboptimal under $Q$ (e.g., predictors based on spurious features), while a large amount of $P$ data helps converge to a remaining predictor that is good under both $P$ and $Q$.

The two regimes are determined by how excess risks (over a hypothesis class $\Hyp$) measured under $P$ relate to the same as measured under $Q$. In particular, our formal results are expressed in terms of so-called \emph{weak and strong moduli of transfer}, two new concepts which allow for a generic treatment of the relatedness between source and target. The weak modulus describes the extent to which low excess risk on the source $P$ translates to low excess risk on the target $Q$. 
As we will see, it (tautologically) captures the best performance achievable in unsupervised transfer learning, while also capturing the best rates achievable in the weak transfer regime.
Furthermore, it also lower-bounds many established measures of relatedness appearing in the literature on transfer learning, and thus describes the limits of transfer performance expressible by such measures. 
On the other hand, the strong modulus \emph{localizes} the weak modulus over predictors with low target excess risk, and thus more explicitly integrates the contribution of target data with that of the source data. 
The strong modulus captures both the strong and weak transfer regimes: it expresses rates in supervised transfer always no worse than those expressible by the weak modulus, and sometimes strictly better.

Now, the strong transfer regime corresponds to situations---i.e., source-target pairs $P, Q$---where the rates expressible by the strong modulus can be strictly faster than those expressible by the weak modulus. 
A main contribution of this work is to completely characterize such situations, and therefore also situations where the weak and strong modulus express the same rate, i.e., where optimal rates are no better than that of the naive oracle that ignores one of the two datasets. 
For instance, certain classification problems with spurious features will fall into the strong transfer regime, while interestingly, any setting with a convex loss over a convex hypothesis class (such as in linear regression) will always fall in the weak transfer regime.

The above dichotomy leads to a new categorization of existing heuristics in terms of their optimality in either regime: for instance common heuristics such as \emph{naive pooling} (i.e. ERM on the combined datasets) can be suboptimal in either regime in the presence of noise, while a procedure akin to \emph{fine-tuning} (which trains first on the $P$ data then fine-tunes on the $Q$ data) is optimal in the weak regime when $P$ and $Q$ share optimal predictors, but can be sub-optimal otherwise in both regimes. It is then natural to ask whether it is possible to achieve optimal oracle rates in both regimes adaptively. We show that this is indeed possible, and derive new algorithmic principles---rooted in a reduction to certain notions of \emph{confidence sets}---that can guarantee optimal rates in both regimes, \emph{adaptively}, i.e., with no prior distributional knowledge of the relation between $P$ and $Q$.  

\paragraph{Formal Overview of Results.} We consider a generic prediction setting with distributions $P$ and $Q$ on $\X\times\Y$, a fixed hypothesis class 
$\Hyp$ of functions $h: \X \to \Y$, and loss function $\ell : \Y^2 \to \reals_+$. We allow general risks of the form $R_\mu(h) = \Expectation_\mu \ell(h(X), Y)$ for $\mu \equiv P$ or $Q$; of special interest is the \emph{excess risk} 
$\E_\mu(h) \doteq R_\mu(h) - \inf_{h'\in \Hyp}R_\mu(h')$. Risk minimizers for $P$ and $Q$ need not exist, or be the same. The learner has access to $n_P> 0$ labeled data from $P$ and $n_Q \geq 0$ labeled data from $Q$.

$\bullet$ The \emph{weak modulus of transfer} is given as the function 
$\mod(\epsilon) \doteq \sup\braces{\E_Q(h): h \in \Hyp, \ \E_P(h) \leq \epsilon}.$
Clearly, in the unsupervised case, we have the tautology that if a learner $\hhat_P$ trained on source sample achieves $\E_P(\hhat_P) \leq \epsilon_P$, {for some rate $\epsilon_P$, then it achieves $\E_Q(\hhat_Q) \leq \mod(\epsilon_P)$. It is therefore not surprising that it is implicitly tied to many existing notions of relatedness designed for unsupervised transfer: popular examples are 
refinements of \emph{total-variation} for domain adaptation in classification \citep{mansour2009domain, ben2010theory}, Wasserstein distances in regression and classification \citep{redko2017theoretical, shen2018wasserstein}, or measures relating covariance structures across $P$ and $Q$ in linear regression \citep{mousavi2020minimax, zhang2022class, ge2023maximum}. Exact relations are discussed in Section \ref{sec:weakmod-upper-bounds}. \emph{What is unclear however is the extent to which the weak modulus captures the best risks achievable in supervised transfer}, if at all, since it does not account for target data.

We first show via minimax results in classification and regression that the best rate expressible (for supervised transfer) in terms of the weak modulus takes the form
\begin{align}
\min\braces{\epsilon_Q, \mod(\epsilon_P)}, \label{eq:intro-rates}
\end{align}
where $\epsilon_\mu$, for $\mu$ denoting $Q$ or $P$, is roughly the best rate achievable {under $\mu$} given $n_\mu$ samples from $\mu$;
for example in classification, typically $\epsilon_\mu \approx n_\mu^{-1/2}$ or faster, while in linear regression with squared loss, $\epsilon_\mu \approx n_\mu^{-1}$. In other words, the best expressible rate in terms of $\mod(\cdot)$ (or in terms of notions of relatedness that it lower-bounds) is 
that of a naive oracle that ignores the worst of the two datasets from $P$ and $Q$.

We can now ask whether the rates in \eqref{eq:intro-rates} are improvable, which leads to the strong modulus.

$\bullet$ The \emph{strong modulus of transfer}  
follows as a natural refinement of the weak modulus, and aims to better capture situations where the aggregate of source and target data might yield strictly better rates than \eqref{eq:intro-rates}, i.e., the best of source or target data alone.  For intuition, suppose, e.g., that the source distribution $P$ admits two distinct risk minimizers ${\hstar_{P, 1}}, {\hstar_{P, 2}}$, the first with $\E_Q({\hstar_{P, 1}})\approx 0$, while the second has large {$\E_Q({\hstar_{P, 2}}) \geq \frac{1}{4}$.}
For example, ${\hstar_{P, 2}}$ might focus on spurious features (e.g. background color in object classification) that are predictive for data from $P$, but are harmful for prediction under $Q$. Then, with enough $Q$ data, this might be evident, so that we may avoid such suboptimal solutions,  
while the value of the weak-modulus is limited by $\hstar_{P, 2}$. With this intuition in mind, a simplified version of the strong modulus roughly takes the form 
$$\mod(\epsilon_1, \epsilon_2) \doteq \sup \braces{\E_Q(h): h \in \Hyp, \ \E_Q(h) \leq \epsilon_1 \text{ and } \E_P(h) \leq \epsilon_2},$$
for a range of values of $\epsilon_1, \epsilon_2$ (the precise definition, stated in Section \ref{sec:prelim-weak-and-strong}, is expressed more-rigorously as localization under $Q$ of the weak modulus). 
We can express tighter rates for supervised transfer, in terms of the strong modulus, namely, oracle rates are of the form 
\begin{align}
\E_Q(\hat h) \lesssim \mod(\epsilon_Q, \epsilon_P) \leq \min \braces{\epsilon_Q, \mod(\epsilon_P)} \label{eq:introRateStrong}
\end{align} As shown in Theorem \ref{thm:strong-modulus-learning-lower-bound}, these rates are tight even in situations where $\mod(\epsilon_Q, \epsilon_P) \ll \min \braces{\epsilon_Q, \mod(\epsilon_P)}$, i.e., situations where there are gaps between moduli. 
\emph{This makes precise the notions of weak and strong transfer regimes discussed above:} namely, the weak transfer regime denotes pairs $P$, $Q$ for which $\mod(\epsilon_1,\epsilon_2) = \min\{\epsilon_1,\mod(\epsilon_2)\}$ $\forall \epsilon_1,\epsilon_2$, whereas the strong transfer regime denotes pairs $P$, $Q$ for which there exist $\epsilon_1, \epsilon_2$ for which $\mod(\epsilon_1,\epsilon_2) < \min\{\epsilon_1,\mod(\epsilon_2)\}$.

$\bullet$ We provide a \emph{full characterization of such gaps} between weak and strong moduli (Theorem \ref{thm:gap-iff-non-monotonic}), i.e., a description of transfer problems where the strong modulus results in strictly better rates. Namely, we show that such gaps exist iff certain \emph{monotonicity} relations between source and target risks hold. This characterization yields new insights into the implicit assumptions underlying many existing relatedness measures, revealing that they only provide tight rates for \((P, Q)\) pairs that satisfy monotonicity conditions. 

We observe several interesting implications of these results.
On one hand, common convex settings---with convex loss and convex hypothesis class---always yield \emph{monotone} transfer problems, with no gap between weak and strong moduli, and therefore are well captured by the weak modulus. On the other hand, some common feature selection settings in regression or classification, as often motivated via causal assumptions \citep{simon1954spurious, Sagawa*2020Distributionally}, are non-monotonic, and thus display strict gaps.

$\bullet$ Finally, new \emph{algorithmic principles} emerge from the analysis as so far alluded to. 
Since in general the relatedness measure $\mod(\cdot,\cdot)$ is a priori unknown and is impractical to estimate, 
we are particularly interested in the existence of \emph{adaptive procedures} that can achieve the oracle rates $\mod(\epsilon_Q, \epsilon_P)$ with no prior distributional knowledge. We show that such adaptive procedures exist, via a reduction to so-called \emph{strong confidence sets}, namely, estimated subsets of the hypothesis class that retain good predictors, up to the variability in the data. Accounting for such variability in both source and target data turns out crucial as we can also show that many existing algorithmic approaches fail to adaptively achieve the same low risks. For instance, perhaps surprisingly, (a) selecting between estimated source and target predictors by cross-validation with target data, or (b) ERM on the combined source and target data, both already fail to be adaptive to the weak modulus rates of \eqref{eq:intro-rates}, let alone the improved rates in \eqref{eq:introRateStrong}. Procedures akin to \emph{fine-tuning} can achieve \eqref{eq:intro-rates} sometimes, e.g., if source and target share an optimal predictor---, but as is well known, may run into \emph{negative transfer} otherwise (which is not possible if \eqref{eq:intro-rates} is achieved). Some recently proposed robust validation technique (\`a la Lepski) of \citep{hanneke2019value} can always achieve \eqref{eq:intro-rates}, but fails to achieve \eqref{eq:introRateStrong}. 
We discuss all of these procedures in more detail in Section~\ref{sec:other-algorithms}.

While the adaptive procedures derived in the analysis are for the most part of a theoretical nature, we hope that the resulting design principles, i.e., our reduction to confidence sets, may yield insights into practical procedures that can automatically leverage the combined information in the source and target data while always avoiding \emph{negative transfer} (situations where the source should not be relied upon \citep{zhang2022survey}).
For instance, it is easily seen that in linear regression with squared loss, our generic approaches result in efficiently solvable convex programs (see Remark \ref{rem:convex-program}). 
This is potentially the case in more general settings with tractable losses, but requires further investigation.

\paragraph{Other Related Works.} 
As alluded to so far, the strong modulus appears related in spirit to recent works on robustness against spurious features and invariant risk minimization \citep[see e.g.,][]{baktashmotlagh2013unsupervised, gong2016domain, arjovsky2019invariant, wu2023prominent, heinze2021conditional, kostin2024achievable}, as it similarly aims to distinguish between good hypotheses under the source in terms of how much their risks vary under the target distribution. However, a key distinction is that, by considering a continuum of possible excess risks, the strong modulus allows us to integrate in the statistical uncertainty inherent in finite-sample regimes. 

Our notions of \emph{moduli of transfer} focus on relations between source and target \emph{risks}, and thus, do not capture the full spectrum of notions of relatedness that have been considered in the context of transfer learning and related tasks. For instance, some notions not captured are those tied to a specific problem setting or algorithms considered---this is often the case in nonparametric settings; others are notions that describe more structural relationships between regression functions. Some relevant works are discussed next. 

Starting with nonparametric regression or classification settings, various works \citep[see, e.g., ][]{kpotufe2021marginal, cai2021transfer, scott2019generalized, reeve2021adaptive, pathak2022new, galbraith2023classification}, consider structural conditions on the change in regression function and or covariate distributions, that appear only indirectly related to changes between source and target risks. 

A different line of work considers more general structural assumptions relating source and target \citep[see, e.g.][]{blitzer2011domain, balcan2019provable, du2020few, tripuraneni2020theory,ben-david:03,ando2005framework,Baxter-model,Baxter-Bayesian,thrun:98,yang2013theory,aliakbarpour2024metalearning,alon2024erm,meunier2023nonlinear}, e.g., regression functions might share the same underlying index subspace (as related to \emph{representation learning}), leading to decreases in problem complexity which are not captured by this work.

On the algorithmic side, various theoretical works are related in spirit to the present work in that they also address the problem of \emph{negative transfer}, which as discussed above, is not our main goal but is automatically addressed here as a consequence of our adaptivity results. 
 For instance, \citep{klivans2024testable} introduces a testing paradigm where one is to reject the source if it leads to large target error;  however they do not aim to adapt to every situation where the source is useful, as their tests rely on unlabeled target data and therefore do not leverage label information when available. As another example, for the problem of mean estimation,  \citep{blanchard:24} tests for source distributions whose first and second moments are sufficiently close to that of the target to help improve performance. As such they aim to be adaptive to a given discrepancy measure encoded by difference in moments. It remains unclear when such discrepancies are appropriate, beyond the mean estimation problem.

\section{Preliminaries}
\label{sec:preliminaries}

 \subsection{Basic Definitions} 
\label{sec:basic-definitions}

Let $X, Y$ be jointly distributed according to some measure $\mu$ (later $P$ or $Q$), where $X$ is in some domain $\X$ and $Y\in \Y$. A \emph{hypothesis class} is a set $\Hyp$ of measurable functions $\X \to \Y$. Given a 
loss function $\loss: \Y^2 \mapsto \real_+$, we consider risks 
$R_\mu(h) \doteq \expec_\mu \ \loss(h(X), Y)$, as measured under a given $\mu$. 

\begin{assumption}[Finite Risks]
\label{asn:finiteness}
We assume throughout that $R_\mu(h) < \infty, \forall h \in \Hyp$, and for measures $\mu$ considered.
\end{assumption}

Note that the above is a mild assumption, as it does not require the loss to be uniformly bounded.

For example, the case of binary classification corresponds to $\Y = \{\pm 1\}$, where we often choose 
$\loss(y,y') = \ind{ y \neq y' }$, 
and results below will depend on the VC dimension of the class $\Hyp$ (see, e.g., \citep{VC:72,koltchinskii:06}).
As another example, the case of regression corresponds to $\Y \subset \reals$, 
and we may choose $\loss(y,y') = ( y - y' )^2$, 
and results below may depend, for instance, on the covering numbers or pseudo-dimension of the class $\Hyp$ (see, e.g. \citep{anthony1999learning}).

\begin{remark}
\label{rem:abstractness}
Our general results will in fact capture quite abstract dependences on $\Hyp$, 
via an abstractly-defined notion of confidence sets,
and any notion of complexity that allows for such confidence sets are therefore admissible.
\end{remark}

\begin{definition} 
The {\bf excess risk w.r.t. a (non-empty) subclass $\Hyp_0 \subset \Hyp$}, and a joint distribution $\mu$, is defined as 
$$\E_\mu(h; \Hyp_0) \doteq R_\mu(h) - \inf_{h' \in \Hyp_0} R_\mu(h').$$
In particular, we will just refer to $\E_\mu(h; \Hyp)$ as \emph{excess risk}, and often write 
$\E_\mu(h)$ in this case for simplicity. 
\end{definition} 

We will need the following useful definition \citep{koltchinskii:06}. Let $\exreals = (0,\infty]$ denote the extended positive reals. 

\begin{definition}[$\epsilon$-Minimal Set]
   For any distribution $\mu$, and $\epsilon \in \exreals$, define 
   $\Hyp_\mu(\epsilon) \doteq \braces{ h \in \Hyp : \E_\mu(h) \leq \epsilon }$. 
\end{definition}

Clearly, $\Hyp_\mu(\infty)$ is just $\Hyp$, so the inclusion of $\infty$ is just for convenience as we will see later.  

\paragraph{Transfer Setting.} 
We consider \emph{source} and \emph{target} distributions $P$ and $Q$ for $(X, Y)$, where we let $\E_P, \E_Q$ denote excess-risks under $P$ and $Q$. We are interested in excess risk $\E_Q(\hat h)$ of predictors $\hhat$ trained jointly on $n_P$ i.i.d samples $S_P$ from $P$, and $n_Q$ i.i.d. samples $S_Q$ from $Q$. Which $\E_Q$ is achievable necessarily depends on the \emph{discrepancy} $P\to Q$ appropriately formalized.

\subsection{Weak and Strong Moduli of Transfer}
\label{sec:prelim-weak-and-strong}

This section presents the two key definitions central to this work.

\begin{wrapfigure}{r}{0.17\textwidth}
  \begin{center}
    \includegraphics[width=0.17\textwidth]{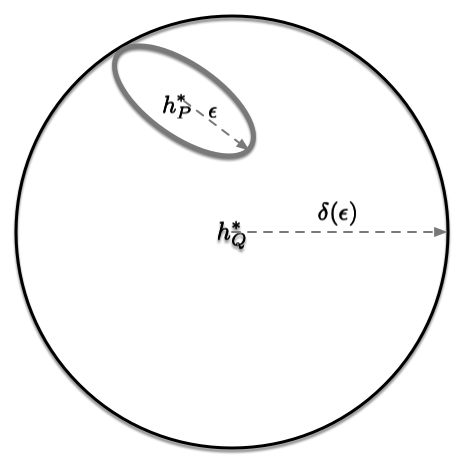}
  \end{center}
  \vspace{-3mm}\caption{\small The weak modulus $\mod(\epsilon)$ is illustrated in terms of risk minimizers $\hstar_P, \hstar_Q$. The modulus is defined by the \emph{smallest} $\Hyp_Q(\cdot)$ (ball centered at $\hstar_Q$) containing $\Hyp_P(\epsilon)$ (ellipsoid centered at $\hstar_P$).}
    \label{fig:weakModulus}
  \vspace{-1cm}
\end{wrapfigure}

\paragraph{The Weak Modulus of Transfer.}
The first notion considered below serves to capture the reduction in target $Q$-risk induced by small source $P$-risk, i.e., by a potentially large amount of $P$ data. 

\begin{definition}[Weak Modulus] \label{def:weak} 
Let $\epsilon > 0$.
Define the modulus
$$\mod(\epsilon) \doteq \sup \braces{\E_Q(h): h\in \Hyp_P (\epsilon)},$$
or written as $\mod_{\PQ}(\epsilon)$ when the dependence on distributions is to be made explicit.
\end{definition}

In the results below involving the weak modulus, 
the argument $\epsilon$ corresponds to the best risk achievable under $P$ given a fixed amount $n_P$ of data from $P$ (see illustration of Figure \ref{fig:weakModulus}).
As such, the weak modulus essentially describes, 
in a sharpest possible form, the rates of transfer 
\emph{when there is no target data}.

As we will argue in Section~\ref{sec:weakmod-upper-bounds}, 
the above simple definition already captures the bulk of notions of discrepancies proposed in the literature on transfer learning.

\paragraph{The Strong Modulus of Transfer.}
While the weak modulus sharply captures the rates of 
transfer when there is no target data, 
we next present a refinement of the modulus,
motivated by scenarios in which we have 
access to some target data.
It aims to capture the reduction in $Q$-risk induced by the combined effect of target $Q$ and source $P$ data.
In particular, the intuition is that, 
by using some $Q$ data to prune away excessively 
high $Q$-risk functions, it is sometimes possible that 
the remaining functions have a better rate of 
transfer of $P$ to $Q$ (i.e., smaller weak modulus).
This idea is not reflected in any of the previous notions of discrepancy in the literature.
Formally, we have the following definition.

\begin{wrapfigure}{r}{0.27\textwidth}
\vspace{-2.8cm}
  \begin{center}
    \includegraphics[width=0.27\textwidth]{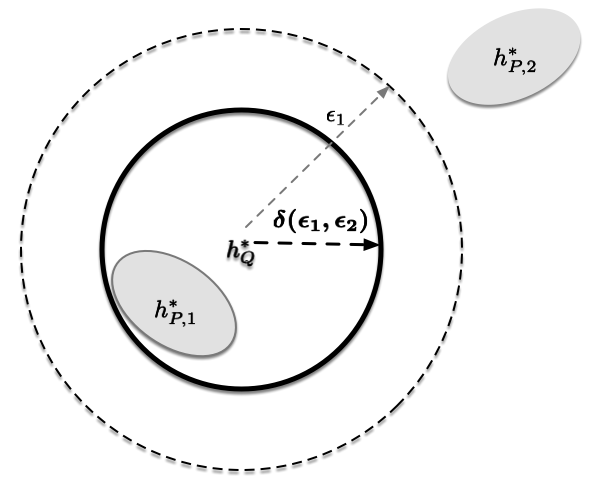}
  \end{center}
  \caption{\small Strong Modulus. The two ellipsoids centered at $\hstar_P$'s together represent $\Hyp_P(\epsilon_2)$, while the dotted ball centered at $\hstar_Q$ represents $\Hyp_Q(\epsilon_1)$. 
  The worst-case $Q$-risk in the inner ellipsoid determines $\mod(\epsilon_1, \epsilon_2)$.
    }
    \label{fig:strongModulus}
\end{wrapfigure}

\begin{definition}[Strong Modulus] \label{def:strong}
For $\epsilon_1,\epsilon_2 >0$, 
define the bivariate modulus: 
\begin{equation} 
{\hskip -2.3cm}\mod(\epsilon_1,\epsilon_2) \doteq \sup\braces{ \E_Q(h) : h \in \Hyp_Q(\epsilon_1), \quad \E_P(h;\Hyp_Q(\epsilon_1)) \leq \epsilon_2 }, 
\end{equation}
or written as $\mod_{\PQ}(\epsilon_1, \epsilon_2)$ when the dependence on $P, Q$ is to be made explicit.
\end{definition}

In our bounds on $Q$ excess risk, the parameters $\epsilon_1$ and $\epsilon_2$ will be instantiated as functions of the amounts $n_Q$ and $n_P$ of $Q$ data and $P$ data, respectively. In particular, as illustrated in Figure \ref{fig:strongModulus}, we have in mind situations where data from $Q$ allow us to reject hypotheses $h'$s that are good under $P$ but bad under $Q$ (see Example \ref{ex:gap} of Section \ref{sec:examplesStrongMod}, as motivated by settings with spurious correlations between irrelevant $X$ features and $Y$.

As we will see, bounds in terms of the strong modulus will refine those based on the weak modulus, 
sometimes reflecting strict improvements over the guarantees that can be expressed purely in terms of the weak modulus.
As such, we will be particularly interested in understanding what conditions on $(Q,P)$ enable such strict improvements.
We will in fact provide a complete characterization of such situations, in Section~\ref{sec:strict-improvements}.

\subsection{Properties and Relation between Moduli}
\label{sec:additional-propositions}

Our main aim in this section is to establish certain relationships between weak and strong moduli, namely as expressed in Corollary \ref{cor:weak-vs-strong} below, that are relevant in comparing achievable rates expressed in terms of moduli. The following quantity turns out useful in establishing this result; it will in fact also appear later in Section \ref{sec:strongConfBounds} when discussing adaptive rates in terms of the strong modulus. It can loosely be viewed as the best $Q$-excess risk of $P$-risk minimizers (see Remark \ref{rem:pivotIntuition}). 

\begin{definition}[Lower Pivot] \label{def:pivotalValue}
The following captures the infimum excess $Q$-risk of good classifiers under $P$: 
$$\pivot \doteq \lim_{\epsilon \to 0} ~\inf \!\braces{\E_Q(h): h \in \Hyp_P(\epsilon)} .$$ 
\end{definition}

We also consider an \emph{upper pivot} counterpart to Definition~\ref{def:pivotalValue}.
Although stated here for contrast, it will turn out useful later in Section \ref{sec:weakmod-upper-bounds} when considering relations to measures of discrepancy, and will also be crucial in Section \ref{sec:strict-improvements} where we characterize the gap between weak and strong strong moduli.

\begin{definition}[Upper Pivot]
\label{defn:upper-pivot}
We denote by $\pivot^\sharp \doteq 
\lim\limits_{\epsilon\to 0} ~\sup \!\braces{\E_Q(h) : h \in \Hyp_P(\epsilon)} = 
\lim\limits_{\epsilon\to 0}\mod(\epsilon)$. 
\end{definition}

This quantity may be viewed loosely as the worst $Q$ excess-risk of $P$-risk minimizers if they exist. 
Note that, by definition, $\pivot^\sharp \geq \pivot$.

\begin{remark}[Intuition on pivots] \label{rem:pivotIntuition} Consider the simple case of a finite $\Hyp$, and denote the set of $P$-risk minimzers as $\Hyp_P^* \doteq \argmin_{h\in \Hyp} R_P(h)$. We see that 
$$\pivot = \min_{\hstarP \in \Hyp_P^*} \E_Q(\hstarP) \text{ while } \pivot^\sharp \doteq \lim_{\epsilon\to 0}\mod(\epsilon) = \max_{\hstarP \in \Hyp_P^*} \E_Q(\hstarP).$$

 The pivots capture the above intuition which is given in terms of risk minimizers, but their definitions allow for general problem settings that might not admit risk minimizers. 
\end{remark}

Next we consider an alternative expression of $\pivot$, which is easier to use at times. For intuition, again consider the simplest case of a finite $\Hyp$; we see that $\pivot$ is the smallest $\epsilon_1$ for which $\Hyp_Q(\epsilon_1)$ contains an $\hstarP$ (whereas $\pivot^\sharp$ represents the smallest $\epsilon_1$ for which $\Hyp_Q(\epsilon_1)$ contains \emph{every} $\hstarP$). The proposition below generalizes this intuition.

\begin{proposition}[Equivalent Form of the Pivot]
\label{prop:pivotal-equivalence} We have 
$\pivot  =  \inf \braces{\epsilon_1 \in \exreals \!\!: \forall \epsilon > 0, 
\Hyp_Q(\epsilon_1) \cap \Hyp_P(\epsilon) \neq \emptyset}.$
\end{proposition}
\begin{proof}
For ease of exposition, let  
$\EPQ$ denote the set $\braces{\epsilon_1 \in \exreals : \forall \epsilon > 0, 
\Hyp_Q(\epsilon_1) \cap \Hyp_P(\epsilon) \neq \emptyset}.$

For any $\epsilon_1 > \pivot$, we have that 
$\epsilon_1 > \inf \!\braces{\E_Q(h): h \in \Hyp_P(\epsilon)}$ for all $\epsilon > 0$, since this last infimum only increases as $\epsilon \to 0$; in other words, for all $\epsilon > 0$, $\Hyp_Q(\epsilon_1) \cap \Hyp_P(\epsilon) \neq \emptyset$, i.e., $\epsilon_1 \in \EPQ$. We therefore have that $\pivot \geq \inf \EPQ$. Notice that, this last inequality also holds when $\pivot = \infty$. 

Now for any $\epsilon_1 \in \EPQ$, by definition, we have that for all $\epsilon > 0$, 
$\inf_{h \in \Hyp_P(\epsilon)} \E_Q(h) \leq \epsilon_1$, in other words, $\epsilon_1 \geq \pivot$. We therefore also have that $\pivot \leq \inf \EPQ$.  
\end{proof}

\begin{remark}
We note that we may have $\pivot$ very large, in fact matching $\sup_{a, b \in \Y} \ell(a, b)$ (admitting $\infty$ for unbounded losses). This describes $P$ \emph{having little information on} $Q$, i.e., where all $h'$s with low $P$-risk have large $Q$-risk. 
\end{remark}

The above proposition leads to a useful simpler expression of the strong modulus for arguments $\epsilon_1$ above $\pivot$.

\begin{corollary}
\label{cor:pivotal-value}
Let $\epsilon_1,\epsilon_2 > 0$.
\begin{itemize}
\item If $\epsilon_1 > \pivot$, 
every $h \in \Hyp$ has 
$\E_P(h;\H_Q(\epsilon_1)) = \E_P(h)$. It follows that, for $\epsilon_1 > \pivot$
$$\mod(\epsilon_1,\epsilon_2) = \sup\braces{ \E_Q(h) : h \in \Hyp_Q(\epsilon_1) \cap \Hyp_P(\epsilon_2)}.$$
\item If $\epsilon_1 \leq \pivot$, then $\epsilon_1 \leq \mod(\epsilon_2)$.  Equivalently, we always have $\pivot \leq \mod(\epsilon_2)$.
\end{itemize}
\end{corollary}
\begin{proof} 
For the first claim, 
for $\epsilon_1 > \pivot$, 
we have $\epsilon_1 \in \EPQ$
(where $\EPQ$ is defined in the proof of Proposition~\ref{prop:pivotal-equivalence}). Now by definition, $\EPQ$ is the set of $\epsilon_1$ such that for all $\epsilon$, $\exists h \in \Hyp_Q(\epsilon_1)$ satisfying $\E_P(h) \leq \epsilon$. 
Therefore, $\inf_{h \in \Hyp_Q(\epsilon_1)} R_P(h) = \inf_{h \in \Hyp} R_P(h)$, and the claim follows.

For the second claim, 
we note that $\mod(\epsilon_2) \in \EPQ$, 
since any $\epsilon > 0$ has  
$\Hyp_P(\epsilon) \cap \Hyp_P(\epsilon_2) = \Hyp_P(\min\{\epsilon,\epsilon_2\}) \neq \emptyset$,
and is a subset of $\Hyp_P(\epsilon_2) \subset \Hyp_Q(\mod(\epsilon_2))$, 
where the last inclusion is from
the definition of $\mod(\epsilon_2)$.
Hence $\pivot \leq \mod(\epsilon_2)$. 
\end{proof} 

The following corollary states that the strong modulus represents a refinement of the weak modulus.

\begin{corollary}[Relation Between Strong and Weak Moduli]
\label{cor:weak-vs-strong}
For $\epsilon_1, \epsilon_2 > 0$, 
$\mod(\epsilon_1,\epsilon_2) \leq \min\braces{ \epsilon_1, \mod(\epsilon_2)}. $
\end{corollary}
\begin{proof}
If $\epsilon_1 > \pivot$, 
then Corollary~\ref{cor:pivotal-value} implies 
$\mod(\epsilon_1,\epsilon_2) = \sup\braces{ \E_Q(h) : h \in \Hyp_Q(\epsilon_1) \cap \Hyp_P(\epsilon_2)}$, 
which is bounded by $\min\braces{\epsilon_1,\mod(\epsilon_2)}$
by definition of $\Hyp_Q(\epsilon_1)$ and $\mod(\epsilon_2)$.

If $\epsilon_1 \leq \pivot$, 
then since Corollary~\ref{cor:pivotal-value} implies 
$\epsilon_1 \leq \mod(\epsilon_2)$ in this case, 
we have $\min\braces{\epsilon_1,\mod(\epsilon_2)} = \epsilon_1 \geq \mod(\epsilon_1,\epsilon_2)$, where the final inequality is immediate from the definition of $\mod(\epsilon_1,\epsilon_2)$.
\end{proof}

In Section~\ref{sec:strict-improvements}, 
we characterize situations with \emph{gaps}, i.e., where where we may have $\mod(\epsilon_1,\epsilon_2) \ll \min\braces{ \epsilon_1, \mod(\epsilon_2)}$. 

\paragraph{Monotonicity of Moduli.} We conclude by stating basic monotonicity 
properties of the weak and strong moduli.
The following proposition holds trivially since $\Hyp_P(\epsilon)\subset \Hyp_P(\epsilon')$, for $\epsilon\leq \epsilon'$, but is given for emphasis. 

\begin{proposition}
The weak modulus $\mod(\epsilon)$ is non-decreasing in $\epsilon$, i.e., for all $\epsilon \leq \epsilon'$ it holds that $\mod(\epsilon) \leq \mod(\epsilon')$. 
\end{proposition}

It turns out the strong modulus is also non-decreasing in both arguments. While monotonicity was immediate by definition for the weak modulus $\mod(\cdot)$, it is a bit more involved for the strong modulus $\mod(\cdot, \cdot)$. 

\begin{proposition}
\label{prop:double-mod-monotone}
The strong modulus $\mod(\epsilon_1,\epsilon_2)$ is non-decreasing, i.e., 
$\forall \epsilon_1 \leq \epsilon'_1$, $\forall \epsilon_2 \leq \epsilon'_2$
we have $\mod(\epsilon_1,\epsilon_2) \leq \mod(\epsilon'_1,\epsilon'_2)$.
\end{proposition}
\begin{proof} 
The fact that $\mod(\epsilon'_1,\epsilon_2) \leq \mod(\epsilon'_1,\epsilon'_2)$ is immediate from the definition.
It remains to show $\mod(\epsilon_1,\epsilon_2) \leq \mod(\epsilon'_1,\epsilon_2)$. 
We first note the main difficulty in this argument arises 
when $\epsilon_1 \leq \pivot$.
In particular, the conditions on $\E_P(h;\Hyp_Q(\epsilon_1))$ and $\E_P(h;\Hyp_Q(\epsilon_1'))$ may involve different $P$-infimum errors, hence the resulting sets become hard to compare. 

So, first, assume 
$\inf_{h \in \Hyp_Q(\epsilon_1)}R_P(h) = \inf_{h \in \Hyp_Q(\epsilon_1')}R_P(h)$ (this is the case, e.g., when $\epsilon_1 > \pivot$, see Corollary \ref{cor:pivotal-value}). It then follows that the set $\{h\in \Hyp_Q(\epsilon_1): \E_P(h;\Hyp_Q(\epsilon_1)) \leq \epsilon_2\} $ is contained in $\{h\in \Hyp_Q(\epsilon_1'): \E_P(h;\Hyp_Q(\epsilon_1')) \leq \epsilon_2\} $, hence $\mod(\epsilon_1,\epsilon_2) \leq \mod(\epsilon'_1,\epsilon_2)$. 

Now, suppose the contrary. Then 
$\exists h \in \Hyp_Q(\epsilon_1')$ such that 
$$\inf_{h' \in \Hyp_Q(\epsilon_1')}R_P(h') \leq R_P(h) < \min \braces{ \inf_{h' \in \Hyp_Q(\epsilon_1)}R_P(h'), \inf_{h' \in \Hyp_Q(\epsilon_1')}R_P(h') + \epsilon_2}.$$
By the second inequality, we have that $h \notin \Hyp_Q(\epsilon_1)$, hence 
$$\mod(\epsilon_1, \epsilon_2) \leq \epsilon_1 \leq \E_Q(h) \leq \mod(\epsilon_1', \epsilon_2).$$
\end{proof}

\section{Rates in terms of Moduli of Transfer}

Our aim in this section is to show that it is possible to (adaptively) achieve rates sometimes faster than the best rates achievable in terms of popular measures of discrepancy. 

As such, we start in Section \ref{sec:weakmod-upper-bounds} below by showing that the weak modulus indeed provides a unified language for popular notions of discrepancy in domain adaptation. 
Next, in order to understand the gains inherent in adapting to the strong modulus rather than to the weak modulus, we outline in Section \ref{sec:weak-lower-bounds} the statistical limitations of the weak modulus---and therefore of the many discrepancies it captures. 
This is then followed in Section \ref{sec:strongConfBounds} by adaptive rates in terms of the strong modulus via an algorithmic reduction to the existence of certain confidence sets. For completeness, examples of such confidence sets are derived in Section \ref{sub:examples-StrongConf}. 
Finally we end in Section \ref{sec:other-algorithms} with a discussion of how previous algorithmic techniques fail to adapt to the strong modulus.

\subsection{Weak Modulus Captures Many Existing Discrepancies}\label{sec:weakmod-upper-bounds}

Here we consider a few examples of existing notions of discrepancy between source and target $P, Q$, and illustrate the types of bounds they imply on $\mod(\epsilon)$. These bounds were already implicit in past work, even while the weak modulus $\mod(\epsilon)$ was never explicitly defined as the main object of study, or as the implied notion of discrepancy. For intuition, whenever we can establish for a measure $\text{dist}(P, Q)$ that $\E_Q(h) \leq F(\E_P(h), \text{dist}(P, Q))$, for some non-decreasing $F$, it certainly follows that 
$\mod(\epsilon) \leq F(\epsilon, \text{dist}(P, Q))$. Details are given in Appendix \ref{app:relation2otherMeasures} along with discussions of situations where the inequalities are strict. 

\paragraph{Some Discrepancies in Classification.} We first remark that some of the notions below can be stated generally beyond classification, e.g., the \emph{$\cal Y$-discrepancy} and the \emph{transfer-exponent}, however they usually appear in works on classification. In what follows assume  
$\loss(a, b) = \ind{a\neq b}$ and suppose $\Y = \{-1,1\}$. 

$\bullet$ {\bf $\cal Y$-discrepancy} \citep*{mohri2012new}:  $\disc_{\cal Y}(P, Q) \doteq \sup_{h \in \Hyp}\abs{R_P(h) - R_Q(h)}$. For any $\epsilon\in (0, 1]$, 
\begin{align}
    \mod(\epsilon) \leq \epsilon + \disc_{\cal Y}(P, Q) +\paren{\inf_{h \in \Hyp} R_P(h) - \inf_{h \in \Hyp} R_Q(h)}. \label{eq:Y-disc-upper-bound}
\end{align}

$\bullet$ {\bf Localized $\cal Y$-discrepancy}: \citet*{hanneke2019value} instead defines $\disc_{\cal Y}'(P,Q) = \sup_{h \in \Hyp} |\E_P(h)-\E_Q(h)|, $
  which can be further \emph{localized} as 
 $\disc_{\cal Y}^*(P,Q; \epsilon_0) = \sup_{h \in \Hyp_P(\epsilon_0)} |\E_P(h)-\E_Q(h)|, $  
for any fixed $\epsilon_0> 0$ (see Example 4 of \citet*{hanneke2019value}, which was later rediscovered in \citep[][Proposition 5]{zhang2021quantifying}). Clearly,  
$$\mod(\epsilon) \leq \epsilon + \disc_{\cal Y}^*(P,Q; \epsilon_0) \text{ for any } \epsilon \leq \epsilon_0.$$

$\bullet$ {\bf $\cal A$-discrepancy} \citep{ben2010theory}: $\disc_{\cal A}(P, Q) \doteq \sup_{h,h' \in \Hyp} \abs{ P_X( h \neq h' ) - Q_X( h \neq h' ) }$.
For simplicity, as in \citep{mansour2009domain}, suppose there exist $\hstarQ, \hstarP \in \Hyp$ that minimize the respective risks, and also let $\hstar = \argmin_{h \in \Hyp} R_Q(h)+R_P(h)$.
For $\epsilon\in (0, 1]$, 
\begin{align} 
\mod(\epsilon) \leq \epsilon + \disc_{\cal A}(P,Q) + (R_P(\hstar_P)-R_Q(\hstar_Q)) + R_Q(\hstar)+R_P(\hstar).
\label{eq:A-disc-upper-bound}
\end{align} 

 $\bullet$ {\bf Transfer exponent} \citep{hanneke2019value}: a value $\rho \geq 0$ is called a transfer exponent if $\exists C_{\rho} > 0$ satisfying:
$\forall h \in \Hyp$, $\E_Q(h) \leq C_{\rho} \cdot \E_P^{1/\rho}(h) + \pivot^{\sharp}$.
For any $\epsilon \in (0,1]$, it's immediate that 
$$\mod(\epsilon) \leq C_{\rho} \cdot \epsilon^{1/\rho} + \pivot^{\sharp}.$$

We also discuss TV,  
KL divergence, and bounded density ratios \citep{sugiyama2008direct} in Appendix~\ref{app:relation2otherMeasures}.

\paragraph{Some Discrepancies in Regression.}
The discussion here may be extended to classification by viewing $\ell(a, b)$ as a \emph{surrogate loss} (and the classifier as $\text{sign}(h)$ for example). In what follows, let $\loss(a,b)=(a-b)^2$ denote the squared loss, and $\cal Y \subset \real$. The discussion inherently assumes that $\expec \ Y^2 < \infty$ under $P$ or $Q$.

$\bullet$ {\bf Wasserstein Distance} \citep{redko2017theoretical, shen2018wasserstein}: Let $\cal L$ denote the set of $1$-Lipschitz functions on $\X$ w.r.t some metric. Then consider the \emph{integral probability metric} 
$W_1(P, Q) \doteq \sup_{f \in \cal L}\abs{\EE_{P_X} (f) - \EE_{Q_X}(f)}$. 
Suppose $\Hyp$ is a set of bounded functions $\X \to [-M,M]$
with $\Hyp \subset \lambda \cdot {\cal L}, \lambda > 0$.
We consider a weak covariate shift 
scenario, where $\EE_{P}[Y|X] = \EE_{Q}[Y|X] = \hstar(X)$ for some $\hstar \in \Hyp$.
Then $\forall \epsilon > 0$,
\begin{align}
\mod(\epsilon) \leq \epsilon + 8 M \lambda \cdot W_1(P,Q). \label{eq:W1-Upperbound}
\end{align}

$\bullet$ {\bf Covariance Ratios} \citep{zhang2022class, ge2023maximum}: consider linear regression, i.e., $X\in \real^d, Y \in \real$, and where
$\Hyp\doteq \braces{h_w(x) \doteq w^\top x: w\in \real^d}$, and let $\Sigma_\mu \doteq \expec_\mu XX^\top$, for $\mu \equiv$ $P$ or $Q$. Assume $\Sigma_P$ is full rank, and let $\hstar_P = \argmin_{h} R_P(h)$. Then, letting 
$\lambda_{\text{max}}$ denote the maximum eigenvalue,
we have for any $\epsilon > 0$, 
\begin{align}
& & \text{if }  \E_Q(\hstar_P) = 0 & \text{ (relaxed covariate-shift condition) } \ \mod(\epsilon) \leq \lambda_{\text{max}}\paren{\Sigma_P^{-1} \Sigma_Q} \cdot \epsilon. 
\label{eq:cov-Ratio1}
\\ & & \text{otherwise, }  \ 
\ \mod(\epsilon) &\leq \lambda_{\text{max}}\paren{\Sigma_P^{-1} \Sigma_Q}\cdot\epsilon + \sqrt{\lambda_{\text{max}}\paren{\Sigma_P^{-1} \Sigma_Q}\cdot\E_Q(\hstar_P)}\cdot\sqrt{\epsilon} + \E_Q(\hstar_P) \label{eq:cov-Ratio2} 
\\ \ & & &\leq  2\lambda_{\text{max}}\paren{\Sigma_P^{-1} \Sigma_Q}\cdot\epsilon + 2\E_Q(\hstar_P). 
\end{align}

Similarly, certain variants of the 
MMD distance \citep{huang2007correcting}, another so-called \emph{integral probability metric}, 
may be viewed as special cases of ${\cal Y}$-discrepancy or extensions of the ${\cal A}$-discrepancy to the regression setting; 
see e.g., \citet{redko2017theoretical}. The implied bounds on $\mod(\epsilon)$ are analogous to those given for $\mathrm{disc}_{\mathcal{Y}}$ and $\mathrm{disc}_{\mathcal{A}}$ above.

Upper-bounds similar to the above have been shown in various recent results of \citep{mousavi2020minimax, zhang2022class, ge2023maximum}. 
We note that a number of recent results, e.g. \citep{zhang2022class, pathak2023noisy, chen2024high}, consider more general linear cases with different risk minimizers under $P$ and $Q$, along with refinements on $\lambda_{\text{max}}\paren{\Sigma_P^{-1} \Sigma_Q}$ in their upper-bounds.

\subsection{Rates Expressible in Terms of the Weak Modulus}
\label{sec:weak-lower-bounds}
In this section, we show via classification and regression examples that the best possible rates expressible in terms of the weak modulus take the form $\min\{\epsilon_Q, \mod(\epsilon_P)\}$ where $\epsilon_\mu$, $\mu = P$ or $Q$, denotes the best achievable rates (up to log factors) using data from $\mu$ alone. This is done via minimax lower-bounds on classes of problems parametrized by generic functions $f$ that upper-bound $\mod(\cdot)$.

\paragraph{Working Examples of Prediction Tasks.} To illustrate our results we will often consider the two classical examples of \emph{classification with bounded VC dimension}, and {linear regression} (viewed generally over any finite-dimensional linear space, i.e., including for instance polynomial regression).

The exact settings are formalized below.

\begin{itemize}[leftmargin=11mm,align=left]
\item[\namedlabel{item:vc-cls}{(VC-CLS)}]
We let $\Y = \{-1,1\}$, $\loss(y, y') = \ind{y\neq y'}$, and we assume $\Hyp$ has finite VC dimension $\V$ \citep*{VC:72}. 

\item[\namedlabel{item:lin-reg}{(LIN-REG)}]  We let $X\in \real^d, Y \in \real$, we let 
$\Hyp\doteq \braces{h_w(x) \doteq w^\top x: w\in \real^d}$, 
and we consider the squared loss $\ell(y, y') = (y -y')^2$.
We assume $\Expectation_\mu Y^2  < \infty$ and $\Sigma_\mu \doteq \expec_\mu XX^\top$ is well defined for $\mu = P$ or $Q$. 
\end{itemize}

The proofs of the lower-bound results below can be found in Appendix \ref{app:lowerBoundWeak}.

\paragraph{Classification \ref{item:vc-cls}.}
We adopt the following classical noise conditions (see e.g. \citep{mammen:99,audibert2007fast,MN:06,koltchinskii:06,bartlett:06b}). 

\begin{definition}[BCC] \label{def: Bernstein noise condition}
In what follows, we use the short-hand $\mu(h \neq h') \doteq \mu_X\paren{\braces{x\in \X: h(x) \neq h'(x)}}$, and we define the {\bf diameter} of any $\Hyp' \subset \Hyp$, as
$
    \diam_\mu(\Hyp') \doteq \sup_{h,h' \in \Hyp'} \mu( h \neq h' ).
$
We say that $\Hyp$ satisfies a {\bf Bernstein Class Condition} (BCC), as measured under $\mu$, with parameters $(C_{\mu}, \beta_\mu$), $C_\mu > 0$ and $\beta_{\mu} \in [0, 1]$, 
if $\forall 0 < \epsilon < 1/2$, 
\begin{equation}
    \diam_\mu(\H_\mu(\epsilon)) \le C_\mu\cdot \epsilon^{\beta_\mu}. \label{eq:BCC}
\end{equation}
\end{definition}
Note that the condition trivially holds for $\beta_\mu = 0$, $C_\mu = 1$.  Also, the constraint $\epsilon < 1/2$ is superfluous in the case $C_\mu \geq 2$, and merely serves to simplify our analysis in some places.
The condition captures the hardness of the learning problem: when $\beta_\mu =1$, which formalizes \emph{low noise} regimes, we expect fast rates of the form $n^{-1}$, in terms of sample size $n$, while for $\beta_\mu = 0$, rates are of the more common form $n^{-1/2}$. 
In general, tightest rates are of the form $n^{-\frac{1}{2-\beta_\mu}}$
(up to logs)
and are achievable by ERM learners \citep*{MN:06,koltchinskii:06}.

We note however that, while $\beta_\mu = 0$ imposes no restriction on the class, any $\beta_\mu > 0$ by definition can be restrictive: suppose for instance that $\inf_h R_\mu(h)$ is achieved, then by \eqref{eq:BCC}, all risk minimizers can only differ on a set of measure $0$. Nonetheless, BCC yields useful insights into the continuum of achievable rates in classification across noise regimes.

\begin{definition}[Weak Modulus Class]
\label{defn:weak-sigma-class}
Let $\Hyp$ denote a class of functions $\X \mapsto \braces{-1, 1}$.
Consider any non-decreasing function $f$ on $(0, 1] \mapsto [0, 1]$. For any such $f$, and $0\leq  \beta_P, \beta_Q \leq 1$, let 
$\Sigma_{\Hyp}(f, \beta_P, \beta_Q)$ denote all pairs of distributions $P, Q$ such that, w.r.t. $\Hyp$ we have  
(\rm{i}) $\mod_{\PQ}(\epsilon) \leq f(\epsilon), \ \forall \ 0 < \epsilon\leq 1/4$, and (\rm{ii}) $P$ and $Q$ satisfy BCC (Definition \ref{def: Bernstein noise condition}) with parameters $(1, \beta_P)$ and $(1, \beta_Q)$ respectively. 
\end{definition}

\begin{theorem}[Classification Lower-Bounds]\label{thm:mod-class-lower-bound}
In what follows, let $\hhat$ denote any learner (possibly improper) having access to independent datasets $S_P$ and $S_Q$ of size $n_P, n_Q$ respectively. The following holds for any $\Hyp$ with $|\Hyp| \geq 3$.

{\rm (i)} Suppose the VC dimension of $\Hyp$ satisfies $\V \leq \max\{ n_Q, n_P \}$, and that for some $\kappa \geq 1$, $f$ is nondecreasing and satisfies $\alpha f(\epsilon) \leq \kappa f(\alpha \epsilon)$ for all $0< \epsilon, \alpha \leq 1$. We have that, for universal constants $c,c_0$: 
\begin{align}
\inf_{\hhat} \sup_{(P, Q) \in \Sigma_{\Hyp}(f, \beta_P, \beta_Q)} \Expectation \ \E_Q(\hhat) \geq c\cdot {\frac{1}{\kappa}}\cdot 
\min\braces{ \left( \frac{\V}{n_Q} \right)^{\frac{1}{2-\beta_Q}}, f\!\left( c_0 \left( \frac{\V}{n_P} \right)^{\frac{1}{2-\beta_P}} \right) }. \label{VC-lower-bound}
\end{align}

{\rm (ii)} For general nondecreasing $f$ we have that, for universal constants $c,c_0$: 
\begin{align}
\inf_{\hhat} \sup_{(P, Q) \in \Sigma_{\Hyp}(f, \beta_P, \beta_Q)} \Expectation \ \E_Q(\hhat) \geq c\cdot 
\min\braces{ \left( \frac{1}{n_Q} \right)^{\frac{1}{2-\beta_Q}}, f\!\left( c_0 \left( \frac{1}{n_P} \right)^{\frac{1}{2-\beta_P}} \right) }.
\end{align}
{\rm (iii)} For general nondecreasing $f$, for any $2 \leq d \leq \max\{n_Q,n_P\}$ there exists $\Hyp$ with $\V = \Theta(d)$ such that \eqref{VC-lower-bound} holds. 
\end{theorem}

Notice that the condition on $f$ admits any nondecreasing concave function $f\geq 0$, with corresponding $\kappa =1$ (i.e., $f(\alpha \epsilon) \geq \alpha f(\epsilon) + (1-\alpha) f(0)$), and allows some convex $f$ for $\kappa \geq 1$. 
For instance, this admits functions of the form $C_1 \epsilon^{1/\rho} + C_2$ for $\rho \geq 1$.
As it turns out, the condition is needed in our construction in order to satisfy 
$\mod(\epsilon) \leq f(\epsilon)$ 
for all $0 < \epsilon \leq 1/4$. 
The proof of Theorem~\ref{thm:mod-class-lower-bound} is included in Appendix~\ref{app:lowerBoundWeak}.

\begin{remark}[$\pivot^\sharp$ as a Learning Limit]
    One may also wonder how $\pivot^\sharp\doteq \lim_{\epsilon\to 0}\mod(\epsilon)$ may lower-bound achievable rates in regimes with small $n_Q$. This is somewhat opaque since the distributions witnessing the above lower-bounds allow for $\pivot^\sharp = 0$. We provide additional classification lower-bounds in Appendix \ref{app:lowerBoundWeak} in terms of $\pivot^\sharp$.
\end{remark}

\paragraph{Linear Regression \ref{item:lin-reg}.}
We introduce the following class of distributions.

\begin{definition}[Regression Transfer class]
For any $\sigma^2_Y, \lambda_0, \epsilon_0>0$, we let $\Lambda\paren{\sigma^2_Y, \lambda_0, \epsilon_0}$ denote the class of all pairs of distributions $P, Q$ (with bounded second moments) such that, 
({\rm ii}) $\text{Var}[Y|X] \leq \sigma^2_Y$, ({\rm ii}) $\Sigma_Q\succeq\frac{1}{d}I_d$, ({\rm iii}) $\Sigma_P\succeq \frac{1}{d\lambda_0}I_d$ and 
$\lambda_{\max}\paren{\Sigma_P^{-1}\Sigma_Q} \leq \lambda_0$, and ({\rm iv}) $\E_Q(\hstarP) \leq \epsilon_0$, for $\hstarP \doteq \argmin_{h \in \Hyp} R_P(h)$. 
\end{definition}

To understand the relevance of the following theorem to our discussion of the weak modulus,
recall from Section \ref{sec:weakmod-upper-bounds}, Equation \ref{eq:cov-Ratio2}, that for every pair $(P, Q)$ in $\Lambda\paren{\sigma^2_Y, \lambda_0, \epsilon_0}$, we have 
$\mod(\epsilon) \leq 2\lambda_0\cdot \epsilon + 2\epsilon_0$.

\begin{theorem}[Regression Lower-Bound]\label{thm:reg-LwBnd} 
Let $n_P\lor n_Q \geq 1$. 
Let $\hhat$ denote any learner (possibly improper) having access to independent datasets $S_P$ and $S_Q$ of size $n_P$, $n_Q$ respectively.
For any $\lambda_0 \geq 1$,
$\sigma^2_Y > 0$, 
$\epsilon_0 \geq 0$, we have 
    \begin{align}
\inf_{\hhat} \sup_{(P, Q) \in \Lambda\paren{\sigma^2_Y,\lambda_0, \epsilon_0}} \EE\  \E_Q(\hhat) \geq c\cdot 
\min\braces{ \frac{d{\cdot \sigma^2_Y}}{n_Q}, \lambda_0 \frac{d{\cdot \sigma^2_Y}}{n_P} +  {\varepsilon_0}}, \text{ for a universal constant } c>0. 
\end{align}
\end{theorem}

The proof of Theorem~\ref{thm:reg-LwBnd} is included in Appendix~\ref{app:lowerBoundWeak}.

\subsection{Adaptive Rates in Terms of the Strong Modulus}
\label{sec:strongConfBounds}

We show in this section that it is possible to achieve rates \emph{adaptive} to the strong modulus (and hence to the weak modulus), i.e., with no a priori knowledge of such relation between $P$ and $Q$. In particular, adaptive rates would take the form $\mod(\epsilon_Q, \epsilon_P) \leq \min\braces{\epsilon_Q, \mod(\epsilon_P)}$ while the modulus $\mod(\cdot, \cdot)$ is unknown to the learner.

Such adaptivity can be reduced to access to \emph{strong} confidence sets $\hat \Hyp_\mu$, where $\mu$ stands for $P$ and $Q$: these are sets estimated from data which, with high probability, (a) only contain good predictors, i.e., $\hat \Hyp_\mu$ is contained in some $\Hyp_\mu(\epsilon_\mu)$, for small $\epsilon_\mu$, and (b) also contain \emph{all} predictors with very low excess risk, i.e., it contains $\Hyp_\mu(\epsilon_{\mu}/C)$ for some $C > 1$. Adaptive transfer rates will then be inherited from such confidence sets, i.e., as determined by $\epsilon_\mu$ and $C$, for $\mu\equiv P$ or $Q$. 
Examples of such sets for various problems of interest are given later in Appendix \ref{app:strongConfSet}.
In particular, for most cases of interest, 
the achievable value of $\epsilon_\mu$ is of essentially 
the same order as the minimax rate for excess risk.

\begin{wrapfigure}{r}{0.17\textwidth}
\vspace{-1.5cm}
  \begin{center}
    \includegraphics[width=0.17\textwidth]{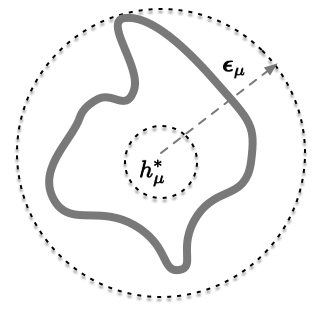}
  \end{center}
  \vspace{-3mm}\caption{\small A confidence set $\conf_\mu$ is illustrated here with solid gray boundary; the \emph{ball} of radius $\epsilon_\mu$ centered at $\hstar_\mu$ represents $\Hyp_\mu(\epsilon_\mu)$, while the inner ball represents $\Hyp_\mu(\epsilon_\mu/C)$.}
  \label{fig:ConfSet}
  \vspace{-1.2cm}
\end{wrapfigure}

\begin{definition}\label{def:strongconf} Let $\mu$ denote either $P$ or $Q$. 
Let $0< \tau \leq 1$ and $C > 1$. We call a random set $\conf_\mu \doteq \conf_\mu(S_\mu)$ an {\bf $(\epsilon, \tau, C)$-strong confidence set} (under $\mu$), for an $\epsilon>0$, 
if the following holds with probability at least $1-\tau$: 
$$ \Hyp_\mu(\epsilon/C) \ \subset \ \conf_\mu \ \subset \ \Hyp_\mu(\epsilon).$$
\end{definition}

Examples of strong confidence sets are deferred for now till Section \ref{sub:examples-StrongConf}. 
We also discuss relaxations of the notion of strong confidence set in Appendix~\ref{app:weakConf},  
but which are only able to achieve the weaker guarantee in Corollary~\ref{cor:strong-mod-rate-bounded-by-weak-mod} in terms of the weak modulus.

For the rest of this section, we require the following definition of ERM which admits situations where the infimum empirical risk might not be attained (as we admit general $\ell$ and $\Hyp$).

\begin{definition} 
Let $\mu$ denote either $P$ or $Q$, and let $\hat{R}_\mu(h) \doteq \frac{1}{n_\mu} \sum_{(x,y) \in S_\mu} \loss(h(x), y)$ denote empirical risk over $S_\mu$. 
For any $\Hyp' \subset \Hyp$, 
we let $\hhat_\mu(\Hyp')$ denote the function returned by an {\bf empirical risk minimization} function 
$\ERM_\mu(\Hyp')$: 
namely, a function mapping the data set $S_\mu$ to an $h \in \Hyp'$ 
satisfying $\hat{R}_\mu(h) < \inf_{h' \in \Hyp'} \hat{R}_\mu(h') + e^{-n_{\mu}}$.
In particular, when $\Hyp' = \Hyp$, 
we simply write $\ERM_\mu = \ERM_\mu (\Hyp)$, and $\hhat_\mu \doteq \hhat_\mu(\Hyp)$.
\end{definition}

\begin{remark} 
In common settings such as classification with $0$-$1$ loss, or linear regression with squared loss, we can let $
\ERM_\mu(\Hyp)$ denote any $h \in \Hyp$ achieving $\inf_{h' \in \Hyp} \hat{R}_\mu(h')$.  
\end{remark}

The following procedure, expressed in terms of strong confidence sets, will be shown to adaptively achieve a rate in terms of the strong modulus.

{\vskip 4mm}
\begin{alg}\label{alg:strongupper}

Input: $\conf_Q^\sharp \supset \conf_Q^\flat$, and $\conf_P$, all subsets of $\Hyp$.

\begin{quote}
If $\conf_Q^\sharp \cap \conf_P \neq \emptyset$, return any $h\in \conf_Q^\sharp \cap \conf_P$, \\
otherwise return $\hhat_{P, Q} \doteq \ERM_{P}( \conf_Q^{\sharp} )$.
\end{quote}
\end{alg}

We require the following local definition of achievable rate (under $P$) over a given subset $\Hyp'$ of $\Hyp$.  
 This is because, as we consider general scenarios, it is possible that $\hhat_P(\Hyp')$ may admit a different generalization rate than $\hhat_P(\Hyp)$: for instance in classification, the Bernstein Class Condition on $\Hyp$ (Definition \ref{def: Bernstein noise condition}) may allow for fast rates $\ll n_P^{-1/2}$ for $\hhat_P(\Hyp)$, but if the same BCC parameters do not hold for $\Hyp'$ we may have $\Prate(\Hyp')\approx n_P^{-1/2}$.

\begin{definition}\label{def:Prate}
Fix $0<\tau < 1$. 
Given a subset $\Hyp' \subset \Hyp$, 
we let  $\Prate(\Hyp')$, 
depending on $n_P$, denote the following: 
$$
\min\braces{ {\epsilon > 0} :   \prf{S_P}{R_P(\hhat_P(\Hyp')) - \inf_{h \in \Hyp'} R_P(h) \leq \epsilon} \geq 1-\tau}.$$
\end{definition}

Note that the infimum is achieved in the set by continuity of probability measures.

\begin{theorem}
\label{thm:double-delta}
Let $0 < \tau \leq 1$, $C^{\sharp},C^{\flat},C > 1$. Suppose $\conf_Q^\sharp$, $\conf_Q^\flat$ are respectively $(\epsilon_Q, \tau,C^{\sharp})$ and $(\epsilon_Q/C^{\sharp}, \tau, C^{\flat})$-strong confidence sets under $Q$, and $\conf_P$ is an $(\epsilon_P, \tau, C)$-strong confidence set under $P$, for some $\epsilon_Q,\epsilon_P >0$ depending respectively on $n_Q$ and $n_P$. Let $\hhat$ denote the hypothesis returned by Algorithm \ref{alg:strongupper}. 
Letting $\Prate = \Prate(\conf_Q^\flat)$,
we then have with probability at least $1-4\tau$:

$$ \E_Q(\hat h) \leq 
\begin{cases} 
\mod(\epsilon_Q, \epsilon_P) & \text{ if } \epsilon_Q > \pivot \text{ or } \conf_Q^\sharp \cap \conf_P\neq \emptyset, \\
C^{\flat} \cdot \mod(\epsilon_Q/C^{\sharp}, \Prate) & \text{otherwise}.
\end{cases}
$$
\end{theorem}
Note that any $C > 1$, i.e., arbitrarily large, may be chosen for $\conf_P$ above as it does not affect the rate.

\begin{proof}
Assume the events of Definition \ref{def:strongconf} (for $\conf_Q^\sharp$, $\conf_Q^\flat$, $\conf_P$)
hold (this occurs with probability at least $1-3\tau$ by the union bound). 
By Definition \ref{def:strongconf}, every $h \in \conf_Q^\sharp \cap \conf_P$ is in 
$\Hyp_Q(\epsilon_Q) \cap \Hyp_P(\epsilon_P)$. Now notice that for any such $h$ in this last intersection, we have $\E_P(h;\Hyp_Q(\epsilon_Q)) \leq  \E_P(h) \leq \epsilon_P$. 
Thus, by definition, we have that any $h \in \conf_Q^\sharp \cap \conf_P$ satisfies $\E_Q(h) \leq 
\mod(\epsilon_Q, \epsilon_P)$. 

Now suppose $\epsilon_Q > \pivot$, and $\conf_Q^\sharp \cap \conf_P = \emptyset$. Recall the equivalent definition of $\mod(\epsilon_Q, \epsilon_P)$ of Corollary~\ref{cor:pivotal-value} for the case $\epsilon_Q > \pivot$. 
By Proposition \ref{prop:pivotal-equivalence}, for any $\epsilon>0$, $\Hyp_Q(\epsilon_Q) \cap \Hyp_P(\epsilon) \neq \emptyset$; therefore, 
pick any 
$h \in \Hyp_Q(\epsilon_Q)\cap \Hyp_P(\epsilon_P/C)$. 
Since $\Hyp_P(\epsilon_P/C) \subset \conf_P\subset \Hyp_P(\epsilon_P)$,
it follows that $h\in\Hyp_Q(\epsilon_Q)\cap \Hyp_P(\epsilon_P)$, hence, $\E_Q(h) \leq \mod(\epsilon_Q, \epsilon_P)$. Furthermore, since $\Hyp_Q(\epsilon_Q/C^{\sharp}) \cap \Hyp_P(\epsilon_P/C) \subset \conf_Q^\sharp \cap \conf_P = \emptyset$, we know $h \notin \Hyp_Q(\epsilon_Q/C^{\sharp})$, i.e., $\E_Q(h)>\epsilon_Q/C^{\sharp}$.
In other words, for the choice $\hhat_{P, Q} \in \conf_Q^\flat \subset \Hyp_Q(\epsilon_Q/C^{\sharp})$, we must have that 
$$\E_Q(\hhat_{P, Q}) \leq \epsilon_Q/C^{\sharp} < \E_Q(h) \leq \mod(\epsilon_Q, \epsilon_P).$$

For the final case, suppose $\epsilon_Q \leq \pivot$, while $\conf_Q^\sharp \cap \conf_P = \emptyset$. We have the following two subcases.

-- First, suppose that $\displaystyle \inf_{h \in \Hyp_Q(\epsilon_Q/C^{\sharp})}R_P(h) = \inf_{h \in \Hyp_Q(\epsilon_Q/(C^{\flat}C^{\sharp}))}R_P(h)$. Then it follows that 
$$\E_P(\hhat_{P, Q}; \Hyp_Q(\epsilon_Q/C^{\sharp})) = \E_P(\hhat_{P, Q}; \conf_Q^\flat) \leq \Prate \text{ (with probability at least $1-\tau$) }.$$
Since we also have that $\hhat_{P, Q}\in \Hyp_Q(\epsilon_Q/C^{\sharp})$, it follows by definition that it satisfies $\E_Q(\hhat_{P, Q}) \leq \mod(\epsilon_Q/C^{\sharp}, \Prate)$. 

-- Otherwise, let $0 < \epsilon \leq \Prate$, and pick $\tilde{h}\in \Hyp_Q(\epsilon_Q/C^{\sharp})$ such that 
$$\inf_{h \in \Hyp_Q(\epsilon_Q/C^{\sharp})}R_P(h) \leq R_P(\tilde{h}) < \min\braces{\inf_{h \in \Hyp_Q(\epsilon_Q/(C^{\flat}C^{\sharp}))}R_P(h), \inf_{h \in \Hyp_Q(\epsilon_Q/C^{\sharp})}R_P(h) + \epsilon}.$$
By the second inequality, namely, the first term in the $\min$, such an $\tilde{h} \notin \Hyp_Q(\epsilon_Q/(C^{\flat}C^{\sharp}))$. Using the second term in the $\min$, we have that $\E_Q(\tilde{h}) \leq \mod(\epsilon_Q/C^{\sharp}, \epsilon)$. 
We therefore have that 
$$\E_Q(\hhat_{P, Q}) \leq \epsilon_Q/C^{\sharp} < C^{\flat}\E_Q(\tilde{h}) \leq 
C^{\flat}\mod(\epsilon_Q/C^{\sharp}, \epsilon)
\leq
C^{\flat}\mod(\epsilon_Q/C^{\sharp}, \Prate).$$

We remark that each of these final two subcases yields a bound slightly better than the one stated in the theorem, which is given as such for simplicity.  
\end{proof}

\begin{corollary}[Improving on Weak Modulus Rate]
\label{cor:strong-mod-rate-bounded-by-weak-mod}
Under the conditions of Theorem~\ref{thm:double-delta}, 
if $C^{\flat} \leq C^{\sharp}$, we have 
with probability at least $1-4\tau$, that Algorithm \ref{alg:strongupper} achieves a rate never larger than 
$\min\!\paren{ \epsilon_Q, \mod(\epsilon_P)}$. 
\end{corollary}
\begin{proof} Whenever $\E_Q(\hat h) \leq \mod(\epsilon_Q, \epsilon_P)$, the claim follows by Corollary~\ref{cor:weak-vs-strong}. Otherwise we have 
$$ \E_Q(\hat h) \leq C^\flat \cdot \mod(\epsilon_Q/C^\sharp, \Prate) \leq C^\flat \cdot \epsilon_Q/C^\sharp \leq \epsilon_Q.$$
Thus, for $\epsilon_Q \leq \pivot \leq  \lim_{\epsilon \to 0} \mod(\epsilon)$, the above implies 
that the bound on $\E_Q(\hhat)$ is at most $\min \paren{\epsilon_Q, \mod(\epsilon_P)}$. 
\end{proof}

\begin{remark}[Matching the Lower-bounds of Section \ref{sec:weak-lower-bounds}] Examples of strong confidence sets and achievable rates are discussed next in Section \ref{sub:examples-StrongConf}, 
which combined with the above Corollary \ref{cor:strong-mod-rate-bounded-by-weak-mod} imply that Algorithm \ref{alg:strongupper} can achieve rates faster than the best possible rates expressible in terms of the weak modulus $\mod(\cdot)$. 
Namely, for Classification \ref{item:vc-cls}, Algorithm \ref{alg:strongupper} achieves $\min \braces{n_Q^{-1/(2-\beta_Q)}, \mod\paren{n_P^{-1/(2-\beta_P)}}}$, valid for any of the upper-bounds on $\mod(\cdot)$ established in Section \ref{sec:weakmod-upper-bounds} in terms of existing measures of discrepancy. 
For Regression 
\ref{item:lin-reg}, Algorithm \ref{alg:strongupper} achieves the rate of $\min \braces{(d/n_Q), \lambda_{\text{max}} \paren{\Sigma_P^{-1} \Sigma_Q} \cdot (d/n_P) + \E_Q(\hstarP) }$
by \eqref{eq:cov-Ratio2},
which extends the results of \citep{mousavi2020minimax, zhang2022class, ge2023maximum} to the setting with both source and target samples. 
Further discussions including recovery of existing upper-bounds for relatedness measures are given in Remark~\ref{rem:basic-implications} of Appendix \ref{app:weakConf}.
\end{remark}

\subsection{Examples of Strong Confidence Sets and Achievable Rates}\label{sub:examples-StrongConf}

We present (informally) some examples for classification and regression, while more complex constructions--in particular, adaptive to unknown BCC condition in classification---are discussed in full detail in Appendix~\ref{app:strongConfSet}. 
Importantly, therein we show how to achieve general relations between the confidence parameters $C^\flat, C^\sharp$ of Theorem \ref{thm:double-delta}.

\begin{remark}[Contrasting with Lower-Bounds]
We note that, for any of the classification and regression situations discussed below, the corresponding confidence sets achieve the same $\epsilon_Q, \epsilon_P$ rates as those in the lower-bounds (of the form $\min\braces{\epsilon_Q, \mod(\epsilon_P)}$) derived in Section \ref{sec:weak-lower-bounds} on limitations of the weak modulus. Together with Theorem \ref{thm:double-delta}, this establishes that it is possible to adaptively achieve $\mod(\epsilon_Q, \epsilon_P) \leq \min\braces{\epsilon_Q, \mod(\epsilon_P)}$ in the same settings. Later in Section \ref{sec:strict-improvements}, we will fully characterize situations where the improvements are strict. 
\end{remark}

We start with examples of (adaptive) confidence sets in classification. 

\paragraph{Classification \ref{item:vc-cls}.}
A simple strong confidence set follows easily from the classic $n^{-1/2}$ uniform convergence bounds of \citet*{talagrand:94}.
Informally, 
for any $\epsilon_\mu \gtrsim \sqrt{(\V + \log (1/\tau))/n_\mu}$, and ERM $\hhat_\mu$, 
the set 
\begin{equation}
\conf_{\mu} \doteq \left\{ h \in \Hyp : \hat{R}_{\mu}(h)- \hat{R}_{\mu}(\hhat_\mu) \lesssim \epsilon_\mu \right\}
\end{equation}
is an $(\epsilon_{\mu},\tau,2)$-strong confidence set (under $\mu$).
For general noise conditions ($\beta_\mu > 0$), as stated above, a strong confidence set achieving $(\V/n)^{1/(2-\beta_\mu)}$ is possible, though more difficult to construct, as it involves empirical (localized) uniform Bernstein fixed points (see Proposition~\ref{prop:strong-conf-noise-adaptive} of Appendix~\ref{app:strongConfSet}). 

\paragraph{Regression \ref{item:lin-reg}.}
For $\mu$ either $P$ or $Q$, 
assume $\Sigma_\mu$ invertible for simplicity (the reader might notice that this is not needed since we might as well consider the data as given in $\text{range}(\Sigma_\mu)$). 

Let $\hat w$ minimize $\hat R_\mu (h_w)$ over $w\in \real^d$. For $\epsilon_\mu  \propto \frac{d + \log(1/\tau)}{n_\mu}$, a set of the form  
\begin{align}
\conf_\mu \doteq \braces{\hat R_\mu (h_w)- \hat R_\mu (h_{\hat w}) \lesssim  \epsilon_\mu }.
\end{align}
can be shown to be an $(\epsilon_\mu, \tau, C)$-strong confidence set (under $\mu$), assuming common regularity conditions (see Proposition \ref{prop:linear-regression-confidence-set} of Appendix~\ref{app:strongConfSet}). 

\begin{remark}[Efficient implementation]
    $\conf_\mu$ is a convex set, so the resulting instantiation of {Algorithm \ref{alg:strongupper} is tractable}. 
\end{remark}

\subsection{Sub-Optimiality of Common Algorithmic Techniques}
\label{sec:other-algorithms}

In this section, we discuss various common techniques for transfer learning; we will argue that each fails to achieve the excess risk guarantee of Theorem~\ref{thm:double-delta}, although they can be successful in some scenarios. The examples focus on classification but can be extended to regression with bounded loss. 

In all that follows, let $\hhat_P = \ERM_P(\Hyp)$, $\hhat_Q = \ERM_Q(\Hyp)$
denote empirical risk minimizers under $P$ and $Q$.

\begin{example}[\bf Validation]
\label{ex:cv}
Let $\tilde R_Q(h)$ denote empirical risk on an independent target sample of size $n_Q$. 
$$\text{Return} \quad \hhat \doteq \argmin_{h \in \braces{\hhat_P, \hhat_Q}} \tilde R_Q(h).$$ 
\end{example}
\emph{Favorable Settings:} it is hard to construct a situation where $\hat h$ is guaranteed to take advantage of the $P$ sample when useful. This requires that $\tilde R_Q(\hhat_P) < \tilde R_Q(\hhat_Q)$, which is hard to guarantee, but can happen sometimes. 

\emph{Sub-Optimality:} Validation can only achieve an excess risk depending on $n_Q$, as it is subject to the variance inherent in the validation sample. That is, it might fail to take advantage of the $P$ sample. 
Consider the following simple classification example. Let ${\cal X} = \braces{x_0}$, and $\Hyp = \braces{h_\sigma}_{\sigma\in \braces{\pm 1}}$ where $h_\sigma(x_0) = \sigma$. Now suppose $\pr{Y = 1| x_0} = 1/2 + \tau$, so that $\hstar_P = \hstar_Q = h_+$. Now let $n_P =\infty$ so that $\hhat_P = h_+$; by usual anti-concentration arguments, if $n_Q \leq c_1/\tau^2$, then with constant probability $c_2 > 0$, the majority of $Y$ labels in both the training and validation samples is $-1$, i.e., 
we'd have $\hhat_Q = h_-$, and also that $\hhat = h_-$. In other words, the best excess risk $\E_Q(\hhat)$ achievable is of order $\tau = \Omega(n_Q^{-1/2})$. 

In contrast, Theorem \ref{thm:double-delta} guarantees a rate of $0$ in the same scenario (with high probability).

\begin{example}[\bf Pooling]
\label{ex:pooling}
    $$\text{Return} \quad \hhat \doteq \argmin_{h \in \Hyp} \quad n_P\cdot \hat R_P(h) + n_Q \cdot \hat R_Q(h).$$
\end{example}

\emph{Favorable Settings:} If there is a common risk minimizer under $Q$ and $P$ (i.e., $\hstarP = \hstarQ$), 
\citet*{hanneke2022no} establish positive results for 
pooling in some regimes (e.g., $\beta_P=\beta_Q=1$ in Definition~\ref{def: Bernstein noise condition}),
though still with some sub-optimality in other regimes.

\emph{Sub-Optimality:} The approach is not robust to \emph{negative transfer}, i.e., to situations where the source data can be hurtful and should be ignored.  

Consider again a simple scenario 
${\cal X} = \braces{x_0}$, and $\Hyp = \braces{h_\sigma}_{\sigma \in \braces{\pm 1}}$ where $h_\sigma(x_0)=\sigma$.
Suppose $\P_Q(Y=1|x_0) = 1$ while $\P_P(Y=1|x_0) = 0$.
If $n_P > n_Q$, the pooling $\hhat$ equals $h_-$
hence has $\E_Q(\hhat) = 1$.
In contrast, simply training ERM on the $Q$ data 
would return $h_+$ with $\E_Q(h_+)=0$.
Thus, this is an example of \emph{negative transfer},
where including the data from $P$ is harmful to this 
method's risk.

In contrast, Theorem~\ref{thm:double-delta} would guarantee a rate of $0$ in the same scenario (with high probability).

We also note that such negative transfer for pooling can even occur with $\hstarP = \hstarQ$ when $n_P \gg n_Q$, provided $\hat{R}_P$ has higher variance than $\hat{R}_Q$ \citep*[see][]{hanneke2022no}. We emphasize that Theorem~\ref{thm:double-delta} always guarantees a rate no worse than that of using the $Q$ data alone, i.e., no negative transfer.

\begin{example}[\bf Fine-Tuning]
\label{ex:hyp-transfer}
\footnote{\emph{Fine-tuning} is a term employed with neural networks, where the source $P$ data may serve to learn the parameters of all bottom layers of a given architecture, which are then \emph{frozen}; training (i.e. ERM) on the target $Q$ data then only updates the parameters corresponding to top layers. This is effectively ERM on a subset $\conf_P$ of the original class $\Hyp$ of all possible weights across all layers.}
Let $\conf_P$ be an $(\epsilon,\tau,C)$-strong confidence set under $P$. 
$$\text{Return} \quad \hhat \doteq \argmin_{h \in \conf_P} \hat R_Q(h).$$
\end{example} 

\emph{Favorable Settings:} This procedure can be shown to adapt to the weak modulus in special cases, particularly when there is a minimizer $\hstar$ of \emph{both} $Q$-risk and $P$-risk (i.e., $\hstar_Q = \hstar_P$).
In such cases, it can be shown that 
$\E_Q(\hat h) \lesssim \min\paren{\epsilon_Q, \mod(\epsilon_P)}$, where $\epsilon_Q, \epsilon_P$ are the achievable rates described in Section 
\ref{sub:examples-StrongConf}.
The proof of this can be adapted from that of Proposition 3 of \citep*{hanneke2019value}.

\emph{Sub-Optimality:} Negative transfer can happen in this case.  If all $P$-good functions are $Q$-bad, restricting to $\conf_P$ eliminates $\hstarQ$.

Consider again the same scenario described for pooling above.
In this case, if $\epsilon < 1$, then with probability $1-\tau$, 
$\conf_P = \{h_-\}$, so that the $\hhat$ predictor from fine-tuning has $\E_Q(\hhat) = 1$.

\begin{example}[\bf Lepski/Robust Validation]
\label{ex:lepski}
Let $\tilde R_Q(h)$ be the empirical validation risk defined in Example \ref{ex:cv}.
$$\text{Return} \quad \hhat \doteq \begin{cases} \hhat_P & \text{if } \tilde R_Q(\hhat_P) \leq \tilde R_Q(\hhat_Q) + \epsilon_Q\\ \hhat_Q & \text{otherwise } \end{cases},$$
for appropriate $\epsilon_Q = \epsilon_Q(n_Q)$, potentially data-dependent (see e.g. \citep{hanneke2019value}).
\end{example}

\emph{Favorable Settings:} This procedure can be shown to adapt to the weak modulus, that is 
$\E_Q(\hat h) \lesssim \min\paren{\epsilon_Q, \mod(\epsilon_P)}$, where $\epsilon_Q, \epsilon_P$ are the achievable rates described in Section 
\ref{sub:examples-StrongConf}.
The proof of this can be adapted from that of Proposition 2 of \citet*{hanneke2019value}. A related method is explained in detail in Appendix \ref{app:weakConf}.

\emph{Sub-Optimality:} 
It cannot adapt to the strong modulus: 
consider the following classification example. 
Let $\Hyp = \braces{h_1, h_2, h_3}$, 
where $\argmin_{h \in \Hyp} R_P(h) = \braces{h_1, h_2}$, and $\argmin_{h \in \Hyp} R_Q(h) = \braces{h_1}$. Again suppose that $n_P = \infty$ and that $\hhat_P =  h_2$ (which can be constructed to happen with probability $1/2$). Now we can construct the $Q$ distribution such that if $\E_Q(h_3) = \Theta (n_Q^{-1/2})$, then we have with constant probability that $\hhat_Q = h_3$ (see, e.g. the above Example \ref{ex:cv} case). Therefore, let $\E_Q(h_2) \doteq \E_Q(\hhat_P) \gg \epsilon_Q$ (where, e.g., $\epsilon_Q = O(n_Q^{-1/2})$) then $\hhat = \hhat_Q $ with high probability, and we can therefore only guarantee $\E_Q(\hhat) = \Omega(n_Q^{-1/2})$. 

In contrast, for the same setting, Theorem \ref{thm:double-delta} guarantees a $Q$-excess risk of  $\mod(\epsilon_Q, 0) = 0$ with high probability. 

A detailed construction for Example \ref{ex:lepski} is given in Appendix \ref{sec:lepski-construction}.

\section{The Gap Between Strong and Weak Moduli}\label{sec:strict-improvements}
We now consider the question of when bounds in terms of the strong modulus are strictly tighter than those in terms of the weak modulus, i.e., whether we have $\mod(\epsilon_1, \epsilon_2) < \min \braces{\epsilon_1, \mod(\epsilon_2)}$ for at least some values of $\epsilon_1, \epsilon_2 > 0$. 
In other words, when does the strong modulus capture a strictly more optimistic relation between $P$ and $Q$? 

We will show in Theorem \ref{thm:gap-iff-non-monotonic} that this is fully characterized by a notion of \emph{monotonity} between achievable excess risks $\E_P$ and $\E_Q$: roughly, $\mod(\epsilon_1, \epsilon_2) =\min \braces{\epsilon_1, \mod(\epsilon_2)}$ for all $\epsilon_1, \epsilon_2$, iff $\E_Q$ decreases together with $\E_P$, for \emph{given values} of $\E_Q$, in a sense to soon be made formal. Interestingly, as a consequence, no gap exists in common situations when both the class $\Hyp$ and the loss $\ell$ are convex (Theorem \ref{thm:convexclassgap}), e.g., in ordinary least-squares linear regression. On the other hand, natural examples of gaps $\mod(\epsilon_1, \epsilon_2) < \min \braces{\epsilon_1, \mod(\epsilon_2)}$ arise in classification and regression, e.g., in feature selection (Example \ref{ex:gap} of Section \ref{sec:examplesStrongMod}). \emph{Of particular importance, whenever such gaps exist, the weak modulus $\mod(\cdot)$ and all the notions of discrepancies it lower-bounds (examples of Section \ref{sec:weakmod-upper-bounds}), do not fully capture the transferable information in the aggregate data from $P$ and $Q$.}  

\subsection{Preliminaries}

We start with the following definitions. The first definition below simply restricts attention to values of $\epsilon_1$ achievable under the target $Q$; the reason for this restriction will soon become clear. 

\begin{definition}
Define $\AQ(\Hyp) \doteq \mathrm{closure}\!\left( \braces{\epsilon \geq 0: \exists h \in \Hyp', \E_Q(h) = \epsilon} \right)$ as the closure of  \emph{achievable} $Q$-excess-risks. 
Furthermore, we consider the mapping 
$\epsilon \mapsto \ddot \epsilon \doteq \max\braces{\epsilon' \in \AQ(\Hyp): \epsilon' \leq \epsilon}$. 
\end{definition}

Remark that we have $\Hyp_Q(\epsilon) = \Hyp_Q(\ddot \epsilon)$. The following proposition then follows by definition and is given for emphasis: 

\begin{proposition}
We always have, $\forall \epsilon_1, \epsilon_2 > 0$, 
$\mod(\epsilon_1, \epsilon_2) \leq \min\braces{\ddot \epsilon_1, \mod(\epsilon_2)} 
\leq  \min\braces{\epsilon_1, \mod(\epsilon_2)}$. 
\end{proposition}

In fact the reader may have noticed that the result of Corollary~\ref{cor:strong-mod-rate-bounded-by-weak-mod} based on the weak-modulus could have been given in terms of achievable $Q$-excess errors. In particular, Algorithm \ref{alg:strongupper} (among others such as the \emph{robust validation} of Example \ref{ex:lepski}) in fact achieves the rate 
$\min\braces{\ddot \epsilon_Q, \mod(\epsilon_P)}$. It is thus the effective rate achievable under the weak modulus, whenever $\AQ(\Hyp)$ has discontinuities (e.g., if $\Hyp$ is finite). 
Hence, we will be interested in characterizing situations with a gap $\mod(\epsilon_1, \epsilon_2) < \min\braces{\ddot \epsilon_1, \mod(\epsilon_2)}$, as opposed to a gap to $\min\braces{\epsilon_1, \mod(\epsilon_2)}$ which is less meaningful. The existence of such gaps is shown to be characterized via notions of \emph{monotonicity} between $P$ and $Q$ excess risks. Also, as it turns out, $\pivot^\sharp \doteq \lim_{\epsilon\to 0} \delta(\epsilon)$ plays a 
{crucial} role. 

We start by considering $h$'s with $\E_Q(h)$ above $\pivot^\sharp$.

\begin{definition}
\label{defn:monotonic-transfer-problem}

We say that a \emph{transfer problem}, as defined by $(P, Q, \Hyp, \ell)$, is {\bf monotonic above $\pivot^\sharp$}, if 
the following holds: $\forall \epsilon_1 \in \AQ(\Hyp)$ with $\epsilon_1 > \pivot^\sharp$, and $\forall \epsilon_2>0$ with $\mod(\epsilon_2) > \epsilon_1$, we have $\forall \tau > 0$: 
\begin{align}
    \exists h,h' \in \Hyp \text { with } \E_Q(h) \in [\epsilon_1-\tau,\epsilon_1] \text{ and } \E_Q(h') \in [\mod(\epsilon_2) - \tau, \mod(\epsilon_2)], \text{ satisfying }\E_P(h) \leq \E_P(h')\leq \epsilon_2. \label{eq:pairwise-monotonicity}
\end{align}
\end{definition}

\begin{wrapfigure}{r}{0.22\textwidth}
  \begin{center}
    \includegraphics[width=0.23\textwidth]{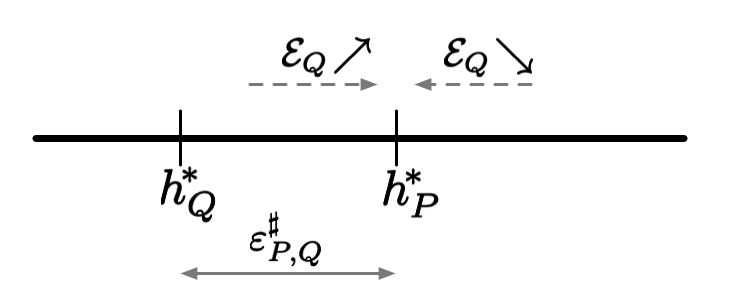}
  \end{center}
  \caption{\small Suppose $\Hyp$ consists of one-sided thresholds on the line, $Y = \hstar_\mu(x)$ for $\mu \equiv P$ or $Q$. It's easy to see that $\E_Q(h)$ (mass of disagreement regions) goes up as $h$ approaches $\hstar_P$ from the left, and goes down otherwise. The problem therefore has no gap between moduli.} 
    \label{fig:Monotonicity}
  \vspace{-0.5cm}
\end{wrapfigure}

The above definition while somewhat opaque, turns out to be equivalent to \emph{pseudo-denseness} (appropriately defined) of \emph{monotone} subsets $\Hyp'$ of $\Hyp$, i.e., subsets where $\E_Q$ goes down with $\E_P$ (see Appendix \ref{app:gapWeakStrong}).

Next we consider the structure of $h$'s with $\E_Q(h)$ at most $\pivot^\sharp$. 

\begin{definition}
We say that the transfer problem $(P,Q,\Hyp,\loss)$ is {\bf anti-monotonic below $\pivot^\sharp$} if either {$\pivot^\sharp =0$ or} the following holds true for all $0< \epsilon_1 { \leq }\ \pivot^\sharp$: 
$$\lim_{\epsilon_2 \to 0} \sup\braces{\E_Q(h): h \in \Hyp_Q(\epsilon_1), \E_P(h; \Hyp_Q(\epsilon_1)) \leq \epsilon_2} = \ddot{\epsilon}_1.$$
\end{definition}
In other words, the best $h$'s in terms of $P$-risk on the set $\Hyp_Q(\epsilon_1)$ are the worst in terms of $Q$ risk. 

\begin{remark} 
 In some sense, anti-monotonicity means that the information $P$ has on $Q$ saturates for large enough $Q$ data: if the $Q$ data alone can decrease risk below some $\epsilon_Q \leq \pivot^\sharp$, then even an infinite amount of $P$ data no longer helps; in fact, the $P$ data may even be harmful as it directs us toward worse predictors at the boundary of $\Hyp_Q(\epsilon_Q)$, while we would rather pick a predictor on its interior.
\end{remark}

\subsection{Characterization of Gaps}

\begin{theorem}[Gap Characterization]
\label{thm:gap-iff-non-monotonic}
For any transfer problem $(P,Q,\Hyp,\loss)$,
we have that $\mod(\epsilon_1,\epsilon_2) = \min\{ \ddot{\epsilon}_1, \mod(\epsilon_2)\}$
(for all $\epsilon_1,\epsilon_2 > 0$)
if and only if $(P,Q,\Hyp,\loss)$ is monotonic above $\pivot^\sharp$ and anti-monotonic below $\pivot^\sharp$.
\end{theorem}

The analysis is given in Appendix \ref{app:gapWeakStrong}. Anti-monotonicity above $\pivot^\sharp$ follows most easily by definition (Proposition \ref{prop:anti-monotonicity}). The proof of monotonicity is most involved: we show in Proposition~\ref{prop:monotone-implies-no-gap} that if Definition~\ref{defn:monotonic-transfer-problem} holds, then there is no gap above $\pivot^\sharp$; on the other hand, 
we will show in Proposition~\ref{prop:no-gap-implies-monotone} that, if 
a transfer problem has no gap above $\pivot^\sharp$ it must hold that monotone sets are pseudo-dense in $\Hyp$. This last statement is constructive, and relies on a mapping from $\epsilon_1 \in \AQ(\Hyp)$ to $\epsilon_2 = \imod(\epsilon_1)$---the pseudo-inverse of the weak modulus, which is then shown to maintain monotonicity.

\subsection{Examples}\label{sec:examplesStrongMod}
This section presents examples witnessing 
both non-trivial gaps between the strong and weak modulus, and also a general case in which there can be no such gap.

\begin{wrapfigure}{r}{0.22\textwidth}
  \begin{center}
    \includegraphics[width=0.22\textwidth]{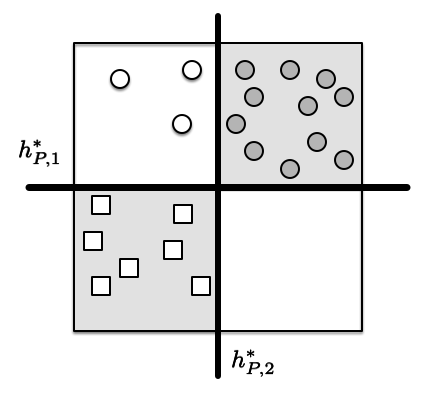}
  \end{center}
  \caption{\small Suppose $P$ is supported in the gray region so that $\hstar_{P, 1}, \hstar_{P, 2}$ are optimal under $P$ for classifying \emph{square vs circle}. Notice that $\hstar_P$ however relies on spurious features; a small amount of data from $Q$, some falling in the upper-left quadrant, allows rejecting $\hstar_{P, 2}$.} 
    \label{fig:exampleStrongMod}
  \vspace{-2cm}
\end{wrapfigure}

We start
with a simple example of gaps, which the reader may find easy to generalize. 
\begin{example}[Gap In Feature Selection]\label{ex:gap}
First we consider a basic regression problem for $X\in \real^2$ where $Q, P$ are given as follows. Let $e_1, e_2$ denote coordinate vectors. 

$\bullet$ $Q_X\sim {\cal N}(0, I_2)$ while $Y \doteq e_1^\top X + \xi$, for  $\xi \sim {\cal N}(0, 1)$. 

$\bullet$ $P_X$ is supported on the subspace
$\braces{x: x^\top e_1 = x^\top e_2}$, i.e., equal coordinates, with marginals ${\cal N}(0, 1)$, while $Y \doteq e_1^\top X + \xi = e_2^\top X + \xi$, for $\xi \sim {\cal N}(0, 1)$.

We let $\Hyp_i \doteq \braces{h_c(x) = c\cdot e_i^\top x: c \in \real}$, and $\Hyp = \Hyp_1 \cup \Hyp_2$, i.e., \emph{sparse vectors}. 

Notice now that, for any $h_c \in \Hyp_1$
, $\E_Q(h_c) = \E_P(h_c) = (1-c)^2$.

On the other hand, for any $h_c \in \Hyp_2$
, $\E_Q(h_c) = 1+ c^2$ while $\E_P(h_c) = (1-c)^2$. 

In other words, we always that (1) $\forall \epsilon_2>0, \ \mod(\epsilon_2) \geq 1 = \pivot^\sharp$, and (2) for $\epsilon_1 < 1$, $\Hyp_Q(\epsilon_1)$ only retains hypotheses in $\Hyp_1$. It follows that, for any $\epsilon_2 < \epsilon_2 < 1$, we have 
$\mod(\epsilon_1, \epsilon_2) = \epsilon_2$ as witnessed by some $h_c \in \Hyp_Q(\epsilon_1)\cap \Hyp_Q(\epsilon_2)$ with $(1-c)^2 = \epsilon_2$. In particular, we have a gap, i.e., $\mod(\epsilon_1, \epsilon_2) < \epsilon_1 < \mod(\epsilon_2)$. 
\end{example}

Notice that the above example is easily generalized to higher dimensions as the main ingredient was to uncover multiple good predictors under $P$ ---by having $\text{supp}(P_X)$ close to an appropriate subspace or manifold relating features of $X$--- some of which are bad under $Q$. Similar examples can be constructed in classification as illustrated in Figure \ref{fig:exampleStrongMod}. 

Next, we consider a surprisingly general set of situations where no gap is possible. 

\begin{theorem}[Convex classes and convex losses]\label{thm:convexclassgap}
Suppose the loss function $\loss(\cdot,\cdot)$ satisfies that, for any $y$, $\loss(\cdot,y)$ is a continuous convex function of its first argument (for instance, the squared loss).
Suppose $\Hyp$ is a convex set of functions $\X \to \reals$, that is, $\forall h,h' \in \Hyp$, $\forall \alpha \in [0,1]$, the function $\alpha h + (1-\alpha) h' \in \Hyp$ (for instance, linear functions).
Then $\forall \epsilon_1,\epsilon_2 > 0$, 
$\mod(\epsilon_1,\epsilon_2) = \min\{ \epsilon_1, \mod(\epsilon_2) \}$.
\end{theorem}
\begin{proof}
If $\epsilon_1 \geq \mod(\epsilon_2)$,
then we have 
$\Hyp_Q(\epsilon_1) \supset \Hyp_P(\epsilon_2)$, 
and it follows that
$\mod(\epsilon_1,\epsilon_2) = \mod(\epsilon_2) = \min\{ \epsilon_1, \mod(\epsilon_2) \}$.

Otherwise, suppose $\epsilon_1 < \mod(\epsilon_2)$.
Then there exists $h \in \Hyp_P(\epsilon_2) \setminus \Hyp_Q(\epsilon_1)$.
Note that we also have 
$\E_P(h; \Hyp_Q(\epsilon_1)) \leq \epsilon_2$.
Let $h'$ be any function in $\Hyp_Q(\epsilon_1)$ with $\E_P(h';\Hyp_Q(\epsilon_1)) \leq \epsilon_2$.

Since both 
$\E_P(h; \Hyp_Q(\epsilon_1)) \leq \epsilon_2$
and 
$\E_P(h';\Hyp_Q(\epsilon_1)) \leq \epsilon_2$,
and $\loss$ is convex in its first argument, 
Jensen's inequality implies any $\alpha \in [0,1]$ has 
$\E_P(\alpha h + (1-\alpha) h'; \Hyp_Q(\epsilon_1)) \leq \alpha \E_P(h; \Hyp_Q(\epsilon_1)) + (1-\alpha) \E_P(h'; \Hyp_Q(\epsilon_1)) \leq 
\epsilon_2$.

Moreover, since $\E_Q(h) > \epsilon_1 \geq \E_Q(h')$, 
by continuity of $\loss$ in its first argument (and the dominated convergence theorem to extend continuity of $\E_Q(\alpha h + (1-\alpha) h')$ in $\alpha$, using $\loss(h(x),y)+\loss(h'(x),y)$ as the dominating function, recalling that all risks are finite by Assumption~\ref{asn:finiteness}) 
we have that there exists $\alpha \in [0,1]$ such that 
$\E_Q(\alpha h + (1-\alpha) h') = \epsilon_1$.
By convexity of $\Hyp$, we know that $\alpha h + (1-\alpha) h' \in \Hyp$.
Therefore, the supremum 
$\sup\{ \E_Q(h'') : \E_Q(h'') \leq \epsilon_1, \E_P(h'';\Hyp_Q(\epsilon_1)) \leq \epsilon_2, h'' \in \Hyp \} = \mod(\epsilon_1,\epsilon_2)$ 
is \emph{achieved} by the hypothesis $\alpha h + (1-\alpha) h$, 
with a value $\epsilon_1$.
Hence, $\mod(\epsilon_1,\epsilon_2) = \min\{ \epsilon_1, \mod(\epsilon_2) \}$
in this case as well.
\end{proof}

\section{Lower Bounds for the Strong Modulus}
\label{sec:strong-modulus-learning-lower-bounds}

In this section we ask whether the upper-bounds of Theorem \ref{thm:double-delta} are tight \emph{even in the case where there might be gaps between strong and weak moduli}, i.e., for classes of distributions $(P, Q)$ admitting $\mod(\epsilon_1, \epsilon_2) < \min \braces{\ddot \epsilon_1, \mod(\epsilon_2)}$ for some values of $\epsilon_1, \epsilon_2$. Certainly, since for some transfer problems we always have $\mod(\epsilon_1, \epsilon_2) = \min \braces{\ddot \epsilon_1, \mod(\epsilon_2)}$, our lower-bounds of Section \ref{sec:weak-lower-bounds} would apply when there is no gap. 

As it turns out, the lower-bounds constructions of this section have to adhere to further structure inherent to \emph{admissible} strong moduli $\mod(\cdot, \cdot)$, i.e., not all functions can be the strong moduli of a pair $(P, Q)$. 
As in Theorem~\ref{thm:mod-class-lower-bound}, 
to approach a minimax lower bound analysis, we must first 
define a family of $(P,Q)$ pairs in terms of their strong moduli: 
namely, $\Sigma_{g} = \{ (P,Q) : \mod_{\PQ}(\cdot,\cdot) \leq g(\cdot,\cdot) \}$ for some function $g$.
In light of the fact that $\mod_{\PQ}(\cdot,\cdot)$ has an 
inherent structure, it only makes sense to restrict to 
functions $g(\cdot,\cdot)$ which also share these same 
structural properties, 
which we describe in detail in the following results.

\begin{lemma}[Gaps are Intervals]
\label{lem:gaps-are-open-intervals}
Fix any $\epsilon_2 > 0$. If for some $\epsilon_1 > 0$, we have $\mod(\epsilon_1, \epsilon_2) < \epsilon_1$, then it must hold that, 
$\forall \epsilon \text{ satisfying } \mod(\epsilon_1, \epsilon_2) \leq \epsilon \leq \epsilon_1$ we also have 
$\mod(\epsilon, \epsilon_2) = \mod(\epsilon_1, \epsilon_2)$. 
\end{lemma}
\begin{proof} 
Define 
$$\Hyp(\epsilon_1, \epsilon_2) \doteq \braces{h \in \Hyp_Q(\epsilon_1): \E_P(h; \Hyp_Q(\epsilon_1)) \leq \epsilon_2}, $$
and recall that $\mod(\epsilon_1, \epsilon_2) = \sup\braces{\E_Q(h): h\in \Hyp(\epsilon_1,\epsilon_2) }$. In other words, if $\epsilon \geq \mod(\epsilon_1, \epsilon_2)$ we have that $\Hyp(\epsilon_1,\epsilon_2) \subset \Hyp_Q(\epsilon)$, and since $\epsilon \leq \epsilon_1$, we also have that $\inf\braces{R_P(h): h \in \Hyp_Q(\epsilon)} = \inf\braces{R_P(h): h \in \Hyp_Q(\epsilon_1)}$ so we conclude that 
$\Hyp(\epsilon_1, \epsilon_2) = \Hyp(\epsilon, \epsilon_2)$. 
\end{proof} 

\begin{remark}[Implications for $\AQ(\Hyp)$]
In particular, note that Lemma~\ref{lem:gaps-are-open-intervals} further implies that if for some 
$\epsilon_1 > 0$ we have 
$\mod(\epsilon_1,\epsilon_2) < \ddot{\epsilon}_1$, then $\forall \epsilon$ satisfying $\mod(\epsilon_1,\epsilon_2) \leq \epsilon \leq \ddot{\epsilon}_1$,
we also have $\mod(\epsilon,\epsilon_2) = \mod(\epsilon_1,\epsilon_2)$,
since $\mod(\epsilon_1,\epsilon_2) = \mod(\ddot{\epsilon}_1,\epsilon_2)$.
\end{remark}

Lemma \ref{lem:gaps-are-open-intervals}
implies 
that the \emph{strong modulus is in fact very structured} as summarized in the corollary below and illustrated in Figure \ref{fig:strongModulusGraph}. 

\begin{figure}
    \centering
    \includegraphics[width=.7\linewidth]{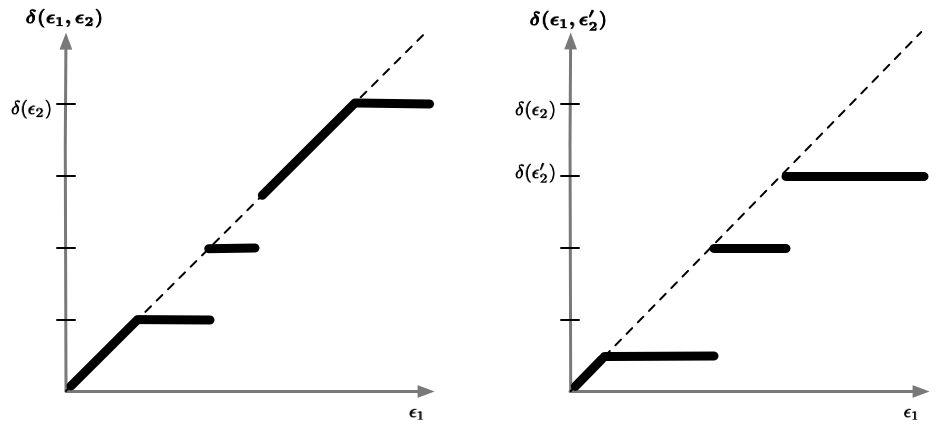}
    \caption{\small Possible graphs of modulus $\mod(\cdot, \cdot)$. Shown are possible cross-sections at some values $\epsilon_2> \epsilon_2'$.}
    \label{fig:strongModulusGraph}
\end{figure}

\begin{corollary}[Graph of the Strong Modulus] 
\label{cor:admissible-moduli}
For any transfer problem, the strong modulus $\mod(\cdot, \cdot)$ satifies the following conditions. Fix any $\epsilon_2 >\epsilon_2' > 0$: 
\begin{itemize}
    \item $\forall \epsilon_1 > 0$, $\mod(\epsilon_1, \epsilon_2)$ matches the function $\min\braces{\epsilon_1, \mod(\epsilon_2)}$ except on a collection ${\cal I}(\epsilon_2)$ of 
    intervals of $(0, \mod(\epsilon_2)]$.
    Furthermore, on any such $I\in {\cal I}(\epsilon_2)$,  $\mod(\cdot, \epsilon_2)$ is constant, equal to $\epsilon_I = \inf I$, 
    where also $\mod(\epsilon_I,\epsilon_2)=\epsilon_I$.
    \item Each 
    interval in ${\cal I}(\epsilon_2)$ is contained in an interval from ${\cal I}(\epsilon'_2)$. 
\end{itemize}
\end{corollary}
\begin{proof}
For $\epsilon_1 \leq \mod(\epsilon_2)$, 
if $\mod(\epsilon_1,\epsilon_2) < \epsilon_1$, 
Lemma~\ref{lem:gaps-are-open-intervals} implies 
$\mod(\cdot,\epsilon_2)$ is constant
between $\epsilon'_1 = \mod(\epsilon_1,\epsilon_2)$
and $\epsilon_1$, 
and moreover satisfies $\mod(\epsilon'_1,\epsilon_2) = \epsilon'_1$ as well.
We can therefore define an interval $I$
of all $\epsilon''_1$ with $\mod(\epsilon''_1,\epsilon_2) = \epsilon'_1$.
Since the above holds for every 
$\epsilon_1 \leq \mod(\epsilon_2)$ with $\mod(\epsilon_1,\epsilon_2) < \epsilon_1$,
every such $\epsilon_1$ is in some such interval $I$.
The second claim follows immediately from this interval structure, together with the fact that $\mod(\cdot,\cdot)$ is non-decreasing in both arguments (Proposition~\ref{prop:double-mod-monotone}).
\end{proof}

\begin{definition}[Strong Modulus Class]
\label{defn:strong-sigma-class}
Let $\H$ denote a hypothesis class.  
Consider any \emph{admissible} function $g : (0,1]^2 \to [0,1]$: namely, $g$ satisfying the properties of $\mod$ stated in Lemma~\ref{lem:gaps-are-open-intervals}. 
Let $0 \leq \beta_P, \beta_Q \leq 1$.
We then let $\Sigma_{\Hyp}(g,\beta_P,\beta_Q)$ 
denote all pairs of distributions $(P,Q)$ 
such that 
(i) $\mod_{\PQ}(\epsilon_1,\epsilon_2) \leq g(\epsilon_1,\epsilon_2)$
$\forall 0 < \epsilon_1,\epsilon_2 \leq 1/4$,
and 
(ii) $P$ and $Q$ satisfy BCC (Definition~\ref{def: Bernstein noise condition})
with parameters $(1,\beta_P)$ and $(1,\beta_Q)$ respectively.
\end{definition}

We have the following theorem.

\begin{theorem}[Strong Modulus Classification Lower-Bound]
\label{thm:strong-modulus-learning-lower-bound}
In all that follows, $C, c_0$ stand for a universal constant independent of problem parameters. 

({\rm i}) {\bf General admissible $\boldsymbol g$.}
For any class $\H$ with $\dim_{\Hyp} \geq 3$, 
for any strong modulus class $\Sigma_{\Hyp}(g,\beta_P,\beta_Q)$ as in Definition~\ref{defn:strong-sigma-class}, 
for any $n_Q,n_P \in \nats$, 
letting $\epsilon_Q = c_0 n_Q^{-\frac{1}{2-\beta_Q}}$ and $\epsilon_P = c_0 n_P^{-\frac{1}{2-\beta_P}}$,
the following holds.
For any learning algorithm $\hhat$, 
there exists $(P,Q) \in \Sigma_{\Hyp}(g,\beta_Q,\beta_P)$ for which the following hold: 
\begin{itemize}
\item $\mod_{\PQ}(\epsilon_Q,\epsilon_P) < \min\!\left\{ \ddot{\epsilon}_Q, \mod_{\PQ}(\epsilon_P) \right\}$,
\item $\EE \E_Q(\hhat)  \geq C g(\epsilon_Q,\epsilon_P)$.
\end{itemize}

({\rm ii}) {\bf Restricted admissible $\boldsymbol g$.} 
Consider the special case of $g$ of the form 
$g(\epsilon_1,\epsilon_2) = \min\{ \epsilon_1, f(\epsilon_2) \}, \forall \epsilon_1, \epsilon_2 >0$,  
for a nondecreasing function $f$ satisfying the condition of Theorem \ref{thm:mod-class-lower-bound}, i.e., $\alpha f(\epsilon) \leq \kappa f(\alpha\epsilon)$ for some $\kappa \geq 1$, and every $0 < \epsilon,\alpha \leq 1$. 
Then, suppose $\max\{n_Q,n_P\} \geq \dim_{\Hyp}$, 
and let $\epsilon_Q = c_0 (\dim_{\Hyp}/n_Q)^{\frac{1}{2-\beta_Q}}$ and $\epsilon_P = c_0 (\dim_{\Hyp}/n_P)^{\frac{1}{2-\beta_P}}$. 
There exists $(P,Q) \in \Sigma_{\Hyp}(g,\beta_Q,\beta_P)$ for which the following hold:
\begin{itemize}
\item $\mod_{\PQ}(\epsilon_Q,\epsilon_P) < \min\!\left\{ \ddot{\epsilon}_Q, \mod_{\PQ}(\epsilon_P) \right\}$,
\item $\EE \E_Q(\hhat) \geq \frac{C}{\kappa} g(\epsilon_Q,\epsilon_P)$.
\end{itemize}
\end{theorem}

The proof is given in Appendix~\ref{app:lowerBoundStrong},
and follows similar arguments as that of Theorem~\ref{thm:mod-class-lower-bound}.
The main distinction is that of 
ensuring the lower bound is witnessed by a pair $(P,Q)$ exhibiting a gap $\mod_{\PQ}(\epsilon_Q,\epsilon_P) < \min\{\ddot{\epsilon}_Q,\mod_{\PQ}(\epsilon_P)\}$, 
which is the reason for requiring $g$ to be admissible, as defined in Definition~\ref{defn:strong-sigma-class}.

\section{Final Remarks}

We conclude below with some open questions which may be of general interest. 

\emph{Mixed Losses.} While we have focused attention in this work to a general but \emph{single} loss $\ell$ for both source and target tasks, there are many situations in transfer learning with mixed losses: for instance in classification, we are often ultimately interested in 0-1 loss under the target, however, for computational reasons, we instead optimize a surrogate objective for the source and or the target tasks. A more refined theory will be needed to properly capture such settings. 

\emph{Benign Moduli.} Our various lower-bounds are established for general classes encoding general moduli. However, in order to match all dependencies on $\Hyp$'s complexity present in upper-bounds, our analysis required we impose some restriction on such moduli. This leaves the possibility that transfer-learning is in fact easier outside such restrictions, i.e., faster rates might be attained.  

\emph{Nonparametric Settings.} While our upper-bounds yield fairly general insights, provided access to particular confidence sets, we only provided examples of such confidence sets for parametric classification and regression. It is quite possible that different approaches might be needed in nonparametric settings, i.e. if rate optimal confidence sets are hard to construct. Furthermore, since our lower-bounds are all dimension dependent, it is unclear whether the proposed moduli in fact yield a tight characterization of transfer rates in nonparametric settings. 

\emph{MultiTask and MultiSource.} Settings with $N\gg 1$ sources $P_i, i\in [N]$ remain largely open. There are certainly many works addressing these settings by extending divergence measures designed for the $N=1$ settings (e.g., any of the divergences discussed above), but it remains unclear which measure of the aggregate information $P_i$'s have on a target $Q$ can be adapted to (e.g., each $P_i$ may be far from $Q$, while an \emph{unknown} mixture of them might be close to $Q$). 

A different assumption, considered by many recent analyses (as discussed in the introduction) is that all $P_i$'s and $Q$ share some common low-dimensional structure that can lower the complexity of the target problem (even if no good source predictor transfers well to the target problem). This is not captured by our notions of moduli.

\bibliography{refs}

\appendix

\section{Weak Modulus}\label{app:weakMod}

\subsection{Relation to Existing Discrepancy Measures}
\label{app:relation2otherMeasures}

\paragraph{Classification.} In what follows assume $\loss(a, b) = \ind{a\neq b}$ and for simplicity suppose $\Y = \{-1,1\}$ 

\begin{example}[$\cal Y$-discrepancy; \citealp*{mohri2012new}]
\label{ex:Y-disc}
Let $\disc_{\cal Y}(P, Q) \doteq \sup_{h \in \Hyp}\abs{R_P(h) - R_Q(h)}$. Similar to \citep*{mohri2012new}: 
\begin{align}
\E_Q(h) &= R_Q(h) - \inf_{h'} R_P(h')+\inf_{h'} R_P(h') - \inf_{h'} R_Q(h') \\
&\leq \E_P(h) + \disc_{\cal Y}(P, Q) +\inf_{h' \in \Hyp} R_P(h') - \inf_{h' \in \Hyp} R_Q(h'). 
\end{align}
In other words, for any $\epsilon\in (0, 1]$, 
\begin{align}
    \mod(\epsilon) \leq \epsilon + \disc_{\cal Y}(P, Q) +\paren{\inf_{h \in \Hyp} R_P(h) - \inf_{h \in \Hyp} R_Q(h)}.
\end{align}
In particular, note that the difference in risks above is at most $\disc_{\cal Y}(P, Q)$, i.e. we also have
$\mod(\epsilon) \leq \epsilon + 2\disc_{\cal Y}(P, Q)$. 
\end{example}

As noted by \citep*{mohri2012new,cortesadaptation}, 
the $\cal Y$-discrepancy in fact works for essentially any 
$\cal Y$ and loss function $\loss$
(in particular, they focused on the cases of bounded losses and convex losses).
We have the following remarks about implications and also refinements on the above.

\begin{remark}[Total Variation and KL-divergence]
Remark that the above immediately implies bounds on $\mod(\epsilon)$ in terms of TV distance $\text{TV}(P, Q) \doteq \sup_{A} \abs{P(A) - Q(A)}$, over all measurable subsets $A$, which clearly upper-bounds $\disc_{\cal Y}(P, Q)$. Furthermore, whenever $P$ dominates $Q$, Pinsker's inequality relates TV to KL-divergence as  
$\text{TV}(P, Q) \leq \sqrt{\frac{1}{2}\KLDiv{Q}{P}}$, implying immediate bounds on $\mod(\cdot)$ in terms of KL.  
Several works have considered the total variation distance as a measure of distribution shift \citep*{bartlett1992learning,barve1997complexity,long1999complexity}.

\end{remark}

Next we consider one of the most notions in the theory of transfer learning. 

\begin{example}[$\cal A$-discrepancy; \citealp{ben2010theory,mansour2009domain}]
\label{ex:A-disc}
Consider $\disc_{\cal A}(P, Q) \doteq \sup_{h,h' \in \Hyp} \abs{ P_X( h \neq h' ) - Q_X( h \neq h' ) }$.
For simplicity, as in \citep{mansour2009domain}, we suppose there exist $\hstarQ, \hstarP \in \Hyp$ that minimize the respective risks, and also let $\hstar = \argmin_{h \in \Hyp} R_Q(h)+R_P(h)$.
\begin{align}
R_Q(h) &\leq Q_X( h \neq \hstar ) + R_Q(\hstar)
\leq P_X( h \neq \hstar ) + \disc_{\cal A}(P,Q) + R_Q(\hstar) \\
&\leq R_P(h) + \disc_{\cal A}(P,Q) + R_Q(\hstar) + R_P(\hstar)
\end{align}
so that 
\begin{equation} 
\E_Q(h) \leq \E_P(h) + \disc_{\cal A}(P,Q) + (R_P(\hstar_P)-R_Q(\hstar_Q)) + R_Q(\hstar)+R_P(\hstar).
\end{equation}
where the second inequality uses $P_X(h\neq h^*) \leq \expec \ \paren{\ind{h(X)\neq Y} + \ind{h^*(X) \neq Y}}$. 
In particular, we have for any $\epsilon\in (0, 1]$, 
\begin{align} 
\mod(\epsilon) \leq \epsilon + \disc_{\cal A}(P,Q) + (R_P(\hstar_P)-R_Q(\hstar_Q)) + R_Q(\hstar)+R_P(\hstar).
\end{align} 

We remark that, similarly to the ${\cal Y}$-discrepancy, the ${\cal A}$-discrepancy can also be localized to $h,h' \in \Hyp_P(\epsilon_0)$ for $\epsilon_0 > 0$, 
as explored by \citet*{zhang2020localized}.
\end{example} 

Next, we revisit a recent proposal by the same authors which may yield tighter bounds on $\mod(\cdot)$ in many situations (see subsequent Remark \ref{rem:tightness}).

\begin{example}[Transfer Exponent; \citealp*{hanneke2019value,hanneke:23}] 
\label{ex:transfer-exponent}
Supposing there exist $\hstarQ, \hstarP \in \Hyp$ with $\E_Q(\hstarQ)=0$ and $\E_P(\hstarP)=0$, 
a value $\rho \geq 0$ is called a \emph{transfer exponent} if there exists $\hstarP \in \argmin_h R_P(h)$ such that $ R_Q(h) - R_Q(\hstarP) \leq C_\rho \cdot \E_P^{1/\rho}(h)$, for some $C_\rho$, and for all $h \in \Hyp$ (equivalently $\forall h\in \Hyp_P(\epsilon_0)$, for some $\epsilon_0>0$). It's then immediate that 
\begin{align}
    \E_Q(h) = R_Q(h) - R_Q(\hstarP) + \E_Q(\hstarP) \leq 
    C_{\rho} \cdot\E_P^{1/\rho}(h) + \E_Q(\hstarP), 
\end{align}
implying that, for all $\epsilon \in (0, 1]$, 
$\mod(\epsilon)\leq C_\rho \cdot \epsilon^{1/\rho} + \E_Q(\hstarP).$

\emph{A more general definition:} we can extend the above definition to the general case with or without $P$-risk minimizers. 
Define $\pivot^{\sharp} \doteq \lim_{\epsilon \to 0} \mod(\epsilon)$.
A transfer exponent $(C_\rho,\rho)$ is defined as satisfying the condition:
$\forall h \in \Hyp$, $\E_Q(h) \leq C_{\rho} \cdot \E_P^{1/\rho}(h) + \pivot^{\sharp}$.
This immediately implies the following bound:
$$\mod(\epsilon) \leq C_{\rho} \cdot \epsilon^{1/\rho} + \pivot^{\sharp}.$$
\end{example}

\begin{remark}[Density Ratios] 
Assuming \emph{covariate shift}, i.e., $P_{Y|X} = Q_{Y|X}$, but $P_X \neq Q_X$, and assuming further that $P_X$ dominates $Q_X$, i.e. $dQ_X/dP_X$ is well defined, many works on domain adaptation, e.g., \citep{sugiyama2008direct, sugiyama2012density}, have characterized the shift from $P$ to $Q$ via $dQ_X/dP_X$, which is typically used there for reweighting the source sample: let $\ell(h)$ denote the r.v. 
$\expec_{Y|X} \ell(h(X), Y)$, reweighting relies on the fact that 
$R_Q(h) = \int \ell(h) \frac{dQ_X}{dP_X} dP_X$. Similarly, notice for instance that if $dQ_X/dP_X$ is bounded by some constant $C$, and assuming the problem admits a Bayes predictor $\hstar\in \Hyp$ (i.e., a pointwise minimizer of $\ell(h)$ a.s.-$P_X$), then we would have 
$$\E_Q(h) = \int \paren{\ell(h) - \ell(\hstar)} \frac{dQ_X}{dP_X} dP_X
\leq C \cdot \E_P(h).$$ In other words, the problem admits a transfer exponent $\rho=1$ with $C_\rho = \sup dQ_X/dP_X$. 
\end{remark}

\begin{remark}[Tightness]\label{rem:tightness}
 While the weak modulus was introduced mainly as a way to capture existing notions of discrepancy and understand their limitations w.r.t. the strong modulus, it is nonetheless interesting on its own. Importantly, the fact that it can be adapted to (in the supervised transfer setting) despite depending on unknown risk minimums, brings up the question of the tightness of known rates expressed in terms of existing discrepancies in comparison to the rates of Corollary~\ref{cor:strong-mod-rate-bounded-by-weak-mod} expressed in terms of the weak modulus.

Moreover, investigating such tightness w.r.t. to the weak modulus, yields some qualitative point of comparison between existing notions of discrepancy, which are otherwise not directly comparable to each other.

Starting with the transfer exponent, 
we first note that this notion attempts to capture some desired properties of $\mod(\epsilon)$. For one, it is \emph{asymmetric} just as $\mod$ is, and thus does not deteriorate in situations where transfer is asymmetric---for instance $P_X$ covers the decision boundary under $Q$ but not vice versa. Second, the resulting upper-bound yields the same limit $\pivot^\sharp$ as $\mod(\epsilon)$, which captures large source sample regimes. 
However, given the arbitrariness of the polynomial form of the transfer-exponent, $\mod(\epsilon)$ can naturally be smaller. Similarly, the $\cal Y$-discrepancy and $\cal A$-discrepancy can be loose: we may have, e.g., that $\pivot^\sharp = 0$ while these discrepancies are not $0$, implying that the corresponding upper-bounds
are loose for small $\epsilon$. 

The localized ${\cal Y}$-discrepancy $\disc_{\cal Y}^*(P,Q;\epsilon_0)$ avoids this problem (taking $\epsilon_0=\epsilon$), asymptoting to $\pivot^\sharp$ similarly to the transfer exponent. 
In fact, notice that 
$\disc_{\cal Y}^*(P,Q;\epsilon_0) \leq \sup_{h \in \Hyp_P(\epsilon_0)} \max\{\E_Q(h),\E_P(h)\} \leq \max\{\mod(\epsilon_0), \epsilon_0\}$.
Thus, dropping $\epsilon_0$ for $\epsilon$, $\disc_{\cal Y}^*(P,Q;\epsilon)$ essentially captures 
the same quantification of $P \to Q$ transfer as $\mod(\epsilon)$ in situations where 
$\mod(\epsilon) \geq \epsilon$ for all $\epsilon$. 
Cases where $\mod(\epsilon) \ll \epsilon$ (e.g., cases with \emph{transfer exponent} $\rho < 1$)
are referred to as \emph{super-transfer} by \citet*{hanneke2019value}, who also argue that $\disc_{\cal Y}^*(P,Q;\epsilon)$ does not capture such cases.
\end{remark}

\paragraph{Regression.} In what follows, let $\loss(a,b)=(a-b)^2$ denote the squared loss, and $\cal Y$ a general subset of $\real$. The discussion inherently assumes that $\expec \ Y^2 < \infty$ under $P$ or $Q$.  

We start with Wasserstein distance which has recently been popular in Machine Learning, and in particular for domain adaptation \citep{redko2017theoretical, shen2018wasserstein}.

\begin{example}[Wasserstein 1] 
\label{ex:wasserstein}
Let $\cal L$ denote the set of $1$-Lipschitz functions on $\X$ w.r.t some metric. Then consider the \emph{integral probability metric} 
$W_1(P, Q) \doteq \sup_{f \in \cal L}\abs{\EE_{P_X} (f) - \EE_{Q_X}(f)}$. 

Suppose $\Hyp \subset \lambda \cdot {\cal L}, \lambda > 0$, and 
is a set of bounded functions $\X \to [-M,M]$.
We consider a weak covariate shift 
scenario, where $\EE_{P}[Y|X] = \EE_{Q}[Y|X] = \hstar(X)$ for some $\hstar \in \Hyp$. Notice that $u\mapsto u^2$ is $4M$-Lipshitz on $[-2M, 2M]$, while for any 
$h, h'\in \Hyp$, $(h-h')$ is $2\lambda$-Lipschitz on 
$\X$, and has range within $[-2M,2M]$. Then for any $h \in \Hyp$, we have that $x \mapsto (h(x)-\hstar(x))^2$ 
is $(8 M \lambda)$-Lipschitz.
Hence, 
$$
\E_Q(h) 
= \EE_Q (h(X) - \hstar(X))^2 
\leq \EE_P (h(X) - \hstar(X))^2 + 8 M \lambda \cdot W_1(P,Q)
= \E_P(h) + 8 M \lambda \cdot W_1(P,Q).
$$
Thus, $\forall \epsilon > 0$,
\begin{align}
\mod(\epsilon) \leq \epsilon + 8 M \lambda \cdot W_1(P,Q).
\end{align}
\end{example}

We note that \citep{shen2018wasserstein} show a similar bound as above\footnote{Their bound is on $L_1$ loss but the arguments are essentially the same.}, but rather than excess risk, are instead interested in relating $Q$-risk to $P$-risk (for $L_1$ loss). These can be transformed into excess risk bounds, however with additional terms of the form of \eqref{eq:A-disc-upper-bound}, which is unsurprising as $W_1$ may be viewed as extending $\cal A$-discrepancy to any class $\Hyp$ of Lipschitz functions).
In fact, the derivations above are clearly similar to those for the $\cal Y$ and $\cal A$ discrepancies. 

\begin{example}[Linear Regression]
\label{ex:linear-regression}
Consider $\X \subset \real^d$ and $\Y \subset \real$, and let 
$\Hyp \doteq \braces{h_w(x)\doteq w^\top x: w \in \real^d}$. 
Write $\Sigma_\mu \doteq \EE_\mu XX^\top$ for $\mu \equiv $ $P$ or $Q$. Remark that, we have for $h_w \in \Hyp$, 
$\E_\mu (h_w) = (w-w_\mu^*)^\top \Sigma_\mu (w- w_\mu^*)$ for 
$w_\mu^* \in \argmin_w R_\mu (h_w)$---this follows by projection in ${\cal L}_{2, \mu_X}$ of $\expec [Y|X]$ onto the subspace $\Hyp$.

Assume that $\Sigma_P$ is full rank. First, we consider 
a relaxed covariate-shift scenario where $\E_Q(\hstar_P) = 0$, for 
$\hstar_P = h_{w^*_P}$. Letting $w_* \doteq w^*_P$, we then have that, for any $h_w$: 
\begin{align} 
\frac{\E_Q(h_w)}{\E_P(h_w)} = \frac{(w-w_*)^\top \Sigma_Q (w-w_*)}{(w-w_*)^\top \Sigma_P (w-w_*)} \leq \sup_{v}\frac{v^\top \Sigma_Q v}{v^\top \Sigma_P v} = \sup_{v = \Sigma_P^{-\frac{1}{2}}u}\frac{u^\top \Sigma_P^{-\frac{1}{2}}\cdot \Sigma_Q \cdot \Sigma_P^{-\frac{1}{2}}u}{u^\top u} = \lambda_{\text{max}}\paren{\Sigma_P^{-1} \Sigma_Q}, 
\end{align} 
where $\lambda_{\text{max}}$ denotes the largest eigenvalue, 
and the last equality used the fact that the matrices are \emph{similar} (see Theorem 1.3.20 of \citep{matrixAnalysis}). Thus, 
for any $\epsilon > 0$, 
\begin{align}
\mod(\epsilon) \leq \lambda_{\text{max}}\paren{\Sigma_P^{-1} \Sigma_Q} \cdot \epsilon. 
\end{align}

Now, outside of the above relaxed covariate-shift scenario, i.e., when $\E_Q(\hstar_P)\neq 0$, we have that 
$$
\E_Q(h) \leq 2\norm{w-w^*_P}^2_{\Sigma_Q} + 2\norm{w^*_P - w^*_Q}^2_{\Sigma_Q} 
\leq 2\lambda_{\text{max}}\paren{\Sigma_P^{-1} \Sigma_Q}\cdot \norm{w-w^*_P}^2_{\Sigma_P} + 2\E_Q(\hstar_P),$$
letting $\norm{v}^2_{\Sigma} \doteq v^\top \Sigma v$. 
The second inequality follows the same argument as earlier (for the case $\E_Q(\hstar_P) =0$). Hence: 
\begin{align}
  \mod(\epsilon) \leq 2\lambda_{\text{max}}\paren{\Sigma_P^{-1} \Sigma_Q}\cdot\epsilon + 2\E_Q(w^*_P).
\end{align}

\end{example}

\begin{remark}[Tightness]\label{rem:tightness-reg}
We emphasize that $\mod(\epsilon)$ may be strictly smaller than the above bounds since \eqref{eq:W1-Upperbound} remains of a worst-case nature over \emph{all} Lipchitz functions rather than localized to the optimal functions in the class. 
\end{remark}

\subsection{Alternative Procedures for the Weak Modulus}\label{app:weakConf}

We saw in Section \ref{sec:other-algorithms} that the rates of Corollary \ref{cor:strong-mod-rate-bounded-by-weak-mod} may be attained by alternative procedures such as the \emph{robust validation} (Lepski) approach of Example \ref{ex:lepski}. We present a similar approach here, that adapts to the weak modulus, and sometimes to the strong modulus, while relying on a relaxed notion of \emph{confidence set}.

\begin{definition}\label{def:weakconf} 
Let $\mu$ denote either $P$ or $Q$. 
Let $0< \tau \leq 1$ and $\epsilon_\mu>0$. We call a random set $\conf_\mu \doteq \conf_\mu(S_\mu)$ an {\bf $(\epsilon_\mu, \tau)$-weak confidence set} (under $\mu$) if the following conditions are met with probability at least $1-\tau$: 
$$\  {\rm (i)} \ \exists \hat{\epsilon}_\mu > \E_\mu(\hhat_\mu) \text{ such that } \Hyp_\mu(\hat{\epsilon}_\mu)\subset \conf_\mu,
\quad \text{ and }
\quad {\rm (ii)}\  \conf_\mu \subset \Hyp_\mu(\epsilon_\mu). 
$$
\end{definition}

We discuss a construction of such a weak confidence set in Appendix~\ref{app:weak-confidence-construction}, 
which is simpler to construct compared to the analogous strong confidence set in Appendix~\ref{app:strong-confidence-classification-beta-example}. 
Note that the relaxation in this definition (compared to Definition~\ref{def:strongconf})
is that $\conf_\mu$ needs to contain a set $\Hyp_\mu(\hat{\epsilon}_\mu)$, rather than $\Hyp_\mu(\epsilon/C)$, in other words, the ratio $\epsilon_\mu/\hat{\epsilon}_\mu$ is allowed to be random. 

Notice that strong confidence are generally weak confidence sets: if we pick $\epsilon$ 
such that, 
with probability at least $1-\tau$, 
$\E_\mu(\hat{h}_\mu) < \epsilon/C$,
then an $(\epsilon,\tau,C)$-strong confidence set 
is also an $(\epsilon,2\tau)$-weak confidence set.

In the following algorithm, the intent is that the input sets 
$\conf_Q$, $\conf_P$ are weak confidence sets
based on the $Q$-data $S_Q$ and $P$-data $S_P$, respectively. Relevant confidence parameters are stated in the theorem below. 

{\vskip 4mm}\begin{alg}\label{alg:weakupper}
Input: $\conf_Q$, $\conf_P \subset \Hyp$, and $\hhat_Q = \ERM_Q$
\begin{quote}
If $\conf_Q \cap \conf_P \neq \emptyset$, return any $h \in \conf_Q \cap \conf_P$, otherwise return $\hhat_Q$. 
\end{quote}
\end{alg}

The following result easily follows. 

\begin{theorem}
\label{thm:single-delta}
Let $0 < \tau \leq 1$, and suppose $\conf_Q$ is an $(\epsilon_Q, \tau)$-weak confidence set under $Q$, and $\conf_P$ is an $(\epsilon_P, \tau)$-weak confidence set under P, for some $ \epsilon_Q,\epsilon_P > 0$ depending respectively on $n_Q$ and $n_P$. Let $\hat h$ denote the hypothesis returned by Algorithm \ref{alg:weakupper}. We then have that, with probability at least $1-2\tau$: 
$$ \E_Q(\hat h) \leq \min \braces{\epsilon_Q, \mod(\epsilon_P)}.$$
Furthermore, if $\conf_Q \cap \conf_P \neq \emptyset$, we have under the same event that $\E_Q(\hat h) \leq \mod(\epsilon_Q, \epsilon_P)$. 
\end{theorem}

\begin{proof}
Assume the events of Definition \ref{def:weakconf} hold for $\conf_Q$ and $\conf_P$ (with probability at least $1-2\tau$). By Definition \ref{def:weakconf}, condition (ii), we have that any $h \in \conf_Q \cap \conf_P$ is in $\Hyp_Q(\epsilon_Q) \cap 
\Hyp_P(\epsilon_P)$ implying  
$\E_Q(h) \leq \min\braces {\epsilon_Q, \delta(\epsilon_P)}$. 

So suppose $\conf_Q \cap \conf_P = \emptyset$. By both conditions of Definition \ref{def:weakconf}, we have that 
$\hat h_Q \in \conf_Q$ and satisfies $\E_Q(\hat h_Q) \leq \epsilon_Q$. Now, let $\hat \epsilon_{P}$ and 
$\hat \epsilon_{Q}$ as in Definition \ref{def:weakconf}. Consider any 
$h \in \Hyp_P(\hat \epsilon_P)$. By condition (i) we have that 
$h \in \conf_P$ so is therefore not in $\conf_Q$, which contains 
$\Hyp_Q(\hat \epsilon_{Q})$. Consequently, we have that 
$\E_Q(\hat h_Q) < \hat \epsilon_{Q} < \E_Q (h) \leq \delta(\epsilon_P)$ where the last inequality follows from that $h \in \Hyp_P(\hat \epsilon_P) \subset \Hyp_P(\epsilon_P)$. 

The last part of the theorem repeats a similar argument as for Theorem \ref{thm:double-delta}: namely, if we have $\hat h \in \Hyp_Q(\epsilon_Q) \cap \Hyp_P(\epsilon_P)$, then it follows that 
$\E_P(\hhat;\Hyp_Q(\epsilon_Q)) \leq \E_P(\hhat) \leq \epsilon_P$, implying by definition that $\E_Q(\hhat) \leq \mod(\epsilon_Q, \epsilon_P)$. 
\end{proof}

\begin{remark}[Implications and Relation to Existing Bounds]\label{rem:basic-implications}

The above rates in terms of the weak modulus (i.e., as given in Theorem \ref{thm:single-delta} and earlier in Corollary \ref{cor:strong-mod-rate-bounded-by-weak-mod}) recover many of the bounds in the literature in terms of relatedness measures, and further extend them to settings with both source and target data. In particular, for the basic setting with no target data (i.e., where $n_Q = 0$), taking $\conf_Q = \Hyp$, the bound trivially becomes $\mod(\epsilon_P)$ and holds for any choice of $\hhat$ in $\conf_P$: in particular, it's immediate that simply taking $\hhat = \hhat_P$ (the ERM under $P$) would suffice.

For instance in classification, classical VC bounds give us $\epsilon_P \approx \sqrt{(\V+ \log (1/\tau))/n_P}$. Thus, combining with the upper-bounds of Appendix \ref{app:relation2otherMeasures} on $\delta(\epsilon_P)$ (first outlined in Section \ref{sec:weakmod-upper-bounds}), we recover many classical results: for instance, \eqref{eq:Y-disc-upper-bound} yields the $Q$-risk bound which can be derived from the analysis of ERM by \citet*{mohri2012new} (specifically, from equation 7 therein),\footnote{The original bound of \citet*{mohri2012new} is expressed for a drift scenario. It can be specialized to our setting  by taking all but the final distribution to be $P$ and the final distribution to be $Q$.  They also express their result in terms of the Rademacher complexity, which is upper bounded by the square-root bound we have presented here.  We could instead use the empirical Rademacher complexity to define the weak confidence set, and recover an essentially similar result.} while \eqref{eq:A-disc-upper-bound} recovers the $Q$-risk results of \citet*{ben2010theory}.  

Furthermore, as discussed earlier Appendix \ref{app:relation2otherMeasures}, integral probability metrics such as $W_1$ and MMD can be used to extend $\cal A$-discrepancy to regression (under e.g. $L_1$ and $L_2$ loss) for specific classes $\Hyp$ (see e.g. \citep{redko2017theoretical, shen2018wasserstein}) resulting in bounds on $\mod(\epsilon_P)$ of similar form as \eqref{eq:A-disc-upper-bound}. 
The upper-bound of \eqref{eq:W1-Upperbound} for $W_1$ holds under the additional assumptions we stated here. 

In the case of linear regression, the bounds of \eqref{eq:cov-Ratio1} and \eqref{eq:cov-Ratio2} recover similar bounds as in the literature (see, e.g. \citep{zhang2022class, ge2023maximum}). In particular, under covariate-shift, the bound of \eqref{eq:cov-Ratio1} on $\mod(\epsilon_P)$ matches that of \citep{ge2023maximum} stated under additional assumptions\footnote{Their result for linear regression is stated under specific distributional assumptions but specializes a more general bound on ratios of Fisher-Information matrices.}.

Finally, we note that the bound of Theorem \ref{thm:single-delta} in terms of the weak modulus can be faster than the existing bounds it recovers, since the upper-bounds of Appendix \ref{app:relation2otherMeasures} on $\mod(\epsilon_P)$ can often be loose (see Remarks \ref{rem:tightness} and \ref{rem:tightness-reg}). 

\end{remark}

\subsubsection{Example of a Weak Confidence Set}
\label{app:weak-confidence-construction}
Here we present a well-known construction of a weak confidence set for classification \ref{item:vc-cls} (as used, for instance, in the work of \citealp*{hanneke2019value}), 
designed to be adaptive to the noise parameters $(C_\mu, \beta_\mu)$ from Definition~\ref{def: Bernstein noise condition}. Importantly we will see that, as a relaxed notion of confidence set, it is naturally easier to state than the strong confidence set described in Section \ref{app:strongConfSet}.

\begin{proposition}
\label{prop:weakconf}
For $\mu$ either $P$ or $Q$, consider the following definition of $\conf_\mu$.
Define
$$
\Abound_\mu(\tau) \doteq \frac{\V \log(n_\mu / \V) + \log(1/\tau)}{n_\mu},
$$
where $\V$ denotes the VC dimension of $\Hyp$ \citep{VC:72}.
For any $h, h' \in \Hyp$, 
define the empirical distance 
$\hat{\mu}( h \neq h' ) = \frac{1}{n_\mu} \sum_{(x,y) \in S_\mu} \ind{ h(x) \neq h'(x) }$.
There exist universal constants $C, C'$ such that
the following  
$$
\conf_\mu \doteq \left\{ h \in \Hyp : \hat{\E}_\mu(h) \leq C \sqrt{ \hat{\mu}( h \neq \hhat_\mu ) \cdot \Abound_\mu(\tau) } + C \cdot \Abound_\mu(\tau) \right\}
$$
is an $(\epsilon_\mu,\tau)$-weak confidence set, for 
$\epsilon_\mu = C' \left( C_{\mu}\cdot \Abound_\mu(\tau)\right)^{\frac{1}{2-\beta_\mu}},$
where $\beta_\mu \in [0,1]$ and $C_\mu \geq 2$ are 
as in the Bernstein Class Condition (Definition~\ref{def: Bernstein noise condition}).
\end{proposition}

The proof of Proposition~\ref{prop:weakconf} is based on the following lemma, 
which 
is a corollary of results of  \citep{VC:74}
(see \citep{hanneke2022no} for a formal proof).

\begin{lemma}[Uniform Bernstein inequality]
\label{lem:Abound}
For $\mu$ either $P$ or $Q$, there exists a universal constant $\Aconst > 1$ such that, with probability at least $1-\tau$, 
every $h,h' \in \Hyp$ satisfy 
\begin{equation}
\label{eqn:uniform-bernstein-1}
\left| (R_\mu(h) - R_\mu(h')) - (\hat{R}_\mu(h) - \hat{R}_\mu(h')) \right| 
\leq \Aconst \sqrt{\hat{\mu}( h \neq h' ) \cdot \Abound_\mu(\tau)} + \Aconst \cdot \Abound_\mu(\tau)
\end{equation}
and
\begin{equation}
\label{eqn:uniform-bernstein-2}
\frac{1}{2} \mu( h \neq h' ) - \Aconst \cdot \Abound_\mu(\tau) \leq \hat{\mu}( h \neq h' ) \leq 2 \mu( h \neq h' ) + \Aconst \cdot \Abound_\mu(\tau). 
\end{equation}
\end{lemma}

We now present the proof of Proposition~\ref{prop:weakconf}.

\begin{proof}[Proof of Proposition \ref{prop:weakconf}]
In Lemma 2 of \citep{hanneke2022no}, property (ii) of Definition~\ref{def:weakconf} is established, on the event of probability at least $1-\tau$ from Lemma~\ref{lem:Abound}.
Moreover, note that on this same event from Lemma~\ref{lem:Abound}, 
for $\hat{\epsilon}_\mu = \E_\mu(\hhat_\mu) + \Abound_\mu(\tau)$, 
any $h \in \Hyp_\mu(\hat{\epsilon}_\mu)$ has 
$$
\hat{\E}_\mu(h) \leq R_\mu(h) - R_\mu(\hhat_\mu) + \Aconst \sqrt{ \hat{\mu}( h \neq \hhat_\mu ) \cdot \Abound_\mu(\tau) } + \Aconst \cdot \Abound_\mu(\tau)
\leq \Aconst \sqrt{ \hat{\mu}( h \neq \hhat_\mu ) \cdot \Abound_\mu(\tau) } + (\Aconst + 1) \cdot \Abound_\mu(\tau),
$$
where the second inequality is due to $h \in \Hyp_\mu(\hat{\epsilon}_\mu)$, 
implying $R_\mu(h) \leq R_\mu(\hhat_\mu)$.
Therefore, any such $h$ is also in $\conf_\mu$ (for $C \geq \Aconst+1$ in the definition of $\conf_\mu$),
and hence property (i) of Definition~\ref{def:weakconf}
is also satisfied on this event.
\end{proof}

\subsection{Lower-Bounds for the Weak Modulus}\label{app:lowerBoundWeak}

\subsubsection{Classification Lower Bound.}

We start with some supporting results. 

\begin{proposition} [Thm 2.5 of \citep{tsybakov2009introduction}] \label{prop:tsy25} Let $\{ \Pi_{h} \}_{h \in \Hyp}$ be a family of distributions indexed over a subset $\Hyp$ of a pseudo-metric $( \mathcal{F}, \semiMetric)$. Suppose $\exists \, h_0, \ldots, h_{M} \in \Hyp$, where $M \geq 2$, such that:
\begin{align} 
\qquad {\rm (i)} \quad  &\semiDist{h_{i}}{h_{j}} \geq 2 s > 0, \quad \forall 0 \leq i < j \leq M,  & \\
\qquad {\rm (ii)} \quad  & \Pi_{h_i} \ll \Pi_{h_0} \quad \forall i \in  [M], \text{ and the average  KL-divergence to } \Pi_{h_0} \text{ satisfies } & \\
& \qquad 
\frac{1}{M} \sum_{i = 1}^{M} \KLDiv{\Pi_{h_i}}{ \Pi_{h_0}} \leq \alpha \log M, \text{ where } 0 < \alpha < 1/8.
\end{align}
Let $Z\sim\Pi_{h}$, and let $\hat h : Z \mapsto \mathcal{F}$ denote any (possibly \emph{improper}) learner of $h\in \Hyp$. We have for any $\hat h$: 
\begin{equation}
\sup_{h \in \Hyp} \Pi_{h} \left( \semiDist{\hat h(Z)}{h} \geq s \right) \geq \frac{\sqrt{M}}{1 + \sqrt{M}} \left( 1 - 2 \alpha - \sqrt{\frac{2 \alpha}{\log(M)}} \right) \geq \frac{3 - 2 \sqrt{2}}{8} > \frac{1}{48}.
\end{equation}
\end{proposition}

The following proposition would be needed to construct packings (of spaces of distributions) of the appropriate size. 

\begin{proposition} [Varshamov-Gilbert bound] \label{lem:VGBound}
Let $d \geq 8$. Then there exists a subset $\{ \sigma_0, \ldots, \sigma_{M}\}$ of $\{-1 ,1 \}^{d}$ such that $\sigma_0 = (1,\ldots,1)$,
\begin{equation}
\text{dist}(\sigma_{i},\sigma_{j}) \geq \frac{d}{8}, \quad \forall\,  0 \leq i < j \leq M, \quad \text{and} \quad M \geq 2^{d / 8},
\end{equation}
where $\text{dist}(\sigma,\sigma') \doteq \text{card}(\{ i \in [m] :  \sigma(i) \neq \sigma'(i) \})$ is the Hamming distance.
\end{proposition}

Results similar to the following lemma are known.
\begin{lemma} [A basic KL upper-bound]
\label{lem:klbound} 
For any $0<p, q<1$, we let $\KLDiv{p}{q}$ 
denote $\KLDiv{\text{Ber}(p)}{\text{Ber}(q)}$. 
Now let $0<\epsilon<1/2$ and let $z\in \{ -1, 1\}$. We have 

$$\KLDiv{1/2 + (z/2)\cdot \epsilon\, }{\, 1/2 - (z/2)\cdot \epsilon} 
\leq C_0\cdot \epsilon^2, \text{ for some } C_0 \text{ independent of } \epsilon.$$
\end{lemma}

\begin{proof}[Proof of Theorem \ref{thm:mod-class-lower-bound}]
For claim {\rm (i)}, the case $\V = 1$ is implied by claim {\rm (ii)},
so we will focus our proof of claim {\rm (i)} on the case $\V \geq 2$.
We establish both claims at once by introducing appropriate notation, as follows.
For $\Hyp$ and $f$ as in claim {\rm (i)}, 
let $\VV = \V-1$, $\kappa_0 = \kappa$, 
and let $x_0,x_1,\ldots,x_{\VV}$ be a shatterable subset of $\X$ under $\Hyp$.
For $\Hyp$ and $f$ as in claim {\rm (ii)},
instead let $\VV = 1$ and $\kappa_0 = 1$,
and let $x_0,x_1$ be such that
there exist $h_{-1}, h_{1} \in \Hyp$
with $h_{y}(x_1) = y$,
and $h_{-1}(x_0) = h_{1}(x_1)$
(such $x_0,x_1,h_{-1},h_{1}$ must exist since $|\Hyp| \geq 3$),
and without loss of generality suppose 
$h_{-1}(x_0) = h_{1}(x_0) = 1$.
The points $x_0,\ldots, x_{\VV}$ will form the support of marginals $P_X, Q_X$. Furthermore, let $\tilde \Hyp$ denote the \emph{projection} of $\Hyp$ onto $\braces{x_i}_{i =0}^{\VV}$ (i.e., the quotient space of equivalences $h \equiv h'$ on $\braces{x_i}$),
with the additional constraint that all $h\in \tilde \Hyp$ classify $x_0$ as $1$. 
We can now restrict attention to $\tilde \Hyp$ as the \emph{effective} class. 

Let $\sigma \in \braces{-1, 1}^{\VV}$. 
We will construct a family of distribution pairs $(P_\sigma, Q_\sigma)$ indexed by $\sigma$ to which we then apply Proposition \ref{prop:tsy25} above. 
For any $P_\sigma, Q_\sigma$, we let $\eta_{P, \sigma}, \eta_{Q, \sigma}$ denote the corresponding regression functions (i.e., $\expec_{P_\sigma} [Y | x]$, and $\expec_{Q_\sigma} [Y | x]$). To proceed, fix 
$$\epsilon_P = c_0 \cdot \left( \frac{\VV}{n_P} \right)^{\frac{1}{2-\beta_P}}, \, 
\epsilon_Q = \left( \frac{\VV}{n_Q} \right)^{\frac{1}{2-\beta_Q}}, \text{ and } 
\epsilon = c_1 \cdot \min\braces{ \epsilon_Q, f\!\left( \epsilon_P \right) }, $$
{for some $c_0, c_1 \leq 1$ to be defined so that $\epsilon_P, \epsilon < 1/2$}.

\emph{- Distribution $Q_\sigma$.} We have that $Q_\sigma = Q_X \times Q_{Y|X}^\sigma$, where 
$Q_X(x_0) = 1- \frac{1}{\kappa_0}\epsilon^{\beta_Q}$, while $Q_X(x_i) = \frac{1}{\VV \kappa_0}\epsilon^{\beta_Q}$ for all $i \geq 1$.
Now, the conditional $Q_{Y|X}^\sigma$ is fully determined by $\eta_{Q, \sigma}(x_0) = 1$, 
and $\eta_{Q, \sigma}(x_i) = 1/2 + (\sigma_i/2)\cdot \epsilon^{1-\beta_Q}$, $i \geq 1$. 

\emph{- Distribution $P_\sigma$}. We have that $P_\sigma = P_X \times P_{Y|X}^\sigma$, $P_X(x_0) = 1- \epsilon_P^{\beta_P}$, while $P_X(x_i) = \frac{1}{\VV}\epsilon_P^{\beta_P}$, $i \geq 1$. Now, the conditional $P_{Y|X}^\sigma$ is fully determined by $\eta_{P, \sigma}(x_0) = 1$, and $\eta_{P, \sigma}(x_i) = 1/2 + (\sigma_i/2) \cdot \epsilon_P^{1-\beta_P}$, $i \geq 1$. 

\emph{- Verifying that $(P_\sigma, Q_\sigma) \in \Sigma(f, \beta_P, \beta_Q)$}. For any $\sigma \in \braces{-1, 1}^{\VV}$, let $h_\sigma \in \tilde \Hyp$ denote the corresponding
Bayes classifier (we remark that the Bayes is the same for both $P_\sigma$ and $Q_\sigma$). Now, pick any other $h_{\sigma'} \in \tilde \Hyp$, and let $\text{dist}(\sigma, \sigma')$ denote the Hamming distance between $\sigma, \sigma'$ (as in Proposition \ref{lem:VGBound}). We then have that 
 \begin{align}& \E_{Q_\sigma}(h_{\sigma'}) = \text{dist}(\sigma, \sigma')\cdot\frac{1}{\VV \kappa_0}\epsilon^{\beta_Q}\cdot \epsilon^{1-\beta_Q} = \frac{\text{dist}(\sigma, \sigma')}{\VV \kappa_0}\cdot \epsilon, 
 \text{ while } 
 Q_X(h_{\sigma'} \neq h_\sigma) = \frac{\text{dist}(\sigma, \sigma')}{\VV \kappa_0}\cdot \epsilon^{\beta_Q}, \label{eq:EQ-comp}\\
 &\text{and similarly, } 
 \E_{P_\sigma}(h_{\sigma'}) = \frac{\text{dist}(\sigma, \sigma')}{\VV}\cdot \epsilon_P, \text{ while } 
P_X(h_{\sigma'} \neq h_\sigma) = \frac{\text{dist}(\sigma, \sigma')}{\VV}\cdot \epsilon_P^{\beta_P}.\label{eq:EP-comp}
\end{align}
Since $({\text{dist}(\sigma, \sigma')}/{\VV}) \leq 1$ and $\kappa_0 \geq 1$, it follows that BCC (Definition \ref{def: Bernstein noise condition}) holds with parameters $(1, \beta_P)$ and $(1, \beta_Q)$ respectively for any $P_\sigma$ and $Q_\sigma$ (noting that $\Hyp_P(1/2)$ and $\Hyp_Q(1/2)$ both project to $\tilde \Hyp$). 

Moreover, notice that for every such $\E_{P_{\sigma}}(h_{\sigma'})$, we also have that (in both cases {\rm (i)} and {\rm (ii)}) 
$$\E_{Q_{\sigma}}(h_{\sigma'}) \leq \frac{c_1}{\kappa_0}\cdot \frac{\text{dist}(\sigma,\sigma')}{d} \cdot f(\epsilon_P) \leq \frac{c_1}{\kappa_0}\cdot \kappa_0 \cdot f \paren{\frac{\text{dist}(\sigma,\sigma')}{d} \cdot \epsilon_P} 
\leq  f\paren{\E_{P_{\sigma}}(h_{\sigma'})}. 
$$

{ 
Notice that, by construction, the mapping $\E_{P_{\sigma}}(h)\mapsto \E_{Q_{\sigma}}(h)$ is nondecreasing, hence the above implies that, for any $\epsilon' = \frac{k}{d} \cdot \epsilon_P$, $k \in \{0,\ldots,d\}$, we do have $\mod(\epsilon') \leq f(\epsilon')$.
Now, for values $\frac{k}{d} \cdot \epsilon_P < \epsilon' < \frac{k+1}{d} \cdot \epsilon_P$,
$k \in \{0,\ldots,d-1\}$,
note that $\Hyp_P(\epsilon') = \Hyp_P(\frac{k}{d} \cdot \epsilon_P)$
so that 
$\mod(\epsilon') = \mod(\frac{k}{d} \cdot \epsilon_P) \leq f(\frac{k}{d} \cdot \epsilon_P) \leq f(\epsilon')$
(since $f$ is nondecreasing).
Noting that the requirement of $\Sigma_{\Hyp}(f,\beta_P,\beta_Q)$
only concerns $\mod(\epsilon') \leq f(\epsilon')$ for $\epsilon' \leq 1/4$, 
and recalling that $\epsilon_P \leq 1/2$,
this completely verifies that 
$(P_{\sigma},Q_{\sigma}) \in \Sigma_{\Hyp}(f,\beta_P,\beta_Q)$
(for both case {\rm (i)} and case {\rm (ii)}).
}

\emph{- Reduction to a packing}. Now apply Proposition \ref{lem:VGBound} to identify a subset $\Sigma$ of $\braces{-1, 1}^{\VV}$, where 
$\abs{\Sigma} = M \geq 2^{\VV/8}$, and $\forall \sigma, \sigma' \in \Sigma$, we have $\text{dist}(\sigma, \sigma') \geq \VV/8$.  It should be clear then that 
for any $\sigma, \sigma' \in \Sigma$, 
$$\E_{Q_\sigma}(h_{\sigma'}) \geq \frac{\VV}{8}\cdot \frac{1}{\VV{\kappa_0}}\epsilon^{\beta_Q}\cdot \epsilon^{1-\beta_Q} = \epsilon/(8{\kappa_0}).$$
Furthermore, by construction, any classifier $\hat h: \braces{x_i : i \leq d} \mapsto \braces{-1, 1}$ can be reduced to a decision on $\sigma$, and we henceforth view $\text{dist}(\sigma, \sigma')$ as the pseudo-metric referenced in Proposition \ref{prop:tsy25}, with effective indexing set $\Sigma$. 

\emph{- KL bounds in terms of $n_P$ and $n_Q$}. 
Define $\Pi_\sigma = P_\sigma^{n_P}\times Q_\sigma^{n_Q}$. We can now verify that all $\Pi_\sigma, \Pi_{\sigma'}$ are close in KL-divergence. First notice that, for any $\sigma, \sigma' \in \Sigma$ (in fact in $\braces{-1,1}^{\VV}$)
\begin{align} 
\KLDiv{\Pi_\sigma}{\Pi_{\sigma'}} &= 
n_P \cdot \KLDiv{P_\sigma}{P_{\sigma'}} + 
n_Q \cdot \KLDiv{Q_\sigma}{Q_{\sigma'}} \nonumber \\
&= n_P \cdot \Expectation_{P_X} \KLDiv{P^\sigma_{Y|X}}{P^{\sigma'}_{Y|X}} + 
n_Q \cdot \Expectation_{Q_X} \KLDiv{Q^\sigma_{Y|X}}{Q^{\sigma'}_{Y|X}} \nonumber\\
&= n_P \cdot \sum_{i=1}^{\VV} \frac{ \epsilon_P^{\beta_P}}{\VV}\KLDiv{P^\sigma_{Y|x_i}}{P^{\sigma'}_{Y|x_i}} 
+ n_Q \cdot \sum_{i=1}^{\VV} \frac{\epsilon^{ \beta_Q}}{\VV{ \kappa_0}}\KLDiv{Q^\sigma_{Y|x_i}}{Q^{\sigma'}_{Y|x_i}} \nonumber\\
&\leq C_0\paren{n_P\cdot {\epsilon_P^{(2-\beta_P)}} + 
{ (n_Q/\kappa_0)}\cdot \epsilon^{(2-\beta_Q)}} \label{eq:firstkl}\\
&\leq C_0\VV({ c_0^{(2-\beta_p)}} + c_1^{2-\beta_Q})
\leq 2C_0{ (c_0 \lor c_1)}  \VV.
\label{eq:finalkl}
\end{align}
where, for inequality \eqref{eq:firstkl}, we used Lemma \ref{lem:klbound} to upper-bound the divergence terms. It follows that, for $c_0, c_1$ sufficiently small so that 
$2C_0{(c_0 \lor c_1)}\leq 1/16$, we get that \eqref{eq:finalkl} is upper bounded by $(1/8) \log M$. Now apply Proposition \ref{prop:tsy25} and conclude with the first two claims {\rm (i)} and {\rm (ii)} of the theorem.

For the final claim in the theorem, consider $\Hyp \doteq \braces{h_\sigma}_{\sigma \in \Sigma}$, where $\Sigma$ is supplied by Proposition~\ref{lem:VGBound} as just described above.
In particular, since all functions in $\Hyp$ agree on $x_0$,
and $2^{d/8} \leq |\Hyp| \leq 2^{d}$, 
the VC dimension $\V$ satisfies $\V \leq d$. Moreover, 
Sauer's lemma \citep*{sauer1972density,VC:72} 
implies $2^{d/8} \leq \left( \frac{e d}{\V} \right)^{\V}$, 
that is, 
$d \leq 8 \V \log_{2}\!\left( \frac{e d}{\V} \right)$,
which further implies
$d \leq 16 \V \log_{2}\!\left( 8 e \right) \leq 72 \V$,
so that $\V \geq d/72$.
Thus, $\V = \Theta(d)$.

{For this class, we modify the construction slightly as follows: now set $\epsilon = \min \braces{\epsilon_Q, \frac{1}{8}\epsilon_P}$, and follow the same construction as is case {\rm (i)}, but with $\kappa$ set to 1.

It follows from \eqref{eq:EQ-comp} and \eqref{eq:EP-comp} that the BCC conditions hold, while we can verify the weak moduli is appropriately bounded. Indeed, we have for any $\E_{P_{\sigma}}(h_{\sigma'}) 
= \frac{\text{dist}(\sigma,\sigma')}{d} \cdot \epsilon_P$ that  

\begin{align} 
\E_{Q_{\sigma}}(h_{\sigma'}) \leq  \frac{\text{dist}(\sigma,\sigma')}{d} f\paren{\frac{1}{8} \cdot \epsilon_P} \leq f\paren{\frac{1}{8} \cdot \epsilon_P} \leq f\!\left( \frac{\text{dist}(\sigma,\sigma')}{d} \cdot \epsilon_P \right),
\end{align} 
since $\text{dist}(\sigma,\sigma') \geq \frac{d}{8}$ for $\sigma \neq \sigma'$.

The rest of the argument is identical to the above cases {\rm (i)} and {\rm (ii)}. 
} 
\end{proof}

\paragraph{Lower Bound in Terms of $\pivot^\sharp$:}
Here we show that in principle, $\pivot^\sharp \doteq \lim_{\epsilon \to 0} \mod(\epsilon)$ is inescapable in regimes with low $Q$ sample size $n_Q$, unless we make refined assumptions beyond what's captured by the weak modulus (since, e.g., $\exists$ procedures $\hat h$, as shown in Section \ref{sec:strongConfBounds} with $\E_Q(\hat h) \leq \mod(\epsilon_Q, \epsilon_P) $, which can be strictly less than $\pivot^\sharp$ when the transfer problem displays gaps between strong and weak moduli). 

We consider classes of distributions $(P, Q)$ with the following restricted form for $\mod_{\PQ}(\cdot)$. 
 \begin{definition} Let $g: (0, 1]\mapsto [0, 1]$ denote any non-decreasing function, and let $\epsilon_0> 0$. We then let 
     $\Sigma_\Hyp(g, \epsilon_0)$ denotes the set of all pairs of distribution $(P, Q)$ such that 
     $\forall 0 < \epsilon < \frac{1}{4}$, $\mod_{\PQ}(\epsilon) \leq g(\epsilon) + \pivot^\sharp$ and $\pivot^\sharp = \epsilon_0$. 
\end{definition}

We have the following result. 

 \begin{theorem} Consider any $\Hyp$ with VC dimension at least $4$, and a class $\Sigma_\Hyp(g, \epsilon_0)$ for some function $g$ and $\epsilon_0 > 0$. Suppose $g\paren{\frac{1}{4 n_P}} + \epsilon_0 < \frac{1}{8 n_Q}$.
 The following holds for any learner $\hat h$ with access to $S_P\sim P^{n_P}$ and $S_Q\sim Q^{n_Q}$:  
 \begin{align}
     \inf_{\hat h}\sup_{(P, Q) \in \Sigma_\Hyp(g, \epsilon_0)}\pr{\E_Q(\hat h) \geq g\paren{\frac{1}{4 n_P}} + \pivot^\sharp}\geq \frac{1}{8}. 
\end{align}
\end{theorem}
\begin{proof} 
Let $\braces{x_0,x_1,x_2,x_3} \subset \X$ be any set shattered by $\Hyp$. Let $\sigma \in \{-1,1\}^2$ with coordinates $\sigma_1, \sigma_2$.

-- Define $P_X(x_1) = \frac{1}{4n_P}$, 
$P_X(x_0) = P_X(x_2) = P_X(x_3) = \frac{1}{3} - \frac{1}{12 n_P}$.
Define $P^{\sigma}$ to have marginal distribution on $\X$ equal $P_X$, 
and $P^{\sigma}_{Y|X}(Y=1|X=x_1) = \frac{1+\sigma_1}{2}$, while $P^{\sigma}_{Y|X}(Y=1|X=x_i) =0$ for $i \neq 1$.  

-- Define $Q_X(x_1) = g\paren{\frac{1}{4 n_P}}$, 
$Q_X(x_2) = Q_X(x_3) = \epsilon_0$, 
$Q_X(x_0) = 1 - \sum_{i=1}^{3} Q_X(x_i)$.

Define $Q^{\sigma}_{Y|X}(Y=1|X=x_1) = \frac{1+\sigma_1}{2}$,
while $Q^{\sigma}_{Y|X}(Y=1|X=x_2) = \frac{1+\sigma_2}{2}$,
$Q^{\sigma}_{Y|X}(Y=1|X=x_3) = \frac{1-\sigma_2}{2}$, 
and $Q^{\sigma}_{Y|X}(Y=1|X=x_0) = 0$.

Note that we indeed have $\pivot^{\sharp} = \epsilon_0$ for any such $(P^{\sigma},Q^{\sigma})$ pairs.
Also note that functions $h \in \Hyp_{P^{\sigma}}(1/4)$
witness only two possible classifications of $x_0,\ldots,x_3$: 
namely, $h(x_0)=h(x_2)=h(x_3)=-1$, 
so that any $0 < \epsilon < \frac{1}{4 n_P}$ 
has $\mod(\epsilon) = \pivot^{\sharp}$, 
and any $\frac{1}{4 n_P} \leq \epsilon < \frac{1}{4}$ 
has $\mod(\epsilon) = g\paren{\frac{1}{4 n_P}} + \pivot^{\sharp} \leq g(\epsilon) + \pivot^{\sharp}$.

Suppose $g\paren{\frac{1}{4 n_P}} + 2\epsilon_0 < \frac{1}{4 n_Q}$.
Let $\sigma \sim \mathrm{Uniform}(\{-1,1\}^2)$,
and let $S_P,S_Q$ be conditionally iid samples (given $\sigma$)
from $P^{\sigma}$ and $Q^{\sigma}$ and of sizes $n_P$ and $n_Q$, respectively.
Let $\hat{h}$ be as defined under these $S_P$, $S_Q$.

Note that, with probability at least $\frac{1}{2}$, 
$x_1$ does not appear in $S_P$ and none of $x_1,x_2,x_3$ appear in $S_Q$.  Refer to this event as $E_1$.
Note that this event only concerns the \emph{marginal} 
distributions $P_X,Q_X$, 
so that its occurrence is independent of $\sigma$.
Moreover, on this event, none of the $(X,Y)$ in $S_Q$ or $S_P$
have $Y|X$ distributions dependent on $\sigma$ either.
Thus, overall, $\hat{h}$ is conditionally independent of $\sigma$ given this event $E_1$, and $\sigma$ remains conditionally $\mathrm{Uniform}(\{-1,1\}^2)$ given $E_1$.

Thus, $\Prob( \hat{h}(x_1) \neq \sigma_1 \land \hat{h}(x_2) \neq \sigma_2 | E_1 ) = \frac{1}{4}$.
Noting that, when $\hat{h}(x_1) \neq \sigma_1$ and $\hat{h}(x_2) \neq \sigma_2$, we have $\E_{Q^{\sigma}}(\hat{h}) \geq Q_X(x_1)+Q_X(x_2) = g\paren{\frac{1}{4 n_P}} + \epsilon_0$,
altogether we have that (lower-bounding sup with average)
\begin{align}
\sup_{\sigma} \Prob_{(S_P, S_Q)\sim (P^\sigma, Q^\sigma)}\paren{\E_{Q}(\hat{h}) \geq g\paren{\frac{1}{4 n_P}} + \epsilon_0}
& \geq \Prob_{\sigma, (S_P, S_Q)\sim (P^\sigma, Q^\sigma)}  \paren{\E_{Q^{\sigma}}(\hat{h}) \geq g\paren{\frac{1}{4 n_P}} + \epsilon_0 }
\\ & \geq \Prob\paren{\E_{Q^{\sigma}}(\hat{h}) \geq g\paren{\frac{1}{4 n_P}} + \epsilon_0 \middle| E_1 } \frac{1}{2}
\geq \frac{1}{8}.
\end{align}
\end{proof} 

The reader may notice that a slightly more-involved version of the same argument would yield a lower bound of the form $g\paren{\frac{c \V}{n_P}} + \pivot^{\sharp}$ (under restrictions on $g$ similar to Theorem~\ref{thm:mod-class-lower-bound}). Additionally, as in the regression case of Theorem~\ref{thm:reg-LwBnd}, we can also involve $\V$ in the $n_Q$ term at the expense of a constant factor on $\pivot^{\sharp}$: 
i.e., to get a lower-bound of the form $c \cdot \min\braces{\frac{\V}{n_Q}, g\paren{\frac{c' \V}{n_P}} + \pivot^{\sharp} }$ for a constant $c < 1$.

\subsubsection{Regression Lower Bound}

Theorem~\ref{thm:reg-LwBnd} will follow immediately from the following two lemmas.

 \begin{lemma}\label{lem:regression-lb-same-hstar} Under the assumptions of Theorem \ref{thm:reg-LwBnd}, we have 
\begin{align}
\inf_{\hhat} \sup_{(P, Q) \in \Lambda\paren{\sigma^2_Y,\lambda_0, \epsilon_0}} \EE\  \E_Q(\hhat) \geq c\cdot 
\min\braces{ \frac{d\cdot \sigma^2_Y}{n_Q}, \lambda_0 \frac{d\cdot \sigma^2_Y}{n_P}}.
\end{align}
 \end{lemma}
 \begin{proof} 
We consider the following construction with fixed marginals $P_X, Q_X$. Let $e_i, i\in [d]$ denote the $i$-th coordinate vector in $\real^d$. We also let $\epsilon \doteq \min\braces{ \frac{d\cdot \sigma_Y^2}{n_Q}, \lambda_0 \frac{d\cdot \sigma_Y^2}{n_P}}$ and let $\tau > 0$ denote a free variable to be specified later in terms of $\epsilon$. Finally, let $\omega \in \braces{\pm 1}^d$ denote a vector of Rademacher r.v's $\omega_i$. 

$\bullet$ $Q_X$ is supported $\braces{e_i}$, and satisfies $Q_X(e_i) = {1}/{d}$. 
We then define the joint $Q_\omega = Q_X \times Q^{\omega}_{Y|X}$, where the conditional $Q^\omega_{Y|X}$ is given by the relation 
$Y = \tau \cdot \omega^\top X + {\cal{N}}(0, \sigma^2_Y)$ where the noise variable is independent of $X$. 

$\bullet$ $P_X$ is supported $\braces{0} \cup \braces{e_i}$, and satisfies $P_X(e_i) = {1}/(d\cdot \lambda_0)$, and $P_X(0) = 1-1/\lambda_0$. 
We then define the joint $P_\omega = P_X \times Q^{\omega}_{Y|X}$. 

First we verify that for every instance $\omega\in \braces{\pm 1}^d$, 
$(P_\omega, Q_\omega) \in \Lambda(\lambda_0, \epsilon_0)$. By definition, the two distributions shares the same risk minimizer 
$h_\omega (x) \doteq \tau \cdot \omega^\top x$. Next, we have 
$\Sigma_Q = \frac{1}{d}\sum_i  e_i e_i^\top = \frac{1}{d}I_d $, while 
$\Sigma_P = \frac{1}{\lambda_0} \Sigma_Q$ with inverse $\lambda_0\cdot \Sigma_Q^{-1}$, and we therefore have $\lambda_{\max}\paren{\Sigma_P^{-1}\Sigma_Q} = \lambda_0$. 

Now, for any $h: \real^d\mapsto \real$, define the excess risk w.r.t. $Q_\omega$ as $\E_{Q_\omega}(h)\doteq \expec_{Q_X} (h(X) - h_\omega(X))^2$. 

Next, let $\hat h: S_P\times S_Q\mapsto {\cal L}_{2, Q_X}$ denote any learner, where $S_P, S_Q$ denote independent datasets of respective sizes $n_P, n_Q$. For simplicity we also let $\hat h$ denote the learner's output. Then for any $i\in [d]$, define $\hat \omega_i = \text{sign}(h(e_i))$. 

We have that 
$\E_{Q_\omega}(\hat h) \geq \sum_{i=1}^d (\tau^2/d)\cdot \ind{\hat \omega_i \neq \omega_i}$. Letting $\Pi_\omega \doteq P_\omega^{n_P} \times Q_\omega^{n_Q}$, it follows that 
\begin{align}
    \sup_\omega \Expectation_{(S_P, S_Q) \sim\Pi_\omega }  \E_{Q_\omega}(\hat h) \geq \Expectation_\omega \Expectation_{(S_P, S_Q) \sim\Pi_\omega }  \E_{Q_\omega}(\hat h)\geq \tau^2\cdot \frac{1}{d} \cdot \sum_{i=1}^d \Expectation_\omega \Expectation_{(S_P, S_Q) \sim\Pi_\omega } \ind{\hat \omega_i \neq \omega_i}. \label{eq:minmaxreg}
\end{align}

Since the integrand in the summation in \eqref{eq:minmaxreg} is bounded, we can change the order of integration as 
In particular, for any fixed sample $S_P, S_Q$, let $\Xspl, \Yspl$ denote respectively the combined $X$ and $Y$ values in both samples, and let let $\Xspl_i$ denote $X$ values falling at $e_i$ (from both $P, Q$ sampled combined), and let $\Yspl_i$ denote the corresponding $Y$ values.
We then consider the following integration order 
\begin{align}
   \Expectation_\omega \Expectation_{(S_P, S_Q) \sim\Pi_\omega } \ind{\hat \omega_i \neq \omega_i}
   = \Expectation_{\Xspl}\ \Expectation_{\omega \setminus \omega_i}\ \Expectation_{\Yspl\setminus \Yspl_i \ \mid \ \omega\setminus \omega_i, \Xspl}\ \expec_i \ \ind{\hat \omega_i \neq \omega_i}, 
\end{align}
where we let $\expec_i (\cdot)$ denote integration over $\Yspl_i, \omega_i$, conditioned on all other random variables in $S_P, S_Q, \omega$. 

Let $\hat n_{i} \doteq \abs{\Xspl_i}$. Suppose $\hat n_i > 0$ and define the constant vector $a \doteq [\tau, \ldots, \tau] \in \real^{\hat n_i}$. Notice that, viewed as a vector, 
$\Yspl_i = \omega_i\cdot a + {\cal N}(0, \sigma^2_Y\cdot I_{\hat n_i})$, i.e., is drawn according to the mixture $\frac{1}{2}{\cal N}(-a, \sigma^2_Y\cdot I_{\hat n_i}) + \frac{1}{2}{\cal N}(a, \sigma^2_Y\cdot I_{\hat n_i})$. The Bayes classifier for $\omega_i$, namely, $\hstar_i(\Yspl_i)\doteq \text{sign}\paren{a^\top \Yspl_i}$, has 0-1 error 
$\expec_i \ \ind{\hstar_i(\Yspl_i) \neq \omega_i} = \phi(-\norm{a}/\sigma_Y)$, where $\phi$ denotes the standard-normal CDF. Hence we have 
\begin{align} 
\expec_i \ \ind{\hat \omega_i \neq \omega_i} \geq 
\phi(-\norm{a}/\sigma_Y) = \phi \!\paren{-{\tau}\sqrt{\hat n_i/\sigma^2_Y}}. 
\end{align}
If $\hat n_i =0$ then 
$\expec_i \ \ind{\hat \omega_i \neq \omega_i} = \frac{1}{2} = \phi(0)$, 
since the only randomness left is then in $\omega_i$. Now let $\Xspl\sim P_X^{n_P} \times Q_X^{n_Q}$ denote all $X$ values in $S_P, S_Q$. Integrating over $\Xspl$, and by convexity of $z\mapsto \phi(-\tau \sqrt{z/\sigma^2_Y}), z\geq 0$, we have 
\begin{align}
 \Expectation_\omega \Expectation_{(S_P, S_Q) \sim\Pi_\omega } \ind{\hat \omega_i \neq \omega_i}&\geq  \expec_{\Xspl} \ \phi\!\paren{-\tau \sqrt{\hat n_i/\sigma^2_Y}}
 \geq \phi\!\paren{-\tau \sqrt{\frac{1}{\sigma^2_Y}\expec_\Xspl \ \hat n_i}} \\
 &= \phi\!\paren{-\tau \sqrt{\frac{1}{d\sigma^2_Y}(n_P/\lambda_0 + n_Q)}}.
 \label{eqn:regression-lb-expectation step}
\end{align}

Setting $\tau^2 = \epsilon \leq 2 \cdot 
\frac{d\sigma^2_Y}{n_P/\lambda_0 + n_Q}$,
we see that the r.h.s. above is at least $\phi(-\sqrt{2})$. Thus, to finish, we write 
\begin{align}
   \sup_\omega \Expectation_{(S_P, S_Q) \sim\Pi_\omega }  \E_{Q_\omega}(\hat h) \geq
   \tau^2\cdot \frac{1}{d} \cdot \sum_{i=1}^d\phi(-\sqrt{2}) = \phi(-\sqrt{2})\cdot \epsilon. 
   \label{eqn:regression-lb-final-step}
\end{align}
This completes the proof. 
\end{proof}

\begin{lemma}\label{lem:regression-lb-different-hstar}
Under the assumptions of Theorem \ref{thm:reg-LwBnd}, we have 
\begin{align}
\inf_{\hhat} \sup_{(P, Q) \in \Lambda\paren{\sigma^2_Y,\lambda_0, \epsilon_0}} \EE\  \E_Q(\hhat) \geq c\cdot 
\min\braces{ \frac{d \sigma^2_Y}{n_Q}, \epsilon_0},
\end{align}
for a universal constant $c > 0$.
 \end{lemma}
 \begin{proof} 
The lemma trivially holds if $\epsilon_0 = 0$, so w.l.o.g.\ suppose $\epsilon_0 > 0$.
We consider the same setup and notation as in the proof of Lemma~\ref{lem:regression-lb-same-hstar}, 
with the following exceptions:
\begin{itemize}
\item $\tau^2 = \min\!\left\{ \frac{d \sigma^2_Y}{n_Q},\epsilon_0\right\}$,
\item we fix $P_X(e_i) = 1/d$ for all $e_i$
and $P_{Y|X} = \mathcal{N}(0,\sigma^2_Y)$.
\end{itemize}
In particular, note that 
the $P$-risk minimizer is given by
the all-$0$ weight vector
(i.e., $\hstarP(x) = 0$), 
whereas (as in the proof of Lemma~\ref{lem:regression-lb-same-hstar})
the $Q_{\omega}$-risk minimizer is given by the weight vector $\tau \omega \in \{-\tau,\tau\}^d$ (i.e., $\hstar_{Q_{\omega}}(x) = \tau \omega^{\top} x$),
and hence $\E_{Q_{\omega}}(\hstarP) = \tau^2 \leq \epsilon_0$.
Moreover, $\Sigma_P = \Sigma_Q = \frac{1}{d}I_d$, 
so that $(P,Q_{\sigma}) \in \Lambda(\sigma^2_Y,\lambda_0,\epsilon_0)$.

Also note that, 
unlike the proof of Lemma~\ref{lem:regression-lb-same-hstar}, 
for $\omega \sim \mathrm{Uniform}(\{-1,1\}^d)$
here the data set $S_P$
is \emph{independent} of $\omega$, 
so that $S_Q$ is a sufficient statistic, 
and therefore
w.l.o.g.\ the optimal 
estimator $\hat{\omega}$ 
can be assumed to only depend on $S_Q$ 
(see Theorem 3.18 of \citealp{schervish:95}).
In particular, 
letting $\hat{n}_{i, Q}$ 
denote the number of $S_Q$
samples with $X = e_i$, 
the arguments in the proof of Lemma~\ref{lem:regression-lb-same-hstar} 
remain valid if we replace $\hat{n}_i$ by $\hat{n}_{i, Q}$, 
leading to the conclusions of  
\eqref{eqn:regression-lb-expectation step}, which now becomes
\begin{equation}
 \Expectation_\omega \Expectation_{(S_P, S_Q) \sim P^{n_P} \times Q_\omega^{n_Q} } \ind{\hat \omega_i \neq \omega_i}\geq  \expec_{\Xspl} \ \phi\!\paren{-\tau \sqrt{\hat n_{i, Q}/\sigma^2_Y}}
 \geq \phi\!\paren{-\tau \sqrt{\frac{1}{\sigma^2_Y}\expec_\Xspl \ \hat n_{i, Q}}}
 = \phi\!\paren{-\tau \sqrt{\frac{n_Q}{d\sigma^2_Y}}} \geq \phi(-1),
\end{equation}
and thus the conclusion in \eqref{eqn:regression-lb-final-step} becomes
\begin{equation}
\sup_\omega \Expectation_{(S_P, S_Q) \sim P^{n_P} \times Q_\omega^{n_Q}}  \E_{Q_\omega}(\hat h) \geq
\tau^2\cdot \frac{1}{d} \cdot \sum_{i=1}^d\phi(-1) = \phi(-1) \cdot \min\!\left\{\frac{d \sigma^2_Y}{n_Q},\epsilon_0\right\}. 
\end{equation}
This completes the proof.
\end{proof} 

\begin{proof}[Proof of Theorem \ref{thm:reg-LwBnd}]
The theorem follows immediately from Lemmas~\ref{lem:regression-lb-same-hstar} and \ref{lem:regression-lb-different-hstar}
by noting that 
for $a,b,c \geq 0$, 
$$\min\{ a, b+c \} \leq 2\max\!\left\{ \min\{a,b\},  \min\{a,c\}\right\}.$$
\end{proof}

\section{Strong Modulus}\label{app:strongMod}

\subsection{Examples of Strong Confidence Sets}
\label{app:strongConfSet}

We present some examples of strong confidence sets for classical problems in classification and regression. We note in particular that, while Theorem \ref{thm:double-delta} requires two such confidence sets for different values of 
$\epsilon$, the constructions below are for a single $\epsilon$ but are stated in a way to make it clear that different values of $\epsilon$ are admissible by varying constant factors in the definition of these sets.

\subsubsection{Classification Example}
\label{app:strong-confidence-classification-beta-example}
We consider classification with a VC class $\Hyp$. In what follows assume $\loss(a, b) = \ind{a\neq b}$ and for the simplest case suppose $\Y = \{-1,1\}$. We first provide a simple construction of a strong confidence set with $\epsilon \approx n^{-1/2}$, based on the classic uniform convergence bound of \citet*{talagrand:94}.
Subsequently, we will refine these sets via a more-sophisticated construction, to provide strong confidence sets adaptively achieving an $\epsilon \lesssim n^{-\frac{1}{2-\beta_\mu}}$,
for $\beta_\mu$ from the noise condition of Definition~\ref{def: Bernstein noise condition},
which are therefore smaller when the (unknown) value 
$\beta_\mu > 0$.

\begin{proposition}[$\sqrt{n}$-confidence sets]
\label{prop:strong-conf-root-n}
Let $\Abound_{\mu}'(\tau) \doteq n^{-1}(\V + \log(1/\tau))$.
For any sufficiently large universal constant $C$, 
the set 
\begin{equation}
\conf_{\mu} \doteq \left\{ h \in \Hyp : \hat{\E}_{\mu}(h) \leq C \sqrt{ \Abound'_{\mu}(\tau) } \right\}
\end{equation}
is an $(\epsilon,\tau,2)$-strong confidence set (under $\mu$), 
for $\epsilon = (4/3) C\sqrt{\Abound'_{\mu}(\tau)}$. 
\end{proposition}
\begin{proof}
The classic result of 
\citep{talagrand:94} provides that, 
for any sufficiently large universal constant $C$, 
with probability at least $1-\tau$, 
\begin{equation}
\label{eqn:sharp-uniform-convergence}
\forall h \in \Hyp, \left| \hat{R}_\mu(h) - R_\mu(h) \right| \leq \frac{C}{6} \sqrt{\Abound_{\mu}'(\tau)}.
\end{equation}
On this event, every $h \in \conf_\mu$ 
satisfies 
\begin{equation} 
\E_\mu(h) = R_\mu(h) - \inf_{h' \in \Hyp} R_\mu(h') \leq \hat{R}_\mu(h) - \inf_{h' \in \Hyp} \hat{R}_\mu(h') + \frac{C}{3} \sqrt{\Abound_{\mu}'(\tau)}
= \hat{\E}_\mu(h) + \frac{C}{3} \sqrt{\Abound_{\mu}'(\tau)}
\leq \frac{4}{3} C \sqrt{\Abound_{\mu}'(\tau)},
\end{equation}
where the first inequality is by \eqref{eqn:sharp-uniform-convergence}
and the last inequality follows from the definition of $\conf_\mu$.
Moreover, on this same event, 
any $h \in \Hyp$ with 
$\E_\mu(h) \leq \frac{2C}{3} \sqrt{\Abound_{\mu}'(\tau)}$
has 
\begin{equation}
\hat{\E}_\mu(h) 
\leq \E_\mu(h) + \frac{C}{3}  \sqrt{\Abound_{\mu}'(\tau)}
\leq C  \sqrt{\Abound_{\mu}'(\tau)},
\end{equation}
and hence $h \in \conf_\mu$.
Thus, for $\epsilon = \frac{4}{3}C  \sqrt{\Abound_{\mu}'(\tau)}$,
we have 
$\Hyp_\mu(\epsilon/2) \subset \conf_\mu \subset \Hyp_\mu(\epsilon)$.
\end{proof}

To allow for fast rates, and
bounds for general $\beta_\mu$ in the 
Bernstein Class Condition (Definition~\ref{def: Bernstein noise condition}), 
we refine the above strong confidence set
definition based on 
localization arguments 
involving the 
\emph{uniform Bernstein inequality}
(Lemma~\ref{lem:Abound}).
The result, and underlying principle, is similar 
to results of 
\citep{koltchinskii:06}
stated there in a more general Rademacher-based formulation 
(namely, a combination of Lemma 2 and Theorem 3 therein).

\begin{proposition} 
\label{prop:strong-conf-noise-adaptive}
Let $(C_\mu,\beta_\mu)$ as in the Bernstein Class Condition (Definition~\ref{def: Bernstein noise condition}) with $C_\mu \geq 2$.
Let $\Aconst$ and $\Abound_\mu(\tau)$ be as in Lemma~\ref{lem:Abound}, where we recall  $\Abound_\mu(\tau) \doteq \frac{\V \log(n_\mu / \V) + \log(1/\tau)}{n_\mu}$.
One can construct an $(\hat{\epsilon},\tau,3)$-confidence set 
$\conf_\mu$ with a data-dependent $\hat{\epsilon} = (3/2)\hat{\epsilon}_{\mu}^{\mathrm{loc}}$ defined below.
Moreover, with probability at least $1-\tau$, $\hat{\epsilon} \lesssim (C_\mu \Abound_\mu(\tau))^{\frac{1}{2-\beta_\mu}}$.

Formally, let $2^{-\nats_0} \doteq \{ 2^{-i} : i \in \nats \cup \{0\} \}$.
For any $\epsilon > 0$, define 
$\Hyp_{\hat{\mu}}(\epsilon) \doteq \braces{ h \in \Hyp : \hat{\E}_\mu(h) \leq \epsilon }$
and for any $\Hyp' \subset \Hyp$, 
define the empirical diameter $\diam_{\hat{\mu}}(\Hyp') \doteq \sup_{h,h' \in \Hyp'} \hat{\mu}(h \neq h')$.
Let $\hat{\epsilon}_{\mu}^{\mathrm{loc}}$ denote the minimal element of $2^{-\nats_0}$
such that $\forall \epsilon \in 2^{-\nats_0}$ with $\epsilon \geq \hat{\epsilon}_{\mu}^{\mathrm{loc}}$, 
\begin{equation}
\label{eqn:localization-empirical}
\min\!\left\{ \Aconst \sqrt{\diam_{\hat{\mu}}\!\left(\Hyp_{\hat{\mu}}(2 \epsilon)\right) \cdot \Abound_\mu(\tau)} + \Aconst \cdot \Abound_\mu(\tau), 1 \right\} \leq \frac{\epsilon}{2}.
\end{equation}
Then with probability at least $1-\tau$, 
for any $\epsilon \in 2^{-\nats_0}$ with $\epsilon \geq \hat{\epsilon}_{\mu}^{\mathrm{loc}}$, 
\begin{equation}
\label{eqn:localization-inclusions}
\Hyp_{\mu}\!\left(\frac{\epsilon}{2}\right) \subset \Hyp_{\hat{\mu}}(\epsilon) \subset \Hyp_{\mu}\!\left(\frac{3}{2}\epsilon\right).
\end{equation}
Furthermore, this implies that for any $C \in \{ 2^{i} : i \in \nats \cup \{0\}\}$, 
$\conf_\mu \doteq \Hyp_{\hat{\mu}}(C \hat{\epsilon}_{\mu}^{\mathrm{loc}})$
is a $((3/2)C \hat{\epsilon}_{\mu}^{\mathrm{loc}}, \tau, 3 )$-strong confidence set.
\end{proposition}
\begin{proof}
Throughout the proof, suppose the event of probability at least $1-\tau$ from Lemma~\ref{lem:Abound} holds.
We proceed to prove \eqref{eqn:localization-inclusions} by induction.
As a base case, considering $\epsilon = 1$, we trivially have $\Hyp_\mu(1/2) \subset \Hyp_{\hat{\mu}}(1) \subset \Hyp_\mu(3/2)$.
Now take as an inductive hypothesis that, 
for some $\epsilon \in 2^{-\nats_0}$ with $1 > \epsilon \geq \hat{\epsilon}_{\mu}^{\mathrm{loc}}$, 
it holds that $\Hyp_\mu(\epsilon) \subset \Hyp_{\hat{\mu}}(2\epsilon) \subset \Hyp_\mu(3\epsilon)$.
We will extend the inclusion \eqref{eqn:localization-inclusions} to hold for $\epsilon$.

Toward this end, consider any $h \in \Hyp_\mu(\epsilon/2)$.
By the inductive hypothesis, we have 
$h \in \Hyp_{\hat{\mu}}(2\epsilon)$.
Moreover, by definition, $\hhat_\mu \in \Hyp_{\hat{\mu}}(2\epsilon)$.
Therefore, $\hat{\mu}( h \neq \hhat_\mu ) \leq \diam_{\hat{\mu}}(\Hyp_{\hat{\mu}}(2\epsilon))$.
Thus, by \eqref{eqn:uniform-bernstein-1} from Lemma~\ref{lem:Abound}, 
we have
\begin{equation}
\hat{\E}_\mu(h) 
\leq \frac{\epsilon}{2} + \min\!\left\{ \Aconst \cdot \sqrt{ \diam_{\hat{\mu}}(\Hyp_{\hat{\mu}}(2\epsilon)) \Abound_\mu(\tau)} + \Aconst \cdot \Abound_\mu(\tau), 1 \right\}
\leq \epsilon,
\end{equation}
where the last inequality follows from the fact that 
$\epsilon \geq \hat{\epsilon}_{\mu}^{\mathrm{loc}}$.
Thus, we have verified that $\Hyp_\mu(\epsilon/2) \subset \Hyp_{\hat{\mu}}(\epsilon)$.

Continuing to the second inclusion, 
consider any $h \in \Hyp_{\hat{\mu}}(\epsilon)$.
By the first inclusion, established above, 
any $h' \in \Hyp(\epsilon/2)$ has 
$\hat{\mu}(h \neq h') \leq \diam_{\hat{\mu}}(\Hyp_{\hat{\mu}}(\epsilon)) \leq \diam_{\hat{\mu}}(\Hyp_{\hat{\mu}}(2\epsilon))$.
Therefore, by \eqref{eqn:uniform-bernstein-1} of Lemma~\ref{lem:Abound}, 
we have
\begin{equation}
\E_\mu(h) = R_\mu(h) - \inf_{h' \in \Hyp} R_\mu(h') 
\leq \epsilon + \min\!\left\{ \Aconst \cdot \sqrt{\diam_{\hat{\mu}}(\Hyp_{\hat{\mu}}(2\epsilon)) \Abound_\mu(\tau)} + \Aconst \cdot \Abound_\mu(\tau), 1 \right\}
\leq \frac{3}{2}\epsilon,
\end{equation}
where the last inequality is again due to the fact that 
$\epsilon \geq \hat{\epsilon}_{\mu}^{\mathrm{loc}}$.
Thus, we have verified that $\Hyp_{\hat{\mu}}(\epsilon) \subset \Hyp_\mu((3/2)\epsilon)$.
This completes the proof of \eqref{eqn:localization-inclusions} by the principle of induction.

As the claimed implication that $\Hyp_{\hat{\mu}}(C\hat{\epsilon}_\mu^{\mathrm{loc}})$ is a 
$((3/2)C\hat{\epsilon}_\mu^{\mathrm{loc}},\tau,3)$-strong
confidence set is immediate from \eqref{eqn:localization-inclusions}, 
it remains only to argue that 
$\hat{\epsilon}_\mu^{\mathrm{loc}} \lesssim (C_\mu \Abound_\mu(\tau))^{\frac{1}{2-\beta_\mu}}$.
We show this holds on the same event 
(of probability $1-\tau$)
as above: namely, the event from 
Lemma~\ref{lem:Abound}.
Let $\epsilon_\mu \doteq C' \cdot (C_\mu \cdot \Abound_\mu(\tau))^{\frac{1}{2-\beta_\mu}}$
for an appropriately large universal constant $C'$, 
and suppose $\epsilon_\mu < 1$ (otherwise the 
result trivially holds).
Consider any $\epsilon \in 2^{-\nats_0}$
with 
$\epsilon \geq \epsilon_\mu$.

Let $h \in \Hyp_{\hat{\mu}}(\epsilon)$.
By \eqref{eqn:uniform-bernstein-2} from Lemma~\ref{lem:Abound}, 
we have 
\begin{equation}
\E_\mu(h) \leq \epsilon + \Aconst \cdot \sqrt{\diam_{\hat{\mu}}( \Hyp_{\mu}(\E_\mu(h) ) ) \Abound_\mu(\tau)} + \Aconst \cdot \Abound_\mu(\tau),
\end{equation}
and moreover (also by Lemma~\ref{lem:Abound}),
\begin{equation}
\diam_{\hat{\mu}}( \Hyp_{\mu}(\E_\mu(h) ) )
\leq 2 \diam_{\mu}( \Hyp_{\mu}(\E_\mu(h) ) ) + \Aconst \cdot \Abound_\mu(\tau).
\end{equation}
By the Bernstein Class Condition (Definition~\ref{def: Bernstein noise condition}), 
$\diam_{\mu}( \Hyp_{\mu}(\E_\mu(h) ) ) \leq C_\mu \cdot \E_\mu(h)^{\beta_\mu}$.
Altogether, we have 
\begin{equation}
\E_\mu(h) \leq \max\!\left\{ 2 \Aconst \cdot \sqrt{ 2 C_\mu \cdot \E_\mu(h)^{\beta_\mu} \Abound_\mu(\tau)}, 2\epsilon + 4 \Aconst^{3/2} \cdot \Abound_\mu(\tau) \right\}.
\end{equation}
Solving for $\E_\mu(h)$ yields
\begin{equation}
\E_\mu(h) \leq \max\!\left\{ \left( 2 \Aconst \right)^{\frac{2}{2-\beta_\mu}} \cdot \left( 2 C_\mu \cdot \Abound_\mu(\tau) \right)^{\frac{1}{2-\beta_\mu}}, 2\epsilon + 4 \Aconst^{3/2} \cdot \Abound_\mu(\tau) \right\}
\leq 4 \epsilon,
\end{equation}
for an appropriately large choice of universal constant $C'$.

We are now ready to argue any $\epsilon \in 2^{-\nats_0}$ with $\epsilon \geq \epsilon_\mu$ satisfies \eqref{eqn:localization-empirical}.
By the above, we have 
$\Hyp_{\hat{\mu}}(2\epsilon) \subset \Hyp_{\mu}(8\epsilon)$.
By Lemma~\ref{lem:Abound}, 
$\diam_{\hat{\mu}}(\Hyp_{\mu}(8\epsilon)) \leq 2 \diam_{\mu}(\Hyp_{\mu}(8\epsilon)) + \Aconst \cdot \Abound_\mu(\tau)$.
Together, 
$\diam_{\hat{\mu}}(\Hyp_{\hat{\mu}}(2\epsilon)) \leq 2 \diam_{\mu}(\Hyp_{\mu}(8\epsilon)) + \Aconst \cdot \Abound_\mu(\tau)$.
Finally, by the Bernstein Class Condition, 
$\diam_{\mu}(\Hyp_{\mu}(8\epsilon)) \leq C_\mu \cdot (8\epsilon)^{\beta_\mu}$.
Together, we have
\begin{align}
& \Aconst \sqrt{\diam_{\hat{\mu}}(\Hyp_{\hat{\mu}}(2\epsilon)) \cdot \Abound_\mu(\tau)} + \Aconst \cdot \Abound_\mu(\tau)
\leq \Aconst \sqrt{2 (8\epsilon)^{\beta_\mu} \cdot C_\mu \cdot \Abound_\mu(\tau)} + 2 \Aconst^{3/2} \cdot \Abound_\mu(\tau)
\leq \frac{\epsilon}{2},
\end{align}
for an appropriately large choice of the universal constant $C'$.
We have thus established that $\hat{\epsilon}_\mu^{\mathrm{loc}} \leq 2 \epsilon_\mu$.
\end{proof}

\subsubsection{Regression Example}\label{sec:regression-strongConf}
We consider the linear regression setting of 
\ref{item:lin-reg}. 
{Assume throughout this section that $\norm{x} \leq 1 \ \forall_\mu x$, and $Y-\expec[Y|x]$ is uniformly subGaussian with parameter $\sigma^2_Y \geq 1$; see Condition 3 of \citep{hsu:12}}. 

Abbreviate $\Sigma = \Sigma_\mu$ and 
define $\hat \Sigma \doteq \expec_{\hat \mu} XX^\top$, where $\hat \mu$ is the empirical version of $\mu$ on $S_\mu \sim \mu^{n_\mu}$. We make use of the following corollary to classical concentration results.

\begin{lemma}[Matrix concentration]\label{lem:matrixmultiplicativeconc}
For any $0< \tau <1$, the following holds with probability at least $1-\tau$. There exists $C =  C(\tau, d, \text{eigs}(\Sigma))$ such that, for $n_\mu \geq C$, we have (writing $A\preceq B $ for \emph{$A-B$ is negative semidefinite}): 
\begin{align}
\frac{1}{2}\Sigma \preceq \hat \Sigma \preceq \frac{3}{2}\Sigma. \quad 
\text{ Equivalently } \forall v \in \real^d, \quad 
\frac{1}{2} \norm{v}_\Sigma^2 \leq \norm{v}_{\hat \Sigma}^2 \leq \frac{3}{2}\norm{v}_\Sigma^2.
\end{align}
\end{lemma}
\begin{proof} 
Let $\hat A \doteq \Sigma^{-\frac{1}{2}}\hat \Sigma \Sigma^{-\frac{1}{2}}$ so that $\expec \hat A = I_d$. Then standard matrix Bernstein bounds (see e.g., \citep[][Theorem 1.6.2]{tropp2015introduction}) imply that, with probability at least $1-\tau$, 
\begin{align}
    \norm{\hat A - I_d}_{\text{op}} \leq \alpha_{\mu}, \text{ for }  \alpha_{\mu} \doteq \sqrt{c \frac{\log(d/\tau)}{n_\mu}} + c\frac{\log(d/\tau)}{n_\mu} \text{ where } c = c\paren{\sup_X\norm{\Sigma^{-\frac{1}{2}}X}^2} = c\paren{\lambda_{\text{min}}(\Sigma)}.
\end{align}
In other words, we have  
$(1-\alpha_{\mu}) I_d \preceq \hat A \preceq (1+\alpha_{\mu}) I_d$. We obtain the result by multiplying either side of each term by $\Sigma^{\frac{1}{2}}$, and setting $n_\mu$ sufficiently large so that $\alpha_u \leq \frac{1}{2}$. 
\end{proof}

We have the following proposition. Notice that, the proposition assumes knowledge of upper-bounds on distributional parameters, namely the subGaussian parameter $\sigma_Y$, and a parameter $c_\mu$ which we assume to be an upper bound on 
$R_\mu(h_{w_*}) - R_\mu^*$, where we let $R_\mu^*$ denotes Bayes risk, and $w_*$ minimizes $R_\mu(h_w)$ over the linear class. Thus, if the regression function $\expec[Y|x]$ is linear as often assumed, we have that $c_\mu =0$; more generally,
if $Y$ itself is bounded, then $c_\mu$ can be taken as the square of this bound.

\begin{proposition}
\label{prop:linear-regression-confidence-set}
Let $\hat w$ minimize $\hat R_\mu(h_w)$ over $w\in \real^d$, and for any such $w$, and symmetric and positive $A\in \real^{d\times d}$, define 
$\norm{w}_{A}^2 = w^\top Aw$. Assume $n_\mu \geq C(\tau, d, \text{eigs}(\Sigma))$ as in Lemma \ref{lem:matrixmultiplicativeconc}. 

In what follows, $c_0$ is a universal constant, and $c_\mu$ depends on $\mu$ as explained above. We have that, the following set is a $(26\epsilon, 2\tau, 26)$-strong confidence set, for $\epsilon = c_0\cdot  \sigma^2_{Y}\frac{d + c_\mu + \log(1/\tau)}{n_\mu}$:  
\begin{align}
\conf_\mu \doteq \braces{h_w: \norm{w - \hat w}_{\hat \Sigma}^2 \leq 6\epsilon}.
\end{align}
\end{proposition}
\begin{proof} 
We rely on the following two facts (see e.g. \citep{hsu:12}). First, with probability at least $1-\tau$, we have that 
$\E_\mu (h_{\hat w}) \leq \epsilon$, for some universal $c_0$. Second, it is a classical fact that, for any $w\in \real^d$, 
$\E_\mu(h_{\hat w}) = \norm{w-w_*}^2_{\Sigma}$, where $w_*$ minimizes $R_\mu(h_w)$. Using Lemma \ref{lem:matrixmultiplicativeconc}, we then have that, with probability at least $1-2\tau$, the following holds. 

- For any $h_w \in \conf_\mu$, we have that 
\begin{align}
    \E_\mu(h_w) = \norm{w-w_*}_{\Sigma}^2 \leq 2\norm{w-\hat w}_{\Sigma}^2 + 2 \norm{\hat w-w_*}_{\Sigma}^2
    \leq
    4\norm{w-\hat w}_{\hat \Sigma}^2 + 2  \epsilon 
    \leq 26 \epsilon. 
\end{align}

- For any $h_w \in \Hyp_\mu(\epsilon)$, we have that 
\begin{align}
\norm{w-\hat w}_{\hat \Sigma}^2 
\leq \frac{3}{2}\norm{w-\hat w}_{\Sigma}^2
\leq 3\paren{\norm{w-w_*}_{\Sigma}^2 + \norm{\hat w-w_*}_{\Sigma}^2}
\leq 6\epsilon.
\end{align}
\end{proof}

\begin{remark}
    Notice that, by the above arguments, we can also obtain a $(26\epsilon', 2\tau, 26)$-strong confidence set, for $\epsilon' = 26\epsilon$ simply by plugging in $\epsilon'$ wherever $\epsilon$ appears in the definition of $\conf_\mu$. 
\end{remark}

\begin{remark}[Efficient Implementation]
\label{rem:convex-program}
It is easily seen that, with the instantiation of strong confidence sets expressed in Proposition~\ref{prop:linear-regression-confidence-set}, our methods used in Theorems~\ref{thm:single-delta} and \ref{thm:double-delta} can be implemented efficiently via formulation as convex programs with quadratic constraints.
\end{remark}

\begin{remark}[High-dimensional Settings]\label{rem:strongModwithoutERM}
  Remark that, in settings where the ERM is inappropriate, other estimators might be used to construct a strong confidence set, e.g., a Ridge estimator: in the above, we only used the fact that $\hat w$ guarantees some excess risk $\epsilon = \epsilon(n_\mu)$, second, that the excess risk $\E(h_w)$ of any $h_w$ may be expressed as $\norm{w - w_*}^2_{\Sigma_\mu}$ for any $w$, and last, that $\hat \Sigma \approx \Sigma$ so that $\norm{\cdot}_{\hat \Sigma} \approx \norm{\cdot}_\Sigma$. 
\end{remark}

\subsection{Details for Example of Suboptimality of Robust Validation (Example \ref{ex:lepski})}
\label{sec:lepski-construction}

Let ${\cal X} = \{x_1,x_2,x_3\}$, $\text{supp}(P_X) = \braces{x_1, x_2}$ while $\text{supp}(Q_X) = \braces{x_2, x_3}$ (both uniform on their respective supports). Define $\eta_\mu(x) \doteq  \prf{\mu}{Y =1 | x}$ and let $\eta_P(x_1) = 1$, $\eta_P(x_2) = 1/2$, while 
$\eta_Q(x_2) = \frac{1}{2} + \tau$, and $\eta_Q(x_3) = 1.$ 

Let $\Hyp= \braces{h_1, h_2, h_3}$, 
where $h_3(x_1) \neq h_1(x_1)=h_2(x_1) =1$ and $h_1(x_2) \neq h_2 (x_2) = h_3(x_2) = -1$. Note that, for $n_P$ sufficiently large to reject $h_3$, the ERM $\hhat_P$ equals $h_1$ or $h_2$ with the same probability $1/2$ (since $Y =\pm 1$ with probability $1/2$ on $x_2$).

Also let, 
$h_2(x_3) \neq  h_1(x_3) = h_3(x_3) =1$.  This ensures that $\E_Q(h_2) \gg \Omega(1)$, while letting $\tau \approx n_Q^{-1/2}$, ensures $\E_Q(h_3) \approx n_Q^{-1/2}$ and has a positive probability of being the ERM $\hhat_Q$.

\section{Gap Between Strong and Weak Moduli}
\label{app:gapWeakStrong}

\subsection{Further Definitions}

The following condition on the \emph{pseudo-denseness} of \emph{monotone} subsets of $\Hyp$ will be shown to be equivalent to the \emph{monotonicity} condition of Definition \ref{defn:monotone-subsets} and can yield better insight on the problem. 

\begin{definition}
\label{defn:monotone-subsets}
We call a subset $\Hyp'\subset \Hyp$ \emph{monotone} if it satisfies:
$\forall h, h' \in \Hyp'$ we have that 
$$\E_P(h') \leq \E_P(h) \implies \E_Q(h') \leq \E_Q(h).$$
\end{definition}

\begin{definition}[Pseudo-dense Monotone Sets]
\label{defn:pseudo-dense-monotone-sets}
We say that \emph{monotone sets are pseudo-dense} if 
there exists a sequence of \emph{monotone} subsets $\Hyp_{k} \subset \Hyp, k \in \nats$, satisfying the following \emph{pseudo-denseness} condition.  
 
\begin{itemize}
\item[(\rm{i})] 
$\forall \epsilon_1 \in \AQ(\Hyp) \text{ s.t. } \epsilon_1 > \pivot^\sharp$, 
{\hskip 3mm}$\lim_{k \to \infty} \sup\braces{ \E_Q(h) : h \in \Hyp_k, \E_Q(h) \leq \epsilon_1 } = \epsilon_1$.
\item[(\rm{ii})] $\forall \epsilon_2>0$ s.t. $\mod(\epsilon_2) > \pivot^\sharp$, {\hskip 6.5mm}$\lim_{k \to \infty} \sup\braces{\E_Q(h): h \in \Hyp_{k}, \E_P(h) \leq \epsilon_2} = \mod (\epsilon_2)$.
\end{itemize}
\end{definition}

It should be immediately clear that the above pseudo-denseness of monotones sets imply the conditions of Definition \ref{defn:monotonic-transfer-problem}, i.e., that the transfer problem is monotonic above $\pivot^\sharp$. The other direction is not obvious: 
we show in Proposition~\ref{prop:monotone-implies-no-gap} that if Definition~\ref{defn:monotonic-transfer-problem} holds, then there is no gap above $\pivot^\sharp$. On the other hand, 
we will show in Proposition~\ref{prop:no-gap-implies-monotone} that, if 
a transfer problem has no gap above $\pivot^\sharp$ it must hold that monotone sets are pseudo-dense in $\Hyp$. 

\begin{remark}Pseudo-denseness roughly states that {the union of these $\Hyp_k$'s is \emph{nearly indistinguishable} from $\Hyp$} 
in the achievable $\E_Q$ values and the induced modulus,
down to $\pivot^\sharp$, while \emph{monotonicity} roughly states that $\E_Q(h)$ is nonincreasing as $\E_P(h)$ decreases over $h$'s in $\Hyp_k$. In particular, \emph{all these conditions are met} if $\Hyp$ itself satisfies monotonicity, i.e., we can take $\Hyp_k = \Hyp$ for all $k \in \nats$.
\end{remark}

\subsection{Proof of Gap Characterization Theorem \ref{thm:gap-iff-non-monotonic}}

We start with the following simple lemma, which allows us to restrict attention to situations where $\ddot \epsilon_1 < \mod(\epsilon_2)$.

\begin{lemma}
\label{lem:gap-below-weak-mod-only}
For any $\epsilon_1, \epsilon_2 > 0$, we have that 
$\mod(\epsilon_1, \epsilon_2) < \min\braces{\ddot \epsilon_1, \mod(\epsilon_2)}$ only if  
$\ddot \epsilon_1 < \mod(\epsilon_2)$. 
\end{lemma}
\begin{proof} 
Suppose $\ddot \epsilon_1 \geq  \mod(\epsilon_2) \geq \pivot^\sharp$. That is,  
$\Hyp_Q(\delta(\epsilon_2)) \subset \Hyp_Q(\epsilon_1)$ so that by Corollary \ref{cor:pivotal-value}, $\mod(\epsilon_1, \epsilon_2) = \mod(\epsilon_2)$. 
\end{proof}

The anti-monotonic part of the theorem almost follows by definition as established in the following proposition. 

\begin{proposition}[Below $\pivot^\sharp$, No-Gap $\equiv$ Anti-Monotonicity below]
\label{prop:anti-monotonicity}
We have that $\mod(\epsilon_1,\epsilon_2) = \min\{ \ddot{\epsilon}_1, \mod(\epsilon_2)\}$
for every $0 < \epsilon_1 \leq \pivot^{\sharp}$
and $\epsilon_2 > 0$, 
if and only if 
$(P,Q,\Hyp,\loss)$ is anti-monotonic below $\pivot^\sharp$.
\end{proposition}
\begin{proof} Suppose $\pivot^\sharp > 0$. 
If $(P,Q,\Hyp,\loss)$ is anti-monotonic below $\pivot^\sharp$, 
then for all $0 < \epsilon_1 < \pivot^\sharp$ and $\epsilon_2 > 0$ 
we have 
$$\ddot{\epsilon}_1 \geq \mod(\epsilon_1,\epsilon_2) \geq \pivot^{\sharp}(\Hyp_{Q}(\epsilon_1)) \doteq \lim_{\epsilon \to 0} \sup\{ \E_Q(h) : h \in \Hyp_Q(\epsilon_1), \E_P(h;\Hyp_Q(\epsilon_1)) \leq \epsilon \} = \ddot{\epsilon}_1,$$
implying equality. Noting that $\epsilon_1 \leq \pivot^\sharp$ implies $\min\{\ddot{\epsilon}_1,\mod(\epsilon_2)\} = \ddot{\epsilon}_1$
completes this half of the claim.

For the other direction,
if $\mod(\epsilon_1,\epsilon_2) = \min\{ \ddot{\epsilon}_1, \mod(\epsilon_2) \}$
for every $0 < \epsilon_1 \leq \pivot^\sharp$ and $\epsilon_2 > 0$, 
then 
again by the fact that any such $\epsilon_1,\epsilon_2$ have 
$\min\{\ddot{\epsilon}_1,\mod(\epsilon_2)\} = \ddot{\epsilon}_1$, 
we have 
$$\pivot^\sharp(\Hyp_Q(\epsilon_1))
\doteq \lim_{\epsilon_2 \to 0} \mod(\epsilon_1,\epsilon_2) 
= \lim_{\epsilon_2 \to 0} \min\!\left\{ \ddot{\epsilon}_1, \mod(\epsilon_2) \right\} = \ddot{\epsilon}_1.$$
Thus, $(P,Q,\Hyp,\loss)$ is anti-monotonic below $\pivot^\sharp$.
\end{proof}

\begin{proposition}[Monotonicity $\implies$ No-Gap Above $\pivot^\sharp$] \label{prop:monotone-implies-no-gap}
If the transfer problem $(P,Q,\Hyp,\loss)$ is monotonic above $\pivot^\sharp$, then $\mod(\epsilon_1,\epsilon_2) = \min\{ \ddot{\epsilon}_1, \mod(\epsilon_2)\}$
for every $\epsilon_1 > \pivot^{\sharp}$
and $\epsilon_2 > 0$. 
\end{proposition}

\begin{proof}
    We will argue by contradiction. Suppose there is a gap: fix some $\epsilon_2 > 0$ and suppose that for some $\epsilon_1 > \pivot^\sharp$, we have $\mod(\epsilon_1,\epsilon_2) <\ddot{\epsilon}_1 < \mod(\epsilon_2)$, which by Lemma \ref{lem:gap-below-weak-mod-only} is exhaustive of situations where $\mod(\epsilon_1,\epsilon_2) \neq \min\{ \ddot{\epsilon}_1, \mod(\epsilon_2)\}$. 

    First, note that $\mod(\epsilon_1, \epsilon_2) \geq \pivot^\sharp$ for $\epsilon_1> \pivot^\sharp$: 
since we have $\pivot^\sharp \doteq \lim_{\epsilon\to 0}\mod(\epsilon)$, pick $\epsilon_2' < \epsilon_2$ sufficiently small so that $\pivot^\sharp \leq \mod(\epsilon_2') \leq \ddot \epsilon_1$; by Lemma \ref{lem:gap-below-weak-mod-only}, we have that $\mod(\epsilon_2') = \mod(\epsilon_1, \epsilon_2') \leq \mod(\epsilon_1, \epsilon_2)$ (by Corollary~\ref{cor:pivotal-value}). It follows that, under the assumption of a gap, we also have $\ddot \epsilon_1 > \pivot^\sharp$. 

Pick any $\tau < \paren{\ddot \epsilon_1 - \mod(\epsilon_1, \epsilon_2)}\land \paren{\mod(\epsilon_2) - \ddot \epsilon_1}$. By monotonicity (Definition \ref{defn:monotonic-transfer-problem}), consider any $h,h' \in \Hyp$ satisfying \eqref{eq:pairwise-monotonicity}
\begin{align}
    & \mod(\epsilon_1, \epsilon_2) < \ddot \epsilon_1 - \tau \leq \E_Q(h) \leq \ddot \epsilon_1 \text { and } \\
   & \E_P(h') \leq \epsilon_2 \text{ and } \E_Q(h') \geq \mod(\epsilon_2) - \tau > \ddot \epsilon_1.
\end{align}
Notice that we must have that 
    $\E_P(h) > \epsilon_2$ since otherwise, we would have 
    $h \in \Hyp_Q(\epsilon_1)\cap \Hyp_P(\epsilon_2)$, implying by Proposition~\ref{prop:pivotal-equivalence} that 
$$\E_Q(h) \leq \mod(\epsilon_1, \epsilon_2) =  \sup\braces{\E_Q(h'): h' \in \Hyp_Q(\epsilon_1)\cap \Hyp_P(\epsilon_2)}.$$

Altogether, we have that 
$\E_P(h') \leq \epsilon_2 < \E_P(h) $ while $\E_Q(h') > \ddot \epsilon_1 \geq \E_Q(h)$. In other words, for all such values of $\tau$, there exist no $h, h'$ satisfying 
\eqref{eq:pairwise-monotonicity}. It follows that $(P,Q,\Hyp,\loss)$ cannot be monotonic above $\pivot^\sharp$. 
\end{proof}

We complete the equivalence in the following proposition. 

\begin{proposition}[No-Gap Above $\pivot^\sharp$ $\implies$ Monotonicity] \label{prop:no-gap-implies-monotone}
If $\mod(\epsilon_1,\epsilon_2) = \min\{ \ddot{\epsilon}_1, \mod(\epsilon_2)\}$ for every $\epsilon_1 {>} \pivot^{\sharp}$
and $\epsilon_2 > 0$, then \emph{monotone} subsets are pseudo-dense in $\Hyp$ (as per Definition \ref{defn:pseudo-dense-monotone-sets}).

\end{proposition}

The proof of Proposition \ref{prop:no-gap-implies-monotone} is most involved, as it requires exhibiting monotone subsets of $\Hyp$ satisfying Definition \ref{defn:monotonic-transfer-problem}.  

\subsubsection{Proof of Proposition \ref{prop:no-gap-implies-monotone}}
\label{sec:proofOfMonotonicity}
The following is assumed throughout this proof section. 

\begin{assumption}[No Gap Above $\pivot^\sharp$]\label{assum:no-gap-above-pivot-sharp}
Suppose $\mod(\epsilon_1,\epsilon_2) = \min\{ \ddot{\epsilon}_1, \mod(\epsilon_2)\}$ for all $\epsilon_1 {>} \pivot^{\sharp}$
and $\epsilon_2 > 0$. 
\end{assumption}

Our construction of monotone subsets relies of the following crucial notion. 

\begin{definition}
For every $\epsilon_1 \in \AQ(\Hyp)$, define the \emph{inverse modulus} as 
$$\imod(\epsilon_1) = \inf \braces{\epsilon_2>0: \mod(\epsilon_2) \geq \epsilon_1},$$
and we let $\imod(\epsilon_1) = \infty$ if the set is empty. 
\end{definition}

The above is always finite for $\epsilon_1 < \sup{\AQ(\Hyp)}$: pick $h\in \Hyp$ such that $\E_Q(h) \geq \epsilon_1$ and let $\epsilon_2 = \E_P(h)$; by definition, $\mod(\epsilon_2) \geq \epsilon_1$. For the same reason, it is also finite if $\sup{\AQ(\Hyp)}$ is achieved by some $h$. 
Notice that a necessary condition for the set $\braces{\epsilon_2>0: \mod(\epsilon_2) \geq \epsilon_1}$ to be empty 
is that 
$\epsilon_1 = \sup \AQ(\Hyp) < \infty$ and $\sup \{ \E_P(h) : h \in \Hyp \} = \infty$.
Thus, in most scenarios of interest (e.g., classification, or linear regression) 
this will never occur.
Nevertheless, our results also allows for this case.

Also, notice that $\imod(\epsilon_1)$ is a non-decreasing function since for $\epsilon_1' > \epsilon_1$, the set 
$\braces{\epsilon_2: \mod(\epsilon_2) \geq \epsilon_1'} \subset \braces{\epsilon_2: \mod(\epsilon_2) \geq \epsilon_1}$.

\begin{lemma}[Link to Strong Modulus]
\label{lem:imod-link-to-strong-mod}
We have:  
$$\forall \epsilon_1 \in \AQ(\Hyp) \setminus [0, \pivot^\sharp], \quad \imod(\epsilon_1) = \inf\braces{\epsilon_2 > 0: \mod(\epsilon_1, \epsilon_2) = \epsilon_1}.$$
\end{lemma}
\begin{proof} 
This follows since we have the tautology 
$\braces{\mod(\epsilon_2) \geq \epsilon_1} \equiv \braces{\min\braces{\epsilon_1, \mod(\epsilon_2)} = \epsilon_1}$, and $\epsilon_1 = \ddot \epsilon_1$. 
\end{proof} 

$\bullet$ {\bf Base Set.}
For each $l \in \nats$, 
for $\alpha = 2^{-l}$, 
define a finite set ${\cal A}_{\alpha} \subset (\pivot^{\sharp},\pivot^{\sharp} + 1/\alpha + \alpha)$ 
as the set of all values 
$\inf \!\left( \left[ \pivot^{\sharp}\!+\!i 2^{-l}, \pivot^{\sharp} \!+\! (i\!+\!1) 2^{-l} \right) \!\cap \AQ(\Hyp) \right)$, 
among $i \in \{1,\dots,2^{2l}\}$ for which this set is non-empty.

\textbf{Properties:}
\begin{enumerate} 
\item For every $l,l' \in \nats$ 
with $l > l'$, 
${\cal A}_{2^{-l'}} \subset {\cal A}_{2^{-l}}$.
\item For every $\epsilon_1 \in \AQ(\Hyp) \cap [\pivot^\sharp + \alpha, \pivot^{\sharp} + \frac{1}{\alpha})$, 
there exists $\epsilon'_1 \in {\cal A}_{\alpha} \cap [\epsilon_1 - \alpha, \epsilon_1]$.
\end{enumerate}

$\bullet$ {\bf Approximation parameters.} For every fixed non-empty $\cal A_\alpha$, pick $\zeta, \beta > 0$ to satisfy: 
$$\zeta < \frac{\alpha}{2} \land \frac{1}{2} \min\braces{\abs{\epsilon_1 - \epsilon_1'}: \epsilon_1, \epsilon_1' \text{ are distinct values in } {\cal A}_\alpha},$$
or any $\zeta \in (0,\alpha)$ if $|{\cal A}_\alpha|=1$, and 
\begin{equation} 
\label{eqn:beta-constraint}
\beta < \frac{1}{2} \min\braces{\abs{\imod(\epsilon_1) - \imod(\epsilon_1')}: \imod(\epsilon_1), \imod(\epsilon_1') \text{ are distinct values, and } \epsilon_1,\epsilon_1' \in {\cal A}_\alpha},
\end{equation}
or any $\beta > 0$ if $|\{ \imod(\epsilon_1) : \epsilon_1 \in {\cal A}_\alpha \}|=1$.

$\bullet$ {\bf Candidate Monotone Subsets.} For any $\alpha>0$, fix $\zeta = \zeta(\alpha)$, as defined above and consider any admissible approximation parameter $\beta>0$. We consider subsets  
$\Hyp_{\alpha, \beta}$ of $\Hyp$, satisfying the following. 

Let $\epsilon_{(1)} < \epsilon_{(2)} \ldots < \epsilon_{(m)}$ denote all ordered elements of ${\cal A}_\alpha$. Starting with $i=m$, pick 
\begin{align} 
& h_i \in \Hyp: \quad \E_Q(h_i) \in [\epsilon_{(i)} - \zeta, \epsilon_{(i)}] \text { and } 
\imod(\epsilon_{(i)}) - \beta \leq \E_P(h_i) \leq \min\braces{\imod(\epsilon_{(i)}), \E_P(h_{i+1})}\!, \text{ if feasible, else pick}
\\ & h_i \in \Hyp: \quad \E_Q(h_i) \in [\epsilon_{(i)} - \zeta, \epsilon_{(i)}] \text { and } 
\E_P(h_i) < \min\braces{\imod(\epsilon_{(i)}) + \beta, \E_P(h_{i+1})},
\end{align}
where we define $\E_P(h_{m+1})=\infty$ for simplicity. 
We then define $\Hyp_{\alpha,\beta} \doteq \{h_1,\ldots,h_m\}$.

\begin{remark}[Proof Structure]\label{rem:proof-structure}
Note that if such an assignment of $\Hyp_{\alpha, \beta}$ exists, it is by construction monotone (i.e., satisfies Definition \ref{defn:monotonic-transfer-problem} (a)). We therefore will first show that it exists, followed by arguing that the collection $\braces{\Hyp_{\alpha, \beta}}, \alpha, \beta\to 0$ necessarily satisfies pseudo-denseness (Definition \ref{defn:monotonic-transfer-problem} (b)). 
\end{remark}

We start with the following key lemma. 

\begin{lemma} \label{lem:EQ-EP-zeta-eta-choice}
Consider any $\epsilon_1 \in \AQ(\Hyp), \epsilon_1 > \pivot^\sharp$, and let $\epsilon_2$ satisfy
$\mod(\epsilon_1, \epsilon_2) = \epsilon_1$. Then for any $\zeta, \eta > 0$, $\exists h$ such that  
$$\E_Q(h) \in [\epsilon_1 - \zeta, \epsilon_1] \text{ and } \E_P(h) \in [\imod(\epsilon_1) -\eta, \epsilon_2].$$
\end{lemma}
\begin{proof} 
First notice that 
$\Hyp_P(\epsilon_2) \cap \Hyp_Q(\epsilon_1) \setminus \Hyp_Q(\epsilon_1 -\zeta)$ is non-empty: 
if it were empty, then $\mod(\epsilon_1, \epsilon_2) \leq \epsilon_1 -\zeta$ (contradicting $\mod(\epsilon_1,\epsilon_2)=\epsilon_1$) since by Corollary~\ref{cor:pivotal-value} the strong modulus equals $\sup\braces{\E_Q(h): h\in \Hyp_P(\epsilon_2) \cap \Hyp_Q(\epsilon_1)}$. Hence, we can define: 
$$\epsilon_2' = \sup\braces{\E_P(h): h \in \Hyp_P(\epsilon_2) \cap \Hyp_Q(\epsilon_1) \setminus \Hyp_Q(\epsilon_1 -\zeta)}.$$
It holds that $\mod(\epsilon_1, \epsilon_2') = \epsilon_1$, since $\mod(\epsilon_1, \epsilon_2') \leq \epsilon_1$ by definition, whereas 
$$\Hyp_P(\epsilon_2') \cap \Hyp_Q(\epsilon_1) \text{ contains } \Hyp_P(\epsilon_2) \cap \Hyp_Q(\epsilon_1) \setminus \Hyp_Q(\epsilon_1 -\zeta),$$
implying $\mod(\epsilon_1, \epsilon_2') \geq \sup \braces{\E_Q(h): h \in \Hyp_P(\epsilon_2) \cap \Hyp_Q(\epsilon_1) \setminus \Hyp_Q(\epsilon_1 -\zeta)} = \epsilon_1$. 

We therefore have $\imod(\epsilon_1) \leq \epsilon_2' \leq \epsilon_2$, since by Lemma~\ref{lem:imod-link-to-strong-mod}, $\imod(\epsilon_1)$ is the infimum of  
$\braces{\epsilon: \mod(\epsilon_1, \epsilon) = \epsilon_1}$. 

It follows that $\exists h \in \Hyp_P(\epsilon_2) \cap \Hyp_Q(\epsilon_1) \setminus \Hyp_Q(\epsilon_1 - \zeta)$ 
with $\E_P(h) \geq \epsilon_2' - \eta$, establishing the claim.
\end{proof}

\begin{corollary}
\label{cor:EQ-EP-zeta-beta-choice}
For any $\epsilon_1 \in \AQ(\Hyp)$ with $\epsilon_1 > \pivot^\sharp$, and any $\zeta, \eta >0$, $\exists h \in \Hyp$ such that 
$$\E_Q(h) \in [\epsilon_1 - \zeta, \epsilon_1] \text { and } 
\E_P(h) \in \left[\imod(\epsilon_1)-\eta,  \imod(\epsilon_1) + \eta\right].$$
\end{corollary}
\begin{proof} 
Pick $\epsilon_2 \in (\imod(\epsilon_1), \imod(\epsilon_1) + \eta)$. 
By definition of $\imod(\epsilon_1)$ and monotonicity of $\mod(\cdot)$,  
we have $\mod(\epsilon_2) \geq \epsilon_1$.
Therefore, by Assumption~\ref{assum:no-gap-above-pivot-sharp}, 
$\mod(\epsilon_1,\epsilon_2) = \min\{ \epsilon_1, \mod(\epsilon_2) \} = \epsilon_1$.
Now apply Lemma~\ref{lem:EQ-EP-zeta-eta-choice}. 
\end{proof} 

The above immediately implies a first simple recursive condition for choices of $h_i$. 

\begin{corollary}[Conditional choice of $h_i$] \label{cor:cond-choice-h-i}
Consider the construction of $\Hyp_{\alpha, \beta}$. 
For any $i \in [m-1]$, suppose a choice of $h_{(i+1)}$ was made and satisfies $\E_P(h_{(i+1)}) \geq \imod(\epsilon_{(i)})$. Then the choice of $h_i$ in $\Hyp_{\alpha,\beta}$ is also feasible.
\end{corollary}
\begin{proof} 
First, suppose $\E_P(h_{(i+1)}) > \imod(\epsilon_{(i)})$. 
Letting $0 < \eta < \min\braces{ \beta, \E_P(h_{(i+1)}) - \imod(\epsilon_{(i)})}$, 
the statement is obtained by Corollary~\ref{cor:EQ-EP-zeta-beta-choice}.

Now suppose $\E_P(h_{(i+1)}) = \imod(\epsilon_{(i)})$. Write $\epsilon_2 \doteq \imod(\epsilon_{(i)})$, and notice that 
$\mod(\epsilon_2) \geq \E_Q(h_{i+1}) \geq \epsilon_1$. Thus, by Assumption \ref{assum:no-gap-above-pivot-sharp}, we have $\mod(\epsilon_{(i)}, \epsilon_2) = \epsilon_{(i)}$. Lemma \ref{lem:EQ-EP-zeta-eta-choice} therefore holds for the pair $\epsilon_{(i)}$, $\epsilon_2 \doteq \imod(\epsilon_{(i)})$.  
\end{proof}

Finally we verify that $\E_P(h_{(i+1)}) < \imod(\epsilon_{(i)})$ cannot actually happen, i.e., the above covers all cases. 

\begin{lemma} \label{lem:approx-from-above}
Let $\epsilon_1, \epsilon_1'\in \AQ(\Hyp)$, satisfying $\epsilon_1' > \epsilon_1 > \pivot^\sharp$.
Then, for any choice of $0< \zeta < \epsilon_1' - \epsilon_1$, it holds for all 
$h\in \Hyp_Q(\epsilon_1')\setminus \Hyp_Q(\epsilon_1' - \zeta)$ that 
$\E_P(h) \geq \imod(\epsilon_1)$. 
\end{lemma}
\begin{proof} 
Pick any such $h\in \Hyp_Q(\epsilon_1')\setminus \Hyp_Q(\epsilon_1' - \zeta)$ and write $\epsilon_2 \doteq \E_P(h)$. We have that 
$\mod(\epsilon_2) \geq \E_Q(h) \geq \epsilon_1' - \zeta \geq \epsilon_1$, hence, by definition of the inverse, 
we have $\epsilon_2 \geq \imod(\epsilon_1)$. 
\end{proof}

\begin{corollary}[$\Hyp_{\alpha, \beta}$ is well-defined.]\label{cor:construction-is-sound}
Fix any $\alpha, \beta > 0$. All choices of $h_i$ in the construction of $\Hyp_{\alpha, \beta}$ are feasible. 
\end{corollary} 
\begin{proof} 
By Corollary \ref{cor:EQ-EP-zeta-beta-choice}, $h_m$ can be chosen to satisfy the conditions of the construction. The induction then proceeds top-down from $i=(m-1)$ to $i=1$ by applying Lemma \ref{lem:approx-from-above} 
(with  $\epsilon_1' = \epsilon_{(i+1)}$ and $\epsilon_1 = \epsilon_{(i)}$) and Corollary \ref{cor:cond-choice-h-i}. 
\end{proof}

Referring back to Remark \ref{rem:proof-structure}, it is left to show that the collection $\Hyp_{\alpha, \beta}$ satisfies \emph{pseudo-denseness} (Definition \ref{defn:pseudo-dense-monotone-sets}). 

\begin{lemma}[Pseudo-denseness]
There exists a sequence of subsets $\braces{\Hyp_k}, k \in  \nats,$ of $\Hyp$, where each $\Hyp_k = \Hyp_{\alpha, \beta}$ for some admissible choice of $\alpha, \beta$, such that $\braces{\Hyp_k}$ is pseudo-dense in $\Hyp$. 
\end{lemma}
\begin{proof} 
Assume the conclusion of Corollary \ref{cor:construction-is-sound} that $\Hyp_{\alpha, \beta}$ is well defined for all admissible pairs $\alpha, \beta>0$. In what follows, for every admissible $\alpha$, we let $\beta(\alpha)$ denote either (i) the largest element of 
$\{2^{-l}\land \alpha: l \in \nats\}$ that is admissible for $\alpha$ (as per \eqref{eqn:beta-constraint}) if $|\{ \imod(\epsilon_1) : \epsilon_1 \in {\cal A}_\alpha \}|>1$, or (ii) $\beta = \alpha$ otherwise. We henceforth use the short notation $\Hyp_k, k \in \nats,$ for $\Hyp_{\alpha, \beta}$, $\alpha = 2^{-k}$, and $\beta = \beta(\alpha)$.

We consider the two aspects of pseudo-denseness separately.

$\bullet$ Pick any $\epsilon_1 \in \AQ(\Hyp)$ with $\epsilon_1 > \pivot^\sharp$. Then for any $\tau >0$, {for any $0< \alpha \leq (\tau/2) \land  (1/\epsilon_1)$}, pick $\epsilon_1' \in  {\cal A}_\alpha \cap [\epsilon_1 - \alpha, \epsilon_1]$. Now, {for any admissible $\beta>0$}, by construction there exists $h \in \Hyp_{\alpha, \beta}$ satisfying 
$$\epsilon_1 \geq \epsilon_1' \geq \E_Q(h) \geq \epsilon_1' - \alpha\geq \epsilon_1 - 2\alpha \geq \epsilon_1 - \tau.$$

Thus, there exists $k_0 = k_0(\tau, \epsilon_1)$ such that for all $k \geq k_0$, $\braces{\E_Q(h): h \in \Hyp_k} \cap [\epsilon_1-\tau, \epsilon_1] \neq \emptyset$. 

$\bullet$ Pick any $\epsilon_2 >0$, {such that $\mod(\epsilon_2) > \pivot^\sharp$}. Let $\epsilon_1 \doteq \mod(\epsilon_2)$; by definition $\epsilon_1\in \AQ(\Hyp)$. Then for any $\tau >0$, {let $0< \alpha_\tau \leq (\tau/2) \land  (1/\epsilon_1)$}, and pick $\epsilon_1' \in  {\cal A}_{\alpha_\tau} \cap [\epsilon_1 - \alpha_\tau, \epsilon_1]$. Define $\epsilon_2' \doteq \imod(\epsilon_1')$. Since $\imod(\cdot)$ is non-decreasing, we necessarily have 
$\epsilon_2' \leq \imod(\epsilon_1) \leq \epsilon_2$. 
We consider a couple cases. In all that follows, for simplicity we define for any $\Hyp' \subset \Hyp$, and $\epsilon>0$, 
$$\mod(\epsilon; \Hyp', \Hyp) \doteq 
\sup \braces{\E_Q(h): h \in \Hyp', \E_P(h) \leq \epsilon}.$$

-- Suppose $\epsilon_2' < \epsilon_2$. {Consider any $0 < \alpha < \alpha_\tau$ and any admissible $\beta$ (for $\alpha$) with $0<\beta \leq \epsilon_2 - \epsilon_2'$}. By construction there exists $h \in \Hyp_{\alpha, \beta}$ such that 
$$\E_Q(h)\geq  \epsilon_1' - \alpha \geq \epsilon_1 - \tau = \mod(\epsilon_2) - \tau,$$
while $\E_P(h) \leq \imod(\epsilon_1') + \beta \doteq  \epsilon_2' + \beta \leq \epsilon_2$. Hence $\mod(\epsilon_2; \Hyp_{\alpha, \beta}, \Hyp) \geq \mod(\epsilon_2) -\tau$ as witnessed by $h$. 

Noting that, for any sufficiently small $\alpha \in \{2^{-k} : k \in \nats\}$, 
we have $\beta(\alpha) \leq \epsilon_2 - \epsilon'_2$, 
it follows that there exists $k_0 = k_0(\tau,\epsilon_2)$ such that every $k \geq k_0$ has $\mod(\epsilon_2;\Hyp_k,\Hyp) \geq \mod(\epsilon_2) - \tau$.

-- Suppose $\epsilon_2' = \imod(\epsilon_1) = \epsilon_2$. Then, since $\mod(\epsilon_2) = \epsilon_1 \geq \epsilon_1'$, we have by Assumption \ref{assum:no-gap-above-pivot-sharp} that 
$\mod(\epsilon_1', \epsilon_2) = \epsilon_1'$. It follows by Lemma \ref{lem:EQ-EP-zeta-eta-choice} that {for any $\alpha \leq \alpha_\tau$ and for any admissible $\beta>0$ (for $\alpha$)}, there exists $h \in \Hyp$ such that 
$$\E_Q(h)\geq  \epsilon_1' - \alpha \geq \epsilon_1 - \tau = \mod(\epsilon_2) - \tau,$$
while 
$$\imod(\epsilon_1') - \beta \leq \E_P(h) \leq \epsilon_2 = \epsilon_2' \doteq \imod(\epsilon_1').$$
Hence some such $h$ would have been included in $\Hyp_{\alpha, \beta}$ (see first condition on the choice of $h_i$ in the \emph{Candidate Monotone Subsets} construction). It follows that $\mod(\epsilon_2; \Hyp_{\alpha, \beta}, \Hyp) \geq \mod(\epsilon_2) -\tau$ as witnessed by this $h$. 
Therefore, for every $k \in \nats$ with $2^{-k} \leq \alpha_\tau$, 
it holds that $\mod(\epsilon_2;\Hyp_k,\Hyp) \geq \mod(\epsilon_2)-\tau$.
\end{proof}

\section{Lower Bounds for the Strong Modulus}\label{app:lowerBoundStrong}

\begin{proof}[Proof of Theorem \ref{thm:strong-modulus-learning-lower-bound}]
As some parts of this proof are identical to that of the 
weak modulus lower bound (proof of Theorem~\ref{thm:mod-class-lower-bound}), 
we mainly discuss here the parts that are different,
and for the sake of brevity we simply reference the 
appropriate portions of the proof of Theorem~\ref{thm:mod-class-lower-bound} 
in the portions that would be identical.

As in the proof of Theorem~\ref{thm:mod-class-lower-bound}, 
we will establish both of the claims simultaneously, 
which we refer to as claims (i) and (ii), respectively, 
via appropriately unified notation.
For claim (ii), let $d = \dim_{\Hyp}-2$, $\kappa_0 = \kappa$.
For claim (i), instead let $d = 1$ and $\kappa_0 = 1$.
In either case, let $x_Q,x_P,x_1,\ldots,x_d$ be a shatterable
subset of $\X$ under $\Hyp$.
These points will form the support of marginals $P_X$, $Q_X$.

We will construct a family of distribution pairs $(P_{\sigma},Q_{\sigma}) \in \Sigma_{\Hyp}(g,\beta_P,\beta_Q)$, $\sigma \in \{-1,1\}^d$, 
such that the claimed lower bound holds for some choice of $(P,Q)$ from among these.
For each $\sigma \in \{-1,1\}^d$, 
define $P_{\sigma}$, $Q_{\sigma}$, as follows: 
let $\eta_{P,\sigma}$, $\eta_{Q,\sigma}$ denote the corresponding regression functions 
$\EE_{P_{\sigma}}[Y|X=x]$, $\EE_{Q_{\sigma}}[Y|X=x]$, respectively.
Let 
$\epsilon_P = c_0 \cdot \left(\frac{d}{n_P}\right)^{\frac{1}{2-\beta_P}}$, 
$\epsilon_Q = \left(\frac{d}{n_Q}\right)^{\frac{1}{2-\beta_Q}}$, 
and $\epsilon = c_1 \cdot g(\epsilon_Q,\epsilon_P)$, 
for some $c_0$, $c_1 < 1$ to be defined so that $\epsilon_P,\epsilon < 1/2$.
\begin{itemize}
\item \emph{Distribution} $Q_{\sigma}$: We set $Q_{\sigma} = Q_X \times Q^{\sigma}_{Y|X}$, where $Q_X(x_Q) = 1 - \frac{1}{\kappa_0} \epsilon^{\beta_Q} - g(\epsilon_Q,\epsilon_P)$,
$Q_X(x_P) = g(\epsilon_Q,\epsilon_P)$, 
$Q_X(x_i) = \frac{1}{d \kappa_0} \epsilon^{\beta_Q}$ for all $i \geq 1$.
The conditional $Q^{\sigma}_{Y|X}$ is fully determined by 
$\eta_{Q,\sigma}(x_P) = \eta_{Q,\sigma}(x_Q) = 1$, 
and $\eta_{Q,\sigma}(x_i) = 1/2 + (\sigma_i/2)\cdot \epsilon^{1-\beta_Q}$, $i \geq 1$.
\item \emph{Distribution} $P_{\sigma}$: We set $P_{\sigma} = P_X \times P^{\sigma}_{Y|X}$, where
$P_X(x_P) = 1 - \epsilon_P^{\beta_P}$, $P_X(x_Q)=0$, $P_X(x_i) = \frac{1}{d} \epsilon_P^{\beta_P}$, $i \geq 1$.
The conditional $P^{\sigma}_{Y|X}$ is fully determined by 
$\eta_{P,\sigma}(x_P) = 1$ and $\eta_{P,\sigma}(x_i) = 1/2+(\sigma_i/2)\cdot \epsilon_P^{1-\beta_P}$, $i \geq 1$.
\end{itemize}

- \emph{Verifying that } $(P_{\sigma},Q_{\sigma}) \in \Sigma_{\Hyp}(g,\beta_P,\beta_Q)$.
For any $\sigma \in \{-1,1\}^d$, 
for any $a,b \in \{0,1\}$, 
let $h_{\sigma}^{(a,b)} \in \Hyp$ have each $h_{\sigma}^{(a,b)}(x_i) = \ind{ \sigma_i = 1 }$, 
and $h_{\sigma}^{(a,b)}(x_Q)=a$, $h_{\sigma}^{(a,b)}(x_P)=b$.
Since $x_Q,x_P,x_1,\ldots,x_d$ are shattered, such $h_{\sigma}^{(a,b)}$ exist in $\Hyp$.
Moreover, since the distributions $P_{\sigma},Q_{\sigma}$ are all supported on $x_Q,x_P,x_1,\ldots,x_d$, 
we may restrict $\Hyp$ to just these functions without loss of generality.
For each $\sigma$, let $h_{\sigma} = h_{\sigma}^{(1,1)}$ denote a Bayes classifier 
(noting that the same function is Bayes for both $P_{\sigma}$ and $Q_{\sigma}$).

We first note that, under $(P_{\sigma},Q_{\sigma})$, 
any $h_{\sigma'}^{(0,b)}$ (i.e., $a = 0$) has $\E_{Q_{\sigma}}(h_{\sigma'}^{(0,b)}) \geq \frac{1}{2}$,
and likewise any $h_{\sigma'}^{(a,0)}$ (i.e., $b = 0$) has $\E_{P_{\sigma}}(h_{\sigma'}^{(a,0)}) \geq \frac{1}{2}$.
Let $\dist(\sigma,\sigma')$ denote the 
Hamming distance between $\sigma$, $\sigma'$ (as in Proposition~\ref{prop:tsy25}).
We then have that
\begin{align} 
& \E_{Q_{\sigma}}(h_{\sigma'}^{(1,b)}) = \ind{ b = 0 } Q_X(x_P) + \dist(\sigma,\sigma') \cdot \frac{1}{d \kappa_0} \epsilon^{\beta_Q} \epsilon^{1-\beta_Q} = \ind{ b = 0 } Q_X(x_P) + \frac{\dist(\sigma,\sigma')}{d \kappa_0} \cdot \epsilon, 
\\ & \text{ while } Q_X(h_{\sigma'}^{(1,b)} \neq h_{\sigma}) = \ind{ b = 0 } Q_X(x_P) + \frac{\dist(\sigma,\sigma')}{d \kappa_0} \cdot \epsilon^{\beta_Q}.
\end{align}
Noting that $x \mapsto x^{\beta_Q}$ is concave, we have
\begin{align}
\left( \ind{ b = 0 } Q_X(x_P) + \frac{\dist(\sigma,\sigma')}{d \kappa_0} \cdot \epsilon \right)^{\beta_Q}
& \geq \left( \ind{ b = 0 } Q_X(x_P) \right)^{\beta_Q} + \left( \frac{\dist(\sigma,\sigma')}{d \kappa_0} \right)^{\beta_Q} \cdot \epsilon^{\beta_Q}
\\ & \geq \ind{ b = 0 } Q_X(x_P) + \frac{\dist(\sigma,\sigma')}{d \kappa_0} \cdot \epsilon^{\beta_Q},
\end{align}
where the last inequality follows from $\dist(\sigma,\sigma') \leq d$ and $\kappa_0 \geq 1$.
Therefore, $Q_{\sigma}$ satisfies BCC (Definition~\ref{def: Bernstein noise condition})
with parameters $(1,\beta_Q)$.
Similarly, 
\begin{align}
\E_{P_{\sigma}}(h_{\sigma'}^{(a,1)}) = \frac{\dist(\sigma,\sigma')}{d} \cdot \epsilon_P, \text{ while } P_X(h_{\sigma'}^{(a,1)} \neq h_{\sigma}) = \frac{\dist(\sigma,\sigma')}{d} \cdot \epsilon_P^{\beta_P}
\leq \left( \frac{\dist(\sigma,\sigma')}{d} \right)^{\beta_P} \cdot \epsilon_P^{\beta_P},
\end{align}
so that $P_{\sigma}$ also satisfies BCC (Definition~\ref{def: Bernstein noise condition}) holds with parameters $(1,\beta_P)$.

Next we verify that $\mod_{P_{\sigma},Q_{\sigma}}(\epsilon_1,\epsilon_2) \leq g(\epsilon_1,\epsilon_2)$ for every $0< \epsilon_1,\epsilon_2 \leq 1/4$.
Since $h^{(1,1)}_{\sigma}$ has minimal 
$\E_{Q_{\sigma}}$ and $\E_{P_{\sigma}}$
(among all possible functions), 
we have $\pivot = 0$ for $(P,Q) = (P_{\sigma},Q_{\sigma})$,
and hence 
Proposition~\ref{prop:pivotal-equivalence} implies
$\mod_{P_{\sigma},Q_{\sigma}}(\epsilon_1,\epsilon_2) = \sup\{ \E_{Q_{\sigma}}(h) : h \in \Hyp_{Q_{\sigma}}(\epsilon_1) \cap \Hyp_{P_{\sigma}}(\epsilon_2) \}$.
In particular, since $\epsilon_1,\epsilon_2 \leq 1/4$, 
we have already argued above that any 
$h^{(a,0)}_{\sigma'}$ or $h^{(0,b)}_{\sigma'}$ 
is \emph{not} in $\Hyp_{Q_{\sigma}}(\epsilon_1) \cap \Hyp_{P_{\sigma}}(\epsilon_2)$, 
so that we can restrict our focus on functions 
$h_{\sigma'} \doteq h^{(1,1)}_{\sigma'}$.

First, in case (i), 
we have  
$\E_{Q_{\sigma}}(h_{-\sigma}) = g(\epsilon_Q,\epsilon_P)$
and $\E_{P_{\sigma}}(h_{-\sigma}) = \epsilon_P$.
Thus, for any $\epsilon_1 < g(\epsilon_Q,\epsilon_P)$ 
or $\epsilon_2 < \epsilon_P$, 
$\mod(\epsilon_1,\epsilon_2) = 0 \leq g(\epsilon_1,\epsilon_2)$.
On the other hand, for $1/4 \geq \epsilon_1 \geq g(\epsilon_Q,\epsilon_P)$ and $1/4 \geq \epsilon_2 \geq \epsilon_P$, 
we have $\mod(\epsilon_1,\epsilon_2) = \E_{Q_{\sigma}}(h_{-\sigma}) = g(\epsilon_Q,\epsilon_P)$.
If $\epsilon_1 \geq \epsilon_Q$, 
admissibility of $g$ implies 
$g(\epsilon_Q,\epsilon_P) \leq g(\epsilon_1,\epsilon_2)$.
Otherwise, if $\epsilon_1 < \epsilon_Q$, 
we have $g(\epsilon_Q,\epsilon_P) \leq \epsilon_1 < \epsilon_Q$, 
and the admissibility property (from Lemma~\ref{lem:gaps-are-open-intervals}) 
implies $g(\epsilon_1,\epsilon_P) = g(\epsilon_Q,\epsilon_P)$, 
so that $\mod(\epsilon_1,\epsilon_2) = g(\epsilon_1,\epsilon_P) \leq g(\epsilon_1,\epsilon_2)$.

Next, consider case (ii).
For any $\sigma' \in \{0,1\}^d$, 
we have (by the assumed properties of $f$)
\begin{equation}
    \E_{Q_{\sigma}}(h_{\sigma'})
    \leq \frac{c_1}{\kappa_0} \cdot \frac{\dist(\sigma,\sigma')}{d} \cdot f(\epsilon_P)
    \leq f\!\left( \frac{\dist(\sigma,\sigma')}{d} \epsilon_P \right)
    = f\!\left( \E_{P_{\sigma}}(h_{\sigma'}) \right).
\end{equation}
Therefore, for any $0 < \epsilon_1,\epsilon_2 \leq 1/4$, 
\begin{align}
\mod(\epsilon_1,\epsilon_2) 
& = \max\!\left\{ \E_{Q_\sigma}(h_{\sigma'}) : \E_{Q_\sigma}(h_{\sigma'}) \leq \epsilon_1, \E_{P_\sigma}(h_{\sigma'}) \leq \epsilon_2 \right\}
\\ & \leq \epsilon_1 \land \max\!\left\{ f(\E_{P_\sigma}(h_{\sigma'}) : \E_{P_\sigma}(h_{\sigma'}) \leq \epsilon_2 \right\}
= \min\{\epsilon_1,f(\epsilon_2)\} 
= g(\epsilon_1,\epsilon_2).
\end{align}

Next we argue the claim that each of these $(P_{\sigma},Q_{\sigma})$ pairs exhibits a gap 
between the strong and weak modulus.
Specifically, note that (since $P_X(x_Q)=0$)
$\mod(\epsilon_P) \geq Q_X(x_Q) \geq \frac{1}{2}$.
Moreover, since $Q_X(x_P) = g(\epsilon_Q,\epsilon_P)$, 
we have $\ddot{\epsilon}_Q \geq g(\epsilon_Q,\epsilon_P)$.
On the other hand, 
$\mod(\epsilon_Q,\epsilon_P) = \epsilon < g(\epsilon_Q,\epsilon_P)$.
Therefore, $\mod(\epsilon_Q,\epsilon_P) < \min\{ \ddot{\epsilon}_Q, \mod(\epsilon_P) \}$.

Finally, we note that the claimed lower bounds on $\E_Q(\hhat)$ 
follow identically (given the above construction and choice of $\epsilon$)
to the corresponding portion of the proof of Theorem~\ref{thm:mod-class-lower-bound} 
(namely, the reduction to a packing, KL bounds in terms of $n_P$ and $n_Q$, 
and the use of Proposition~\ref{prop:tsy25}).
For brevity, we refer the reader to the details 
of these steps in the proof of Theorem~\ref{thm:mod-class-lower-bound}.
\end{proof}
\end{document}